\documentclass[sigconf]{acmart}
\makeatletter
\def\@ACM@checkaffil{
    \if@ACM@instpresent\else
    \ClassWarningNoLine{\@classname}{No institution present for an affiliation}%
    \fi
    \if@ACM@citypresent\else
    \ClassWarningNoLine{\@classname}{No city present for an affiliation}%
    \fi
    \if@ACM@countrypresent\else
        \ClassWarningNoLine{\@classname}{No country present for an affiliation}%
    \fi
}
\makeatother
\usepackage{cleveref} 
\usepackage{multirow}
\usepackage{graphicx}
\usepackage{amsmath}

\usepackage{amssymb}
\usepackage[ruled,linesnumbered]{algorithm2e}
\usepackage{float}
\usepackage{graphicx}
\usepackage{subfigure}
\usepackage[utf8]{inputenc} 
\usepackage[T1]{fontenc}    
\usepackage{hyperref}       
\usepackage{url}            
\usepackage{booktabs}       
\usepackage{amsfonts}       
\usepackage{nicefrac}       
\usepackage{microtype}      
\usepackage{xcolor}         
\usepackage{verbatim}
\usepackage{bm, amsthm}
\usepackage{balance}

\usepackage{wrapfig}
\usepackage{picinpar}
\usepackage{cutwin}
\usepackage{fancyhdr}
\newcommand\vldbdoi{XX.XX/XXX.XX}

\begin{document}
\title{Scaling Up Graph Propagation Computation on Large Graphs: A Local Chebyshev Approximation Approach}

\makeatletter
\newenvironment{breakablealgorithm}
{
	\begin{center}
		\refstepcounter{algorithm}
		\hrule height.8pt depth0pt \kern2pt
		\renewcommand{\caption}[2][\relax]{
			{\raggedright\textbf{\ALG@name~\thealgorithm} ##2\par}%
			\ifx\relax##1\relax 
			\addcontentsline{loa}{algorithm}{\protect\numberline{\thealgorithm}##2}%
			\else 
			\addcontentsline{loa}{algorithm}{\protect\numberline{\thealgorithm}##1}%
			\fi
			\kern2pt\hrule\kern2pt
		}
	}{
	\kern2pt\hrule\relax
\end{center}
}

\author{Yichun Yang}
\affiliation{%
  \institution{Beijing Institute of Technology}
  \city{Beijing}
  \state{China}
}\email{3120241475@bit.edu.cn} 
\author{Rong-Hua Li}
\affiliation{%
  \institution{Beijing Institute of Technology}
  \city{Beijing}
  \country{China}
}\email{lironghuabit@126.com} 
\author{Meihao Liao}
\affiliation{%
  \institution{Beijing Institute of Technology}
  \city{Beijing}
  \country{China}
}\email{mhliao@bit.edu.cn}
\author{Longlong Lin}
\affiliation{%
  \institution{Southwest University}
  \city{Chongqing}
  \country{China}
}\email{longlonglin@swu.edu.cn} 
\author{Guoren Wang}
\affiliation{%
  \institution{Beijing Institute of Technology}
  \city{Beijing}
  \country{China}
}\email{wanggrbit@gmail.com}

\newcommand{\ignore}[1]{}
\newcommand{\nop}[1]{}
\newcommand{\eat}[1]{}
\newcommand{\kw}[1]{{\ensuremath{\mathsf{#1}}}\xspace}
\newcommand{\kwnospace}[1]{{\ensuremath {\mathsf{#1}}}}
\newcommand{\stitle}[1]{\vspace{1ex} \noindent{\bf #1}}
\long\def\comment#1{}
\newcommand{\eop}{\hspace*{\fill}\mbox{$\Box$}}

\newtheorem{property}{Property}
\newtheorem{assumption}{Assumption}

\newcommand{\rank}{\kw{rank}}
\newcommand{\push}{\kw{Push}}
\newcommand{\truncatepush}{\kw{Push}}
\newcommand{\hkrelax}{\kw{Hk\ Relax}}
\newcommand{\hkpush}{\kw{Hk\ Push}}
\newcommand{\agp}{\kw{AGP}}
\newcommand{\tea}{\kw{TEA}}
\newcommand{\teaplus}{\kw{TEA+}}
\newcommand{\ppr}{\kw{PPR}}
\newcommand{\ssppr}{\kw{SSPPR}}
\newcommand{\hkpr}{\kw{HKPR}}
\newcommand{\powerpush}{\kw{PwPush}}
\newcommand{\pwpush}{\kw{PowerPush}}
\newcommand{\powerpushsor}{\kw{PwPushSOR}}
\newcommand{\ltwocheb}{\kw{ChebyPower}}

\newcommand{\chebpush}{\kw{ChebyPush}}
\newcommand{\powermethod}{\kw{PwMethod}}
\newcommand{\rw}{\kw{RW}}
\newcommand{\lv}{\kw{LV}}
\newcommand{\bippr}{\kw{Bippr}}
\newcommand{\fora}{\kw{FORA}}
\newcommand{\speedppr}{\kw{SpeedPPR}}
\newcommand{\foralv}{\kw{FORALV}}
\newcommand{\speedlv}{\kw{SpeedLV}}
\newcommand{\setpush}{\kw{SetPush}}
\newcommand{\chopper}{\kw{CHOPPER}}

\newcommand{\dblp}{\kw{Dblp}}
\newcommand{\asskitter}{\kw{As\textrm{-}Skitter}}
\newcommand{\orkut}{\kw{Orkut}}
\newcommand{\youtube}{\kw{Youtube}}
\newcommand{\livejournal}{\kw{LiveJournal}}
\newcommand{\pokec}{\kw{Pokec}}
\newcommand{\twitter}{\kw{Twitter}}
\newcommand{\friendster}{\kw{Friendster}}

\begin{abstract}

%

Graph propagation (GP) computation plays a crucial role in graph data analysis, supporting various applications such as graph node similarity queries, graph node ranking, graph clustering, and graph neural networks. Existing methods, mainly relying on power iteration or push computation frameworks, often face challenges with slow convergence rates when applied to large-scale graphs. To address this issue, we propose a novel and powerful approach that accelerates power iteration and push methods using Chebyshev polynomials. Specifically, we first present a novel Chebyshev expansion formula for general GP functions, offering a new perspective on GP computation and achieving accelerated convergence. Building on these theoretical insights, we develop a novel Chebyshev power iteration method (\ltwocheb) and a novel Chebyshev push method (\chebpush). Our \ltwocheb method demonstrates an approximate acceleration of $O(\sqrt{N})$ compared to existing power iteration techniques for both personalized PageRank and heat kernel PageRank computations, which are well-studied GP problems. For \chebpush, we propose an innovative subset Chebyshev recurrence technique, enabling the design of a push-style local algorithm with provable error guarantee and reduced time complexity compared to existing push methods. We conduct extensive experiments using 5 large real-world datasets to evaluate our proposed algorithms, demonstrating their superior efficiency  compared to state-of-the-art approaches.

\end{abstract}

\maketitle



\section{Introduction}\label{sec:intro}
Graph propagation (GP) computation is a fundamental operation in graph data analysis with diverse applications, such as graph node similarity queries \cite{lofgren2013personalized, lofgren16bidirection, wang2017fora, wei2018topppr, wu2021unifying, liao2022efficient, 23sigmodepvr}, graph node ranking \cite{page1999pagerank, ding2009pagerank}, graph clustering \cite{andersen2006local, chung2007heat, DBLP:journals/eswa/HeLYLJW24, DBLP:journals/pvldb/LinYLZQJJ24, DBLP:conf/aaai/LinLJ23}, and graph neural networks \cite{guo2021subset, DBLP:conf/mir/YuLLWOJ24, bojchevski2019pagerank, gasteiger2018predict, tsitsulin2018verse, bojchevski2020scaling, wu2019simplifying, gasteiger2019diffusion}. Consequently, numerous GP computation problems have been extensively studied in the literature \cite{yang2024efficient}. Among them, two representative and well-investigated GP models are single-source Personalized PageRank (\ssppr) \cite{andersen2006local} and heat kernel PageRank (\hkpr) \cite{chung2007heat}. Informally, these two GP models are defined as concise summaries of an infinite number of random walks, featuring elegant structural properties and strong interpretability for the aforementioned graph tasks.   





In an undirected graph $\mathcal{G}=(\mathcal{V},\mathcal{E})$ with $n$ nodes and $m$ edges, state-of-the-art (SOTA) GP computation algorithms generally fall into two categories: Power-iteration based method and Push-based approaches (details in Section ~\ref{sec:preliminaries}). Perhaps the most fundamental and representative power-iteration based method is \powermethod \cite{page1999pagerank}. However, due to its slow convergence rate and the $O(m+n)$ number of operations required for each iteration, \powermethod faces significant efficiency challenges when handling large graphs. To illustrate, when the parameter $\alpha$ is small or $t$ is large (i.e. the case when \powermethod is hard to converge), the running time for an \ssppr/\hkpr query using \powermethod may exceed $10^3$ seconds even on moderately sized graphs. While subsequent studies \cite{wu2021unifying,chen2023accelerating} have enhanced the efficiency of \powermethod, their approaches are heavily dependent on the structure of \ssppr and cannot be used for general GP computation (e.g., \hkpr), thus limiting their applicability. 

In contrast, push-based algorithms (such as \push \cite{kloster2014heat,yang2019efficient}) implement \powermethod locally on the graph. The most notable characteristic of push-based algorithms is their local time complexity (i.e., independent of $n$ and $m$), making them highly scalable for massive graphs. However, push-based algorithms are primarily effective for low-precision GP computation. When high-precision GP computation is necessary, similar to \powermethod, push-based approaches also face challenges with slow convergence rates. Moreover, as reported in \cite{wu2021unifying,chen2023accelerating}, push-based algorithms perform even worse than \powermethod in practice for high-precision \ssppr computation. Therefore, developing efficient and provable  algorithms for GP computation remains a challenging task.

To overcome these challenges, we propose several novel approaches to accelerate the computation of general GP using Chebyshev polynomials. Specifically, we observe that the core of both \powermethod and push-based algorithms lies in expanding the GP function using the Taylor series. However, since the polynomial series $\{x^k\}$ from Taylor expansion lacks orthogonality, it often fails to provide the optimal approximation for GP functions, as established in approximation theory \cite{sachdeva2014faster}. Consequently, this results in poor convergence rates when computing the GP function via Taylor expansion. In contrast, approximation theory demonstrates that Chebyshev polynomials, which are orthogonal, offer the best approximations \cite{sachdeva2014faster}. This leads to a natural question: can we leverage Chebyshev polynomials to accelerate the convergence rates of \powermethod or push-based algorithms? In this paper, we answer this question affirmatively. Specifically, we present a novel Chebyshev expansion formula for general GP functions, providing a new perspective on GP computation. Theoretically, we prove that this new Chebyshev expansion achieves faster convergence rates compared to traditional Taylor expansion methods. For instance, for both \ssppr and \hkpr computations, we prove that the convergence rate is accelerated by roughly $O(\sqrt{N})$ times compared to that achievable through Taylor expansion (details in Lemma \ref{lem:converge_rate_cheb}).

Based on our Chebyshev expansion technique, we propose two efficient algorithms: \ltwocheb and \chebpush, tailored to accelerate \powermethod and \push methods, respectively. \ltwocheb utilizes a new Chebyshev power-iteration principle to enhance the convergence rate of \powermethod, while \chebpush leverages a novel subset Chebyshev recurrence technique to obtain a local algorithm. \chebpush works similarly to \push but converges faster than \push. Theoretically, we conduct a thorough analysis of the time complexity and error bounds of our proposed algorithms, demonstrating that \chebpush achieves faster convergence compared to existing methods. Extensive experiments on 5 large real-world graphs confirm the high efficiency of the proposed approaches for both \ssppr and \hkpr computations. In summary, the main contributions of this paper are as follows:



\stitle{New theoretical results.} 
We establish a new Chebyshev expansion formula for approximating general graph propagation functions, demonstrating a faster convergence rate compared to the traditional Taylor expansion. Based on this, we also derive exact Chebyshev coefficients for both \ssppr and \hkpr, which provide new interpretations of \ssppr and \hkpr. We also propose a novel subset Chebyshev recurrence technique and establish theoretical connections between subset Chebyshev recurrence and the standard Chebyshev method. We believe that our subset Chebyshev recurrence technique may prove beneficial for solving other numerical computation problems, thereby may attract independent interests.

\stitle{Novel algorithms.} 
Based on the Chebyshev expansion, we propose a novel Chebyshev power-iteration algorithm, named \ltwocheb. We show that \ltwocheb achieves an approximate acceleration of $O(\sqrt{N})$ over existing power iteration methods for both \ssppr and \hkpr computations. Together with our subset Chebyshev recurrence, we develop a new push-style algorithm, called \chebpush, with time complexity independent of the graph size. We conduct a comprehensive theoretical analysis of \chebpush regarding its accuracy and running time. Our results show that, compared to the SOTA \push algorithm for general graph propagation computations, \chebpush generally exhibits lower time complexity while maintaining similar error bounds. 

\stitle{Extensive experiments.} 
We conduct extensive experiments on 5 real-world graphs. The results show that: (1) For \ssppr computation, our \chebpush algorithm achieves superior performance across all methods in large graphs, while in relatively small graphs, both \ltwocheb and \chebpush perform comparably to SOTA baselines; (2) For \hkpr computation, \chebpush substantially outperforms all other algorithms over all datasets, and \ltwocheb surpass all SOTA baselines on most datasets. In particular,  \chebpush often achieves $3 \sim 8$ times faster speeds than other methods while maintaining the same approximation accuracy.


\section{Preliminaries}\label{sec:preliminaries} 
\subsection{Notations} \label{subsec:notations}
Consider an undirected graph \(\mathcal{G} = (\mathcal{V}, \mathcal{E})\), where \(\mathcal{V}\) denotes the set of nodes with \(|\mathcal{V}| = n\) and \(\mathcal{E}\) represents the set of edges with \(|\mathcal{E}| = m\). For any node $u\in \mathcal{V}$, we denote $\mathcal{N}(u)=\{v|(u,v)\in \mathcal{E}\}$ as the neighbor set of $u$ and $d_u=|N(u)|$ is the degree of $u$. Let $\mathbf{d}\in \mathbb{R}^n$ be the degree vector with $\mathbf{d}(u)=d_u$. Let \(\mathbf{A} \in \mathbb{R}^{n \times n}\) be the adjacency matrix, where \(\mathbf{A}_{i,j} = \mathbf{A}_{j,i} = 1\) if \((i,j) \in \mathcal{E}\), otherwise \(\mathbf{A}_{i,j} = 0\). Denote by \(\mathbf{D} \in \mathbb{R}^{n \times n}\) the diagonal degree matrix, where the diagonal element \(\mathbf{D}_{i,i} = d_i\), and the off-diagonal elements are 0. The classic random walk matrix is defined as \(\mathbf{P} = \mathbf{AD}^{-1}\), where each element represents the transition probability of 
a random walk on the graph. Let \(\mathbf{e}_s \in \mathbb{R}^n\) be a one-hot vector, which has a value of 1 at node \(s\) and 0 elsewhere. Then, the random-walk related GP vector can be defined as follows. Before we define the GP vector, we first define the matrix function.

\begin{definition}[Matrix Function \cite{DBLP:books/daglib/0086372}]
    Given a matrix $\mathbf{X}\in \mathbb{R}^{n\times n}$ with eigenvalues $\lambda(\mathbf{X})\in [-a,a]$, and an infinitely-differentiable (or smooth) function $f$ on the interval $[-a,a]$ (denote by $f\in C_{\infty}[-a,a]$ for short) with Taylor expansion $f(x)=\sum_{k=0}^{+\infty}{\zeta_kx^k}$ ($\zeta_k$ denotes the Taylor expansion coefficient). Suppose the Taylor expansion of $f$ converges in the range $[-a, a]$. Then, the matrix function $f(\mathbf{X})$ is defined as $f(\mathbf{X}) \triangleq \sum_{k=0}^{+\infty}{\zeta_k\mathbf{X}^k}$.
\label{def:matrix-fun}
\end{definition}

By Definition~\ref{def:matrix-fun}, the random-walk related GP function~\cite{wang2021approximate} (GP function for short) can be defined as follows.  

\begin{definition}[Graph Propagation (GP) Function] 
Given a function $f\in C_{\infty}[-1,1]$ for $x\in[-1,1]$, the Taylor expansion coefficient $\{\zeta_k\}$, the GP function is defined as $f\left(\mathbf{P}\right) \triangleq \sum_{k=0}^{\infty}{\zeta_k\mathbf{P}^k}$.
\label{def:gp-fun}
\end{definition}

Given a one-hot vector $\mathbf{e}_s$, the GP vector w.r.t.\ $\mathbf{e}_s$ is defined as 
\begin{equation}\label{equ:graph diffusion}
\begin{aligned}
\mathbf{y}=f\left(\mathbf{P}\right)\mathbf{e}_s=\sum_{k=0}^{\infty}{\zeta_k\mathbf{P}^k\mathbf{e}_s}.
\end{aligned}
\end{equation}   

Based on Eq.~(\ref{equ:graph diffusion}), the well-known single-source personalized PageRank (\ssppr) vector \cite{andersen2006local} and Heat Kernel PageRank (\hkpr) vector \cite{chung2007heat,kloster2014heat} can be represented as the GP vector. 

\begin{definition} [Single-source Personalized PageRank (\ssppr) vector \cite{andersen2006local}]
    Given a parameter $\alpha\in(0,1)$ and a source node $s$, let $\zeta_k=\alpha (1-\alpha)^k$, an \ssppr vector is defined as 
    \begin{equation}
\bm{\pi}_s=\alpha(\mathbf{I}-(1-\alpha)\mathbf{P})^{-1}\mathbf{e}_s=\alpha\sum_{k=0}^{\infty}{(1-\alpha)^k\mathbf{P}^k\mathbf{e}_s}.
    \end{equation}
\end{definition}
\begin{definition} [Heat Kernel PageRank (\hkpr) vector ~\cite{chung2007heat}]
    Given a parameter $t\in(1,\infty)$ and a source node $s$, let $\zeta_k=e^{-t}\frac{t^k}{k!}$, an \hkpr vector is defined as 
    \begin{equation}
    \bm{\rho}_s=e^{-t(\mathbf{I}-\mathbf{P})}\mathbf{e}_s=e^{-t}\sum_{k=0}^{\infty}{\frac{t^k}{k!}\mathbf{P}^k\mathbf{e}_s}.
    \end{equation}
\end{definition}

The two GP vectors mentioned above are extensively employed in graph analysis tasks \cite{andersen2006local, wang2017fora, kloster2014heat, wang2021approximate}. However, computing these GP vectors on large graphs can be costly. Therefore, similar to previous studies \cite{andersen2006local, wang2017fora, kloster2014heat, wang2021approximate}, we primarily focus on approximate computation of these vectors. To evaluate the approximation quality, we use several widely adopted 
error metrics ~\cite{wang2021approximate,yang2024efficient}.

\begin{definition} \label{def:matrics}
   Consider a positive real-number vector $\mathbf{y}\in \mathbb{R}^n$ with its estimate $\hat{\mathbf{y}}$, several commonly-used error measures are defined as follows:
\begin{itemize} 
\item[$\bullet$] $l_1$-error: $\Vert \mathbf{y}-\hat{\mathbf{y}} \Vert_1 =\sum_{u\in \mathcal{V}}{|\mathbf{y}(u)-\hat{\mathbf{y}}(u)|}$;
\item[$\bullet$] $l_2$-error: $\Vert \mathbf{y}-\hat{\mathbf{y}} \Vert_2 =\sqrt{\sum_{u\in \mathcal{V}}{\left(\mathbf{y}(u)-\hat{\mathbf{y}}(u)\right)^2}}$;
\item[$\bullet$] degree-normalized error: $\Vert \mathbf{D}^{-1}(\mathbf{y}-\hat{\mathbf{y}}) \Vert_\infty =\mathop{\max}\limits_{u\in \mathcal{V}}{\frac{|\mathbf{y}(u)-\hat{\mathbf{y}}(u)|}{d_u}}$;
 

\end{itemize} 
\end{definition} 

While our focus is on computing the \ssppr and \hkpr vectors, it is worth noting that the proposed techniques can be used to compute more general GP vectors (see Section~\ref{subsec:more-general-gp}).

\subsection{Existing Solutions and Their Limitations} \label{subsec:exist-alg}
The state-of-the-art (SOTA) algorithms for approximately computing GP vectors can be roughly categorized into two types: Power-iteration based methods and Push-based approaches.

\comment{
\begin{table*}[t!]
	\centering
	\caption{Comparison of different algorithms for computing \ssppr, \hkpr and general graph propagation. For the truncation step $N$ of \powermethod, $N=\frac{1}{\alpha}\log \frac{1}{\epsilon}$ for \ssppr, $N=2t\log \frac{1}{\epsilon}$ for \hkpr. } 
	\scalebox{1}{
		\begin{tabular}{c|c|c|c|c}
			\toprule
			\multicolumn{1}{c|}{Problems} & \multicolumn{1}{c|}{Algorithms}&
			\multicolumn{1}{c|}{Accuracy Guarantee}&\multicolumn{1}{c|}{Time Complexity}&
			\multicolumn{1}{c}{Remark}\\
			\midrule
            \multirow{5}{*}	{General GP} & \powermethod & $\Vert \mathbf{y}-\hat{\mathbf{y}} \Vert_1<\epsilon$ & $O(Nm)$ & PowerMethod-based\\
            & \ltwocheb(Our) & $\Vert \mathbf{y}-\hat{\mathbf{y}} \Vert_2<\epsilon$ & $O(Km)$ & Chebyshev-based \\
            \cline{2-5}
            & \multirow{2}{*}{\agp ~\cite{wang2021approximate}} & $|\mathbf{y}(u)-\hat{\mathbf{y}}(u)|<\epsilon \mathbf{y}(u)$ & \multirow{2}{*}{$O(\frac{N^3}{\epsilon^2\delta})$}& Push-based, $N$ is a truncation  \\
            & &for any $\mathbf{y}(u)>\delta$ w.h.p.& & step of Taylor expansion\\
            & \chebpush (Our) & $\Vert \mathbf{D}^{-1}(\mathbf{y}-\hat{\mathbf{y}}) \Vert_\infty <\epsilon$ & $O(\frac{C\sqrt{|\bar{S}|}}{\epsilon})$& $C=O(N)$ for \ssppr and \hkpr\\
    \midrule
            \multirow{5}{*}	{\ssppr} & \powermethod \cite{page1999pagerank}& $\Vert \bm{\pi}_s-\hat{\bm{\pi}}_s \Vert_1 <\epsilon$ & $O(\frac{m}{\alpha}\log \frac{1}{\epsilon})$ & PowerMethod-based \\
            & \pwpush~\cite{wu2021unifying}& 	$\Vert \bm{\pi}_s-\hat{\bm{\pi}}_s \Vert_1 <\epsilon$& $O(\frac{m}{\alpha}\log\frac{1}{\epsilon})$&	\push-based\\
	       & \ltwocheb (Our) & $\Vert \bm{\pi}_s-\hat{\bm{\pi}}_s \Vert_2 <\epsilon$ & $O(\frac{m}{\sqrt{\alpha}}\log \frac{1}{\epsilon})$ & Chebyshev-based \\
            \cline{2-5}
            & \push ~\cite{andersen2006local} & $\Vert \mathbf{D}^{-1}(\bm{\pi}_s-\hat{\bm{\pi}}_s) \Vert_\infty <\epsilon$ & $O(\frac{1}{\alpha\epsilon})$ &\push-based\\
	       & \chebpush (Our) & $\Vert \mathbf{D}^{-1}(\bm{\pi}_s-\hat{\bm{\pi}}_s) \Vert_\infty <\epsilon$ & $\tilde{O}(\frac{\sqrt{|\bar{S}|}}{\alpha\epsilon})$ & Subset Chebyshev-based\\
   \midrule
			\multirow{6}{*}	{\hkpr} & \powermethod & $\Vert \bm{\rho}_s-\hat{\bm{\rho}}_s \Vert_1 <\epsilon$ & $O(tm\log \frac{1}{\epsilon})$ & PowerMethod-based \\
                & \ltwocheb (Our) & $\Vert \bm{\rho}_s-\hat{\bm{\rho}}_s \Vert_2 <\epsilon$ & $O(\sqrt{t}m\log \frac{1}{\epsilon})$& Chebyshev-based \\
                \cline{2-5}
                & \agp ~\cite{wang2021approximate} &$|\bm{\rho}_s(u)-\hat{\bm{\rho}}_s(u)|<\epsilon \bm{\rho}_s(u)$  &$\tilde{O}(\frac{t^3}{\epsilon^2\delta})$ &  Randomized Push \\
			& \teaplus ~\cite{yang2019efficient} & for any $\bm{\rho}_s(u)>\delta$ w.h.p.& $\tilde{O}(\frac{t}{\epsilon^2\delta})$& Bidirectional Methods \\
            & \hkrelax ~\cite{kloster2014heat} & $\Vert \mathbf{D}^{-1}(\bm{\rho}_s-\hat{\bm{\rho}}_s) \Vert_\infty <\epsilon$ & $O(\frac{te^t}{\epsilon}\log \frac{1}{\epsilon})$& \push-based \\
	       & \chebpush (Our) & $\Vert \mathbf{D}^{-1}(\bm{\rho}_s-\hat{\bm{\rho}}_s) \Vert_\infty <\epsilon$ & $\tilde{O}(\frac{t\sqrt{|\bar{S}|}}{\epsilon})$ & Subset Chebyshev-based\\
    \bottomrule	
		\end{tabular}
	}
	\label{tab:alg} 
\end{table*}
}

\stitle{Power-iteration based method (\powermethod)} is a fundamental method for GP computation \cite{page1999pagerank}. Specifically, the \powermethod procedure (Algorithm ~\ref{algo:powermethod}) is stated as follows:  First, it sets the initial estimate $\hat{\mathbf{y}} = 0$ and $\mathbf{r}_0 = \mathbf{e}_s$ (Line 1). Then, in each iteration (Lines 2-6), it performs the following updates: (i) $\hat{\mathbf{y}} \leftarrow \hat{\mathbf{y}} + \zeta_k\mathbf{r}_k$ (Line 3), (ii) $\mathbf{r}_{k+1} \leftarrow \mathbf{P}\mathbf{r}_k$ (Line 4), until the procedure stops after the $N$-th iteration (Lines 2 and 6). Finally, it outputs $\hat{\mathbf{y}}$ as the approximation of $\mathbf{y}=f(\mathbf{P})\mathbf{e}_s$. It is easy to show that by setting $ N $ (the truncation step of Taylor expansion) sufficiently large, \powermethod can always produce a high-accurate approximation vector \cite{page1999pagerank,wang2021approximate}. The time complexity of \powermethod is $O(N(m+n))$ since it necessitates exploring the entire graph in each iteration. A significant limitation of \powermethod is its slow convergence to the exact solution, with a convergence speed of $O(N)$. Therefore, to achieve a highly accurate GP vector, often a large truncation step $N$ is required \cite{page1999pagerank, wang2021approximate}, rendering \powermethod very costly for large graphs.

\begin{algorithm}[t!]
\small
	\SetAlgoLined
	\KwIn{A graph $\mathcal{G}=(\mathcal{V},\mathcal{E})$, GP function $f$ with its Taylor expansion coefficients $\{\zeta_k\}$, a source node $s$, and the truncation step $N$}
        $\hat{\mathbf{y}}=\mathbf{0}$; $\mathbf{r}_0=\mathbf{e}_s$;  $k=0$\;
		\While{$k< N$}{
         $\hat{\mathbf{y}}\leftarrow \hat{\mathbf{y}}+ \zeta_k \mathbf{r}_k$\;
		$\mathbf{r}_{k+1}\leftarrow \mathbf{P}\mathbf{r}_k$\;
        $k\leftarrow k+1$\;
		}
	\KwOut{$\hat{\mathbf{y}}$ as the approximation of $\mathbf{y}=f(\mathbf{P})\mathbf{e}_s$}
	\caption{\powermethod\cite{page1999pagerank}}\label{algo:powermethod}
\end{algorithm}

\stitle{Push-based method (\truncatepush)} is a \textit{local} algorithm (Algorithm~\ref{algo:truncate-push}) that can asynchronously prune the ``unimportant'' propagation of \powermethod \cite{andersen2006local}. The key feature of this algorithm is that it only needs to explore a small portion of the graph when computing approximate \ssppr \cite{andersen2006local} or \hkpr \cite{yang2019efficient} vectors. Specifically, \truncatepush first sets the initial approximate solution $\hat{\mathbf{y}}=\mathbf{0}$ and the residual $\mathbf{r}_0=\mathbf{e}_s$ (Line 1). Then, in each iteration (Lines 2-10), it only performs the matrix-vector multiplication $\mathbf{P}\mathbf{r}_k$ on ``important'' nodes. That is, if there exists a node $u$ such that $\mathbf{r}_k(u)> \epsilon_k d_u$ (such $u$ is regarded as an ``important'' node), it performs the following \textit{push} operation: (i) adds $\zeta_k \mathbf{r}_k(u)$ to $\hat{\mathbf{y}}(u)$,  (ii) uniformly distributes $\mathbf{r}_k(u)$ to each neighbor of $u$, namely, $\mathbf{r}_{k+1}(v)\leftarrow \mathbf{r}_{k+1}(v)+\mathbf{r}_k(u)/d_u$ for all $u$'s neighbor $v$, and (iii) sets $\mathbf{r}_k(u)$ to $0$. When the iteration stops (i.e., the number of iterations reached $N$ in Line 2 or there is no ``important'' node to \textit{push} in Line 3), it outputs $\hat{\mathbf{y}}$ as an approximate solution. Although \truncatepush can significantly accelerate \powermethod, its convergence speed remains the same when aiming for a highly accurate GP vector. Indeed, as shown in \cite{wu2021unifying}, \powermethod is equivalent to \truncatepush for achieving high-accuracy \ssppr vector computations. 

In addition, it is worth mentioning that several randomized techniques have also been proposed to speed up \powermethod and \truncatepush for \textit{specific} GP vector computations. Most of these randomized algorithms are based on a bidirectional framework that combines \push (or \powermethod \cite{23sigmodepvr}) with a Monte Carlo sampling technique \cite{wang2017fora, yang2019efficient, wu2021unifying, liao2022efficient, 23sigmodepvr, yang2024efficient, wei2024approximating}. The core idea of these bidirectional randomized algorithms is to use \powermethod or \truncatepush to reduce the variance of the Monte Carlo estimator. Although these algorithms can theoretically achieve relative error guarantees, they often perform worse than \truncatepush when high accuracy is required. Furthermore, these bidirectional algorithms are generally designed for computing either \ssppr or \hkpr, and typically cannot handle the general GP vectors. A more versatile approach is the randomized push algorithm~\cite{wang2021approximate,wang2023singlenode}, which can approximate the computation of general GP vectors. However, similar to \truncatepush, this randomized push algorithm still faces the challenge of low convergence speed. For instance, for GP vector computation, the time complexity of \agp ~\cite{wang2021approximate} is $O(\frac{N^3}{\epsilon^2 \delta})$ to achieve an $\epsilon$-relative error guarantee, where 
$N$ denotes the truncation step in the Taylor expansion (typically requiring a large value for accurate GP vector computation).

To overcome the limitations of existing SOTA methods, in this work, we propose a novel approach based on Chebyshev polynomials to compute general GP vectors, which can significantly boost the convergence speed of both \powermethod and \truncatepush. Similar to \truncatepush, we devise a local and more efficient algorithm using a newly developed subset Chebyshev technique. Moreover, our solution can also be integrated into the bidirectional framework \cite{wang2017fora, yang2019efficient, wu2021unifying, liao2022efficient} (see Section~\ref{subsec:bidirect-method}). Below, we will first introduce a new power iteration method accelerated by Chebyshev polynomials. This method relies on a novel Chebyshev expansion of the GP function. Then, we propose the subset Chebyshev technique, utilizing which we will develop a novel push-style local algorithm.

\begin{algorithm}[t!]
\small
	\SetAlgoLined
	\KwIn{A graph $\mathcal{G}=(\mathcal{V},\mathcal{E})$, GP function $f$ with its Taylor expansion coefficients $\{\zeta_k\}$, a source node $s$, the truncation step $N$, a set of thresholds $\{\epsilon_k\}$}
        $\hat{\mathbf{y}}=\mathbf{0}$; $\mathbf{r}_0=\mathbf{e}_s$ ;  $k=0$ \;
		\While{$k< N$}{
        \For{$u\in \mathcal{V}$ with $\mathbf{r}_k(u)>\epsilon_k d_u$}{
         $\hat{\mathbf{y}}(u)\leftarrow \hat{\mathbf{y}}(u)+ \zeta_k \mathbf{r}_k(u)$ \;
            \For{$v\in \mathcal{N}(u)$}{
		      $\mathbf{r}_{k+1}(v)\leftarrow \mathbf{r}_{k+1}(v)+\mathbf{r}_k(u)/d_u$ \;
            }
            $\mathbf{r}_k(u)\leftarrow 0$ \;
        }
        $k\leftarrow k+1$ \;
		}

	\KwOut{$\hat{\mathbf{y}}$ as the approximation of $\mathbf{y}=f(\mathbf{P})\mathbf{e}_s$}
	\caption{\truncatepush ~\cite{kloster2014heat,yang2019efficient}}\label{algo:truncate-push}
\end{algorithm}

\section{A Chebyshev Power Method}\label{sec:interpretation}
In this section, we first introduce a new concept called Chebyshev expansion of GP function and then give an interpretation of why Chebyshev expansion can boost the convergence speed. Finally, we will develop a new Chebyshev power iteration method based on the Chebyshev expansion, namely \ltwocheb.

\subsection{Chebyshev Expansion of GP Function}\label{subsec:cheby-expan}
Unlike existing methods for computing GP vectors, which are based on Taylor expansions of the GP function, we introduce a novel technique called Chebyshev expansion, leveraging Chebyshev polynomials. Below, we first define Chebyshev polynomials~\cite{mason2002chebyshev}.

\begin{definition}[Chebyshev polynomials]\label{def:chebyshev}
  For any $x\in \mathbb{R}$, the (first kind of) Chebyshev polynomials $\{T_k(x)\}$  is defined  as the following recurrence: (i) $T_0(x)=1$, $T_1(x)=x$; (ii) $T_{k+1}(x)=2xT_k(x)-T_{k-1}(x)$ for $k\geq 1$.
\end{definition}

Since Chebyshev polynomials $\{T_k(x)\}$ form an orthogonal polynomial basis, they are capable of representing any real-valued function \cite{mason2002chebyshev}. 
Therefore, rather than employing Taylor expansion to represent the GP function, we utilize Chebyshev polynomials for this purpose. This representation of the GP function using Chebyshev polynomials is termed as Chebyshev expansion. The advantage of employing Chebyshev expansion is that it can derive a new power iteration method with a faster convergence speed. Since the GP function is a matrix function (not a traditional real-valued function), deriving the Chebyshev expansion is non-trivial. Below, we address this challenge by employing the eigen-decomposition technique. Due to space limits, all the missing proofs can be found in the full version of this paper~\cite{full-version}.

\comment{
Since the Chebyshev polynomials $\{T_k(x)\}$ form an orthogonal polynomial basis (e.g. see Chapter 4.2 in \cite{mason2002chebyshev}), the natural idea is to replace $\mathbf{P}^k\mathbf{e}_s$ in graph propagation with $T_k(\mathbf{P})\mathbf{e}_s$. Let's first get some intuition as to why this operation accelerates the convergence rate. Consider a $50\times 50$ grid graph with source node $s=(25,25)$, Figure ~\ref{fig:power} and ~\ref{fig:chebyshev} depicts the difference behavior between $\mathbf{P}^k\mathbf{e}_s$ and $T_k(\mathbf{P})\mathbf{e}_s$. There are several key observations: (i) $\mathbf{P}^k\mathbf{e}_s$ is a probability distribution, satisfying $\sum_{u\in \mathcal{V}}{\mathbf{P}^k\mathbf{e}_s(u)}=1$ and $\mathbf{P}^k\mathbf{e}_s(u)\geq 0$. But $T_k(\mathbf{P})\mathbf{e}_s$ is not a probability distribution, since $T_k(\mathbf{P})\mathbf{e}_s(u)$ maybe smaller than $0$. (ii) The distribution of $\mathbf{P}^k\mathbf{e}_s$ is more concentrated on the source node $s$, while the distribution of $T_k(\mathbf{P})\mathbf{e}_s$ will concentrated on nodes far away from $s$. Based on these in-depth observations, we use Chebyshev polynomials to devise the following another novel explanation of graph propagation.

\begin{figure}
    \centering
    \subfigure{
		\includegraphics[scale=0.24]{illustrative exp/P^k value distribution.png}}
     \subfigure{
     \includegraphics[scale=0.24]{illustrative exp/P^k distribution on grid.png}}
    
    \caption{$50\times 50$ grid with source node $s=(25,25)$, the distribution of $\mathbf{P}^k \mathbf{e}_s$ with $k=30$. For each node $u=(i,j)$ with $i,j\in[50]$, node index is defined as $idx_u=50*i+j$}\label{fig:power}
\end{figure}

\begin{figure}
    \centering
    \subfigure{
		\includegraphics[scale=0.24]{illustrative exp/T_k(P) value distribution.png}}
     \subfigure{
     \includegraphics[scale=0.24]{illustrative exp/T_k(P) distribution on grid.png}}
    
    \caption{$50\times 50$ grid with source node $s=(25,25)$, the distribution of $T_k(\mathbf{P}) \mathbf{e}_s$ with $k=30$. For each node $u=(i,j)$ with $i,j\in[50]$, node index is defined as $idx_u=50*i+j$}\label{fig:chebyshev}
\end{figure}
}

\begin{lemma}[chebyshev expansion for general GP function]\label{lem:cheby-exapnas}
    The following Chebyshev expansion holds for any graph propagation function $f$:
    \begin{equation}
f(\mathbf{P})=\sum_{k=0}^{+\infty}{c_kT_k(\mathbf{P})},
\end{equation}\label{equ:chebyshev expansion}
    where $c_k=\langle f,T_k\rangle=\frac{2}{\pi}\int_{-1}^{1}{\frac{1}{\sqrt{1-x^2}}f(x)T_k(x)dx}$ if $k\geq 1$; $c_0=\langle f,T_0\rangle =\frac{1}{\pi}\int_{-1}^{1}{\frac{1}{\sqrt{1-x^2}}f(x)T_0(x)dx}$ .
\end{lemma}
\begin{proof}
By the definition of Chebyshev expansion of any function $g\in L_2[-1,1]$ ~\cite{mason2002chebyshev} (where $g\in L_2[-1,1]$ means $\int_{-1}^1{|g|^2}<+\infty$), we have:
    $$g(x)\sim \sum_{k=0}^{+\infty}{c_kT_k(x)}.$$
    Note that the graph propagation function $f$ is infinitely differentiable and continuous. Since $f$ is a continuous function, by Weierstrass approximation theorem (e.g. see Theorem 3.2 in ~\cite{mason2002chebyshev}), $f$ can be arbitrarily approximated by polynomials. Therefore, the equality holds: $f(x)= \sum_{k=0}^{+\infty}{c_kT_k(x)}$. By the property of random-walk matrix, we have $\mathbf{P}=\mathbf{AD}^{-1}=\mathbf{D}^{1/2}\tilde{\mathbf{A}}\mathbf{D}^{-1/2}$, where $\tilde{\mathbf{A}}=\mathbf{D}^{-1/2}\mathbf{A}\mathbf{D}^{-1/2}$. Further, by the property of matrix function, we have $f(\mathbf{P})=\mathbf{D}^{1/2}f(\tilde{\mathbf{A}})\mathbf{D}^{-1/2}$.

    Since $\tilde{\mathbf{A}}$ is a real symmetric matrix, it has eigen-decomposition $\tilde{\mathbf{A}}=\mathbf{U\Lambda}\mathbf{U}^T$ with orthogonal matrix $\mathbf{U}$ and diagonal matrix $\mathbf{\Lambda}$. Thus, 
    by the property of matrix function, we have $f(\tilde{\mathbf{A}})=\mathbf{U}f(\mathbf{\Lambda})\mathbf{U}^T$. As a consequence, we only need to focus on the expansion of $f(\mathbf{\Lambda})$. Since the eigenvalues $\lambda(\mathbf{P})\in [-1,1]$ for undirected graphs, the Chebyshev expansion holds: $f(\mathbf{\Lambda})=\sum_{k=0}^{\infty}{c_kT_k(\mathbf{\Lambda})}$. Therefore, the following equation holds:
    \begin{equation*}
    \begin{aligned}
f(\mathbf{P})&=\mathbf{D}^{1/2}\mathbf{U}f(\mathbf{\Lambda})\mathbf{U}^T\mathbf{D}^{-1/2}\\
&=\mathbf{D}^{1/2}\mathbf{U}\left(\sum_{k=0}^{+\infty}{c_kT_k(\mathbf{\Lambda})}\right)\mathbf{U}^T\mathbf{D}^{-1/2}=\sum_{k=0}^{+\infty}{c_kT_k(\mathbf{P})}.
\end{aligned}
\end{equation*}
This completes the proof.
\end{proof}


Based on Lemma~\ref{lem:cheby-exapnas}, we are able to derive the Chebyshev expansions for both \ssppr and \hkpr vectors. Specifically, by setting  $f(x)=\alpha(1-(1-\alpha)x)^{-1}$ ($f(x)=e^{-t(1-x)}$), we can obtain the Chebyshev expansion of the \ssppr (\hkpr) vector. 


\begin{lemma}[chebyshev expansion for \ssppr]\label{lem:pagerank_expansion} 
    Let $\gamma=\frac{\alpha}{\sqrt{2\alpha-\alpha^2}}$, $\beta=\frac{1-\sqrt{2\alpha-\alpha^2}}{1-\alpha}$. Then, the \ssppr vector can be represented as follows:
    \begin{equation}
    \begin{aligned}
\bm{\pi}_s=\alpha[\mathbf{I}-(1-\alpha)\mathbf{P}]^{-1}\mathbf{e}_s=\gamma\mathbf{e}_s+2\gamma\sum_{k=1}^{\infty}{\beta^k}T_k(\mathbf{P})\mathbf{e}_s.
    \end{aligned}
\end{equation}\label{pagerank:l2aprox}
\end{lemma}
\begin{proof}
    
    Since the PageRank function follows $f(x)=\alpha(1-(1-\alpha)x)^{-1}$ and Chebyshev polynomials follows $T_n(x)=\cos(n\arccos(x))$, by definition ~\ref{lem:cheby-exapnas}, we can express the coefficients $c_n$ as follows:
    $$\frac{\pi}{2}c_n=\int_{-1}^{1}{\frac{1}{\sqrt{1-x^2}}T_n(x)f(x)dx}=\alpha \int_{0}^{\pi}{\frac{\cos(n\theta)}{1-(1-\alpha)\cos\theta}d\theta}, n\geq1$$
with $x=\cos\theta$. For $n=0$, we have:
\begin{equation*}
\begin{aligned}
        \alpha \int_{0}^{\pi}{\frac{1}{1-(1-\alpha)\cos\theta}d\theta}&=\alpha \int_{0}^{\pi}{\frac{1}{\alpha +2(1-\alpha)\sin^2(\frac{\theta}{2})}d\theta} \\
        &=\frac{\alpha\pi}{\sqrt{2\alpha-\alpha^2}}=\pi c_0.
        \end{aligned}
        \end{equation*}
        Since $\int_{0}^{\pi}{\frac{1}{p +\sin^2(\frac{\theta}{2})}d\theta}=\frac{\pi}{\sqrt{p(p+1)}}$ by basic calculus. For $n=1$, similarly, we have: 
        \begin{equation*}
            \begin{aligned}
                \alpha\int_{0}^{\pi}{\frac{\cos\theta}{1-(1-\alpha)\cos\theta}}&= \alpha\int_{0}^{\pi}{\frac{1-\sin^2(\frac{\theta}{2})}{\alpha+2(1-\alpha)\sin^2(\frac{\theta}{2})}d\theta}\\
                &=-\frac{\alpha}{1-\alpha}\pi+\alpha\int_{0}^{\pi}{\frac{1+\alpha/(1-\alpha)}{\alpha+2(1-\alpha)\sin^2(\frac{\theta}{2})}d\theta}\\
            &=    \frac{\alpha}{1-\alpha}\pi\left[-1+\frac{1}{\sqrt{2\alpha-\alpha^2}}  \right]=\frac{\pi}{2}c_1.
            \end{aligned}
        \end{equation*}
        
We use the induction argument to prove the case of $n\geq2$. Suppose that for $n=k$ we have
$$\alpha\int_{0}^{\pi}{\frac{\cos(k\theta)}{1-(1-\alpha)\cos\theta}d\theta}=\frac{\alpha\pi}{\sqrt{2\alpha-\alpha^2}}\left( \frac{1-\sqrt{2\alpha-\alpha^2}}{1-\alpha} \right)^k =\frac{\pi}{2}c_k.$$
Then, for $n=k+1$, by the Chebyshev recurrence, we have
$$\alpha\int_{0}^{\pi}{\frac{\cos((k+1)\theta)}{1-(1-\alpha)\cos\theta}d\theta}=\alpha\int_{0}^{\pi}{\frac{2\cos\theta\cos(k\theta)-\cos((k-1)\theta)}{1-(1-\alpha)\cos\theta}d\theta}.$$
We consider the following integral:
\begin{equation*}
    \begin{aligned}
        &\int_{0}^{\pi}{\frac{\cos\theta\cos(k\theta)}{1-(1-\alpha)\cos\theta}d\theta}\\
        &=-\frac{1}{1-\alpha}\int_0^{\pi}{\cos(k\theta)d\theta}+\left(1+\frac{\alpha}{1-\alpha}\right)\int_0^\pi{\frac{\cos(k\theta)}{1-(1-\alpha)\cos\theta}d\theta}\\
        &=\left(1+\frac{\alpha}{1-\alpha}\right)\frac{\pi}{\sqrt{2\alpha-\alpha^2}}\left( \frac{1-\sqrt{2\alpha-\alpha^2}}{1-\alpha} \right)^k.
    \end{aligned}
\end{equation*}
Based on this, we have:
\begin{equation*}
    \begin{aligned}
        &\alpha\int_{0}^{\pi}{\frac{\cos((k+1)\theta)}{1-(1-\alpha)\cos\theta}d\theta}\\
        &=\frac{\alpha\pi}{\sqrt{2\alpha-\alpha^2}}\left( \frac{1-\sqrt{2\alpha-\alpha^2}}{1-\alpha} \right)^{k-1} \left( \frac{2}{1-\alpha}\frac{1-\sqrt{2\alpha-\alpha^2}}{1-\alpha}-1\right)\\
        &= \frac{\alpha\pi}{\sqrt{2\alpha-\alpha^2}}\left( \frac{1-\sqrt{2\alpha-\alpha^2}}{1-\alpha} \right)^{k+1}=\frac{\pi}{2}c_{k+1}.
    \end{aligned}
\end{equation*}
This completes the proof. 
\end{proof}

\begin{lemma}[chebyshev expansion for \hkpr]\label{lem:hkpr_expansion}
Let $I_n(t)$ be the first type of modified Bessel functions ~\cite{olver2010bessel}, namely, $I_n(t)=\frac{1}{\pi}\int_0^\pi{e^{t\cos\theta}\cos(n\theta)d\theta},n\geq0$. Then, \hkpr can be expanded as follows:
\begin{equation}
    \begin{aligned}
    \bm{\rho}_s=e^{-t(\mathbf{I}-\mathbf{P})}\mathbf{e}_s=e^{-t}I_0(t)\mathbf{e}_s+2e^{-t}\sum_{k=1}^{\infty}{I_k(t)T_k(\mathbf{P})\mathbf{e}_s}.
    \end{aligned}
\end{equation}\label{hk:l2approx}
\end{lemma}
\begin{proof}
See the full version of this paper.
 Since the \hkpr function is $f(x)=e^{-t(1-x)}$,  so the coefficients of the Chebyshev expansion satisfy:
\begin{equation*}
\begin{aligned}
c_n&=\frac{2}{\pi}\int_{-1}^{1}{\frac{1}{\sqrt{1-x^2}}T_n(x)f(x)dx}\\
&=\frac{2}{\pi}e^{-t}\int_0^\pi{e^{t\cos\theta}\cos(n\theta)d\theta}, n\geq1.
\end{aligned}
\end{equation*}
with $x=\cos\theta$. By the definition of the first type of modified Bessel function,  we have:
$$I_n(t)=\frac{1}{\pi}\int_0^\pi{e^{t\cos\theta}\cos(n\theta)d\theta},n\geq0.$$
Therefore, the coefficients of Chebyshev expansion for \hkpr is just $c_n=2e^{-t}I_n(t)$ for $n\geq 1$ and $c_0=e^{-t}I_0(t)$.
\end{proof}

Note that Lemma~\ref{lem:pagerank_expansion} and Lemma~\ref{lem:hkpr_expansion} explicitly provide the Chebyshev coefficients for representing both \ssppr and \hkpr vectors, respectively. These coefficients will be useful in designing our algorithms. It is worth mentioning that we are the first to derive such Chebyshev coefficients for both \ssppr and \hkpr vectors. Below, we give a theoretical interpretation of why Chebyshev expansion can enhance the convergence speed of the algorithm. 

\subsection{Why Chebyshev Expansion} \label{subsec:why-cheby-expan}
Here we prove that by utilizing Chebyshev expansion to represent the GP function, the polynomial degree is lower compared to traditional Taylor expansion methods. This indicates that power iteration and push-style algorithms will achieve faster convergence speeds by using Chebyshev expansion. Below, we first present a useful property of Chebyshev polynomials. Specifically, this property asserts that the $l_2$-norm of $T_k(\mathbf{P})$ is bounded by 1, and the element-wise summation of $T_k(\mathbf{P}) \mathbf{e}_s$ equals 1.

\begin{property}\label{prop:l2-expansion-sum}
    $\Vert T_k(\mathbf{P}) \Vert_2\leq 1$, and $\sum_{u\in \mathcal{V}}{T_k(\mathbf{P}) \mathbf{e}_s(u)}=1$.
\end{property}
\begin{proof}
    First, we prove $\Vert T_k(\mathbf{P})  \Vert_2\leq 1$. By the classical properties of Chebyshev polynomials~\cite{mason2002chebyshev}, we have $\mathop{\max}\limits_{x\in [-1,1]}{|T_k(x)|}=1$. Since $\mathbf{P}$ is a random-walk matrix for undirected graph $\mathcal{G}$, we have $\lambda(\mathbf{P})=\lambda(\mathbf{AD}^{-1})=\lambda(\tilde{\mathbf{A}})\in [-1,1]$, where $\tilde{\mathbf{A}}=\mathbf{D}^{-1/2}\mathbf{A}\mathbf{D}^{-1/2}$ is the normalized Adjacency matrix, and $\lambda(\tilde{\mathbf{A}})$ is the set of eigen-values of $\tilde{\mathbf{A}}$. Therefore, $\Vert T_k(\mathbf{P})  \Vert_2 = \Vert T_k(\tilde{\mathbf{A}})  \Vert_2= \max_{\lambda_i \in \lambda(\tilde{\mathbf{A}})}{|T_k(\lambda_i)|} \leq 1$.

    For $\sum_{u\in \mathcal{V}}{T_k(\mathbf{P}) \mathbf{e}_s(u)}=1$, we prove this by induction. Since $T_0(\mathbf{P})\mathbf{e}_s=\mathbf{e}_s$, $T_1(\mathbf{P})\mathbf{e}_s=\mathbf{P}\mathbf{e}_s$, the property is clearly true for $k=0$ and $1$. We suppose that the property is true for every $l\leq k$. Then, for $l=k+1$, by the Chebyshev recurrence, we have
    $$T_{k+1}(\mathbf{P})\mathbf{e}_s=2\mathbf{P}T_k(\mathbf{P})\mathbf{e}_s - T_{k-1}(\mathbf{P})\mathbf{e}_s.$$

    Since $\mathbf{P}$ is a random-walk matrix, $\mathbf{P}T_k(\mathbf{P})\mathbf{e}_s$ does not change the element-wise summation of $T_k(\mathbf{P})\mathbf{e}_s$:
    $$\sum_{u\in \mathcal{V}}{T_k(\mathbf{P}) \mathbf{e}_s(u)}=\sum_{u\in \mathcal{V}}{[\mathbf{P}T_k(\mathbf{P}) \mathbf{e}_s](u)}.$$
    Therefore, for $l=k+1$, we still have $\sum_{u\in \mathcal{V}}{T_{k+1}(\mathbf{P}) \mathbf{e}_s(u)}=1$. This completes the proof.
\end{proof}

Let $N$ denote the truncation step of the Taylor expansion of the GP function, and $K$ be the truncation step of the Chebyshev expansion. For a general GP function $f$, since Chebyshev expansion is the best $l_2$-approximation of $f$ (e.g., see Theorem 4.1 in ~\cite{mason2002chebyshev}), $K$ is strictly smaller than $N$. For both \ssppr and \hkpr, we prove that $K$ is approximately $O(\sqrt{N})$.

Note that for traditional Taylor expansion, we need to set $N=O(\frac{1}{\alpha}\log{\frac{1}{\epsilon}})$ to achieve $\Vert\hat{\bm{\pi}}_s-\bm{\pi}_s\Vert_1< \epsilon$ for the \ssppr vector ~\cite{wu2021unifying,chen2023accelerating}, and set $N=O(2t\log{\frac{1}{\epsilon}})$ to achieve $\Vert\hat{\bm{\rho}}_s-\bm{\rho_s}\Vert_1< \epsilon$ for the \hkpr vector~\cite{kloster2014heat}. The following lemma shows that $K$ is approximately $O(\sqrt{N})$ for \ssppr and \hkpr to achieve a comparable  error bound. 

\begin{lemma}\label{lem:converge_rate_cheb}
    Let $K= O\left(\frac{1}{\sqrt{\alpha}}\log{\frac{1}{\epsilon}}\right)$, we define the $K$ step truncation of the Chebyshev expansion of \ssppr vector as follows:
    \begin{equation}
\hat{\bm{\pi}}_s=\gamma \mathbf{e}_s+2\gamma\sum_{k=1}^{K}{\beta^k}T_k(\mathbf{P})\mathbf{e}_s.
    \end{equation}
    Then, we have $\Vert\hat{\bm{\pi}}_s-\bm{\pi}_s\Vert_2< \epsilon$.
    
    Similarly, let $K= O\left(\sqrt{ t}\log{\frac{1}{\epsilon}}\right)$, we define the $K$ step truncation of the Chebyshev expansion of \hkpr vector as follows:
    \begin{equation}
\hat{\bm{\rho}}_s=e^{-t}I_0(t)\mathbf{e}_s+2e^{-t}\sum_{k=1}^{K}{I_k(t)T_k(\mathbf{P})\mathbf{e}_s}.
    \end{equation}
    Then, we have $\Vert\hat{\bm{\rho}}_s-\bm{\rho}_s\Vert_2< \epsilon$.
\end{lemma}

\begin{proof}
   First, we prove the case for \ssppr. By Property ~\ref{prop:l2-expansion-sum}, we have $\sum_{u\in V}{T_k(\mathbf{P})\mathbf{e}_s(u)}=1$ for every $k\geq 0$. By the property of \ssppr, we have $\Vert\bm{\pi}_s\Vert_1=1$ and $\bm{\pi}_s(u)\geq 0$ for $u\in \mathcal{V}$. Thus, the coefficients of the Chebyshev expansion sum to 1:        $$\gamma+2\gamma\sum_{k=1}^{\infty}{\beta^k}=1.$$
    Hence, to achieve $\Vert\hat{\bm{\pi}}_s-\bm{\pi}_s\Vert_2<\epsilon$, we just need $2\gamma\sum_{k=K+1}^{\infty}{\beta^k}<\epsilon$. Note that $2\gamma\sum_{k=K+1}^{\infty}{\beta^k}=\beta^{K+1}\frac{2\gamma}{1-\beta}$. When $\alpha$ is sufficiently small, we have :
    $$\frac{2\gamma}{1-\beta}=\frac{\frac{\alpha}{\sqrt{2\alpha-\alpha^2}}}{1-\frac{1-\sqrt{2\alpha-\alpha^2}}{1-\alpha}}=\frac{\alpha(1-\alpha)}{\sqrt{2\alpha-\alpha^2}\left(-\alpha+\sqrt{2\alpha-\alpha^2}\right)}\sim \frac{1}{2}.$$
    
    Therefore, we just need to determine $K$ such that $\beta^{K+1}<2\epsilon$. We take the logarithm on both sides. By the fact that $1-\beta=\frac{\sqrt{2\alpha-\alpha^2}-\alpha}{1-\alpha}\sim \sqrt{2\alpha}$ when $\alpha $ is sufficiently small, we have $K=O\left(\frac{\log{\epsilon}}{\log{\beta}}\right)=O\left(\frac{1}{1-\beta}\log{\frac{1}{\epsilon}}\right)=O\left(\frac{1}{\sqrt{\alpha}}\log{\frac{1}{\epsilon}}\right)$.

    The proof for the case of \hkpr is similar. Similar to \ssppr, by Property ~\ref{prop:l2-expansion-sum} and the definition of \hkpr, the coefficients of Chebyshev expansion for \hkpr sums to 1: $e^{-t}\left[I_0(t)+2\sum_{k=1}^{\infty}{I_k(t)}\right]=1$. Therefore, to determine the order of the truncation step $K$, we just need the partial sum of the coefficients $e^{-t}\left[I_0(t)+2\sum_{k=1}^{K}{I_k(t)}\right]\geq 1-\epsilon$. On the other hand, by the asymptotic property of the Bessel function \cite{olver2010bessel}, we have $e^{-t}I_n(t)\sim \frac{1}{\sqrt{2\pi t}}$ when $n\ll t$. Therefore, for $t\gg 1$ , set the truncation step $K=O\left(\sqrt{2\pi t}\right)$ is enough. This completes the proof.
\end{proof}

Note that since $\log{\frac{1}{\epsilon}}$ is often a small constant, Lemma~\ref{lem:converge_rate_cheb} indicates that $K$ is roughly $O(\sqrt{N})$. These findings suggest that we can achieve an approximately $O(\sqrt{N}) $ acceleration of convergence for the computation of \ssppr and \hkpr by Chebyshev expansion.

\subsection{The Proposed \ltwocheb Algorithm}\label{sec:global-algo}
In this subsection, we develop a new power iteration algorithm based on Chebyshev expansion, called \ltwocheb. Specifically, consider the GP vector  with the following Chebyshev expansion:
\begin{equation}\label{equ:propagation_l2_expansion} \mathbf{y}=f(\mathbf{P})\mathbf{e}_s=\sum_{k=0}^{\infty}{c_kT_k(\mathbf{P})\mathbf{e}_s}.
\end{equation}

Suppose there exists a truncation step $K$ such that $\Vert \mathbf{y}-\hat{\mathbf{y}}\Vert_2<\epsilon$, where $\hat{\mathbf{y}}$ is the $K$ step truncated Chebyshev expansion. We define $\mathbf{r}_k=T_k(\mathbf{P})\mathbf{e}_s$, then the recurrence of $\mathbf{r}_k$ follows the \textit{three-term recurrence} of Chebyshev polynomials:
\begin{eqnarray}\label{equ:res-recurrence}
        \left\{
        \begin{aligned}
        \mathbf{r}_0&=\mathbf{e}_s; \mathbf{r}_1=\mathbf{Pe}_s;\\
        \mathbf{r}_{k+1}&=2\mathbf{P}\mathbf{r}_k-\mathbf{r}_{k-1}, & &  k\geq 1.
        \end{aligned}
        \right.
    \end{eqnarray} 
As a result, we can devise a power-iteration style algorithm to compute the approximate vector $\hat{\mathbf{y}}$ of  $\mathbf{y}=f(\mathbf{P})\mathbf{e}_s$ based on Eq.~(\ref{equ:res-recurrence}). Specifically, in each iteration, we do the following two steps: (i) compute the new $\mathbf{r}_{k+1}$ according to Eq.~(\ref{equ:res-recurrence}); (ii) Update $\hat{\mathbf{y}}\leftarrow \hat{\mathbf{y}}+ c_k \mathbf{r}_k$ until the algorithm stops at step $K$. The detailed pseudocode is shown in Algorithm ~\ref{algo:accurate-chebyshev}.

    \begin{algorithm}[t!]
\small
	\SetAlgoLined
	\KwIn{A graph $\mathcal{G}=(\mathcal{V},\mathcal{E})$, a GP function $f$ with its Chebyshev expansion coefficients$\{c_k\}$, a node $s$, truncation step $K$}
        $\hat{\mathbf{y}}=c_0\mathbf{e}_s$; $\mathbf{r}_0=\mathbf{e}_s$; $\mathbf{r}_1=\mathbf{Pe}_s$;
        $k=1$\;
		\While{$k< K$}{
         $\hat{\mathbf{y}}\leftarrow \hat{\mathbf{y}}+ c_k \mathbf{r}_k$;
		$\mathbf{r}_{k+1}\leftarrow 2\mathbf{P}\mathbf{r}_k-\mathbf{r}_{k-1}$;
        $k\leftarrow k+1$;
		}
$\hat{\mathbf{y}}\leftarrow \hat{\mathbf{y}}+c_K\mathbf{r}_{K}$ \;
	\KwOut{$\hat{\mathbf{y}}$ as the approximation of $\mathbf{y}=f(\mathbf{P})\mathbf{e}_s$}
	\caption{\ltwocheb}\label{algo:accurate-chebyshev}
\end{algorithm}

\stitle{Analysis of Algorithm ~\ref{algo:accurate-chebyshev}}. The error bound of Algorithm ~\ref{algo:accurate-chebyshev} is analyzed in the following theorem.  
\begin{theorem}\label{thm:err-ltwocheb}
    Let $K$ be the truncation step of Chebyshev expansion such that $\sum_{k=K+1}^{+\infty}{|c_k|}<\epsilon$, then the approximation $\hat{\mathbf{y}}$ calculated by Algorithm ~\ref{algo:accurate-chebyshev} satisfy the $l_2$-error bound: $\Vert \mathbf{y}-\hat{\mathbf{y}}\Vert_2<\epsilon$.
\end{theorem}
\begin{proof}
     Since $\mathbf{r}_k=T_k(\mathbf{P})\mathbf{e}_s$ by our definition, the approximation $\hat{\mathbf{y}}=\sum_{k=0}^{K}{c_k\mathbf{r}_k}=\sum_{k=0}^{K}{c_kT_k(\mathbf{P})\mathbf{e}_s}$ is the $K$ step truncation of Chebyshev expansion. Since $\Vert T_k(\mathbf{P})\mathbf{e}_s \Vert_2\leq 1$ by Property ~\ref{prop:l2-expansion-sum}, we have:
    $$\Vert \mathbf{y}-\hat{\mathbf{y}}\Vert_2=\Vert \sum_{k=K+1}^{+\infty}{c_k T_k(\mathbf{P})\mathbf{e}_s} \Vert_2\leq \sum_{k=K+1}^{+\infty}{|c_k|}<\epsilon.$$
    Thus, the $l_2$-error bound $\Vert \mathbf{y}-\hat{\mathbf{y}}\Vert_2<\epsilon$ is established.
\end{proof}
According to Theorem~\ref{thm:err-ltwocheb}, we can easily derive the time complexity of Algorithm ~\ref{algo:accurate-chebyshev}.
\begin{theorem}\label{thm:time-ltwocheb}
    To achieve $\Vert \mathbf{y}-\hat{\mathbf{y}}\Vert_2<\epsilon$, the time overhead of  Algorithm ~\ref{algo:accurate-chebyshev} is $O(K(m+n))$, where $K$ satisfies $\sum_{k=K+1}^{+\infty}{|c_k|}<\epsilon$.
\end{theorem}
\begin{proof}
In each iteration, Algorithm ~\ref{algo:accurate-chebyshev} performs the three-term recurrence $\mathbf{r}_{k+1}\leftarrow 2\mathbf{P}\mathbf{r}_k-\mathbf{r}_{k-1}$ which takes $O(m+n)$ time. There are $K$ iterations in total, thus the time overhead of Algorithm ~\ref{algo:accurate-chebyshev} is $O(K(m+n))$. Note that by Theorem~\ref{thm:err-ltwocheb}, $K$ must satisfy $\sum_{k=K+1}^{+\infty}{|c_k|}<\epsilon$ to achieve the desired $l_2$-error.
\end{proof}

For \ssppr, according to Lemma ~\ref{lem:converge_rate_cheb}, the truncation step $K$ is set to $O\left(\frac{1}{\sqrt{\alpha}}\log{\frac{1}{\epsilon}}\right)$ , thus the time complexity of \ltwocheb is $O\left(\frac{m+n}{\sqrt{\alpha}}\log{\frac{1}{\epsilon}}\right)$.  Similarly, for \hkpr, the truncation step $K$ is set to $O\left(\sqrt{ t}\log{\frac{1}{\epsilon}}\right)$, thus the time complexity of \ltwocheb is $O\left(\sqrt{t}(m+n)\log{\frac{1}{\epsilon}}\right)$. Compared to the traditional power iteration methods \cite{wu2021unifying}, our \ltwocheb algorithm improves the time complexity by a factor $\sqrt{\alpha}$ and $\sqrt{t}$ for \ssppr and \hkpr respectively.

\stitle{Compared to other Chebyshev-polynomial Methods.} We note that Chebyshev polynomial methods have been previously employed to accelerate the computation of ``random walk with restart'' based proximity measures in graphs \cite{coskun2016efficient,cocskun2018indexed}. However, our techniques differ fundamentally from these studies in several aspects. First, the methods proposed in \cite{coskun2016efficient, cocskun2018indexed} are primarily designed for top-$k$ proximity queries. These methods cannot be applied to general GP vector computation (e.g., \hkpr), and also lack detailed theoretical analysis of an $O(\sqrt{N})$ faster convergence speed. Second, our Chebyshev power method differs from those in \cite{coskun2016efficient, cocskun2018indexed}; specifically, their methods require iterative computation of Chebyshev polynomial coefficients. In contrast, our approach, as detailed in Lemma~\ref{lem:pagerank_expansion}, allows for explicit coefficient expressions without iteration. These explicit coefficients facilitate the implementation of power iteration or push-style algorithms for general GP vector computation. Third, the algorithms in \cite{coskun2016efficient, cocskun2018indexed} are not local, whereas based on our Chebyshev expansion technique, we can devise a push-style local algorithm that significantly accelerates computation compared to power iteration methods (see Section~\ref{sec:local-algo}).

\section{A Novel Chebyshev Push Method}\label{sec:local-algo}
In this section, we propose a novel push-style local algorithm, called \chebpush, based on a newly-developed subset Chebyshev recurrence technique. Before introducing our techniques, let us closely examine the time complexity of \ltwocheb. Note that the primary time complexity arises from the matrix-vector multiplication in each iteration: computing $\mathbf{Pr}_k$ takes $O(m)$ operations, and updating $\mathbf{r}_{k+1}=2\mathbf{P}\mathbf{r}_k-\mathbf{r}_{k-1}$ along with $\hat{\mathbf{y}}\leftarrow \hat{\mathbf{y}}+ c_k \mathbf{r}_k$ takes $O(n)$ operations. Thus, \ltwocheb is inherently a \textit{global} algorithm since it necessitates exploring the entire graph in every iteration.

To transform \ltwocheb into a local algorithm, a natural question arises: can we execute the Chebyshev recurrence $\mathbf{r}_{k+1}=2\mathbf{P}\mathbf{r}_k-\mathbf{r}_{k-1}$ \emph{partially}? Alternatively, can we implement the Chebyshev recurrence with fewer than $O(m+n)$ operations while maintaining rapid convergence? One potential approach involves mimicking the structure of \truncatepush by truncating certain ``less significant'' nodes during the matrix-vector multiplication $\mathbf{P}\mathbf{r}_k$. However, due to the inherent structural differences between $\mathbf{P}^k\mathbf{e}_s$ and $T_k(\mathbf{P})\mathbf{e}_s$, achieving a local implementation of \ltwocheb appears considerably more challenging than that of \powermethod. To address these challenges, we develop a novel and powerful technique termed subset Chebyshev recurrence, demonstrating how to locally implement Chebyshev recurrence on graphs.

\subsection{Subset Chebyshev Recurrence} \label{subsec:subset-cheby-recur}
As discussed previously, executing the following Chebyshev recurrence takes $O(m+n)$ time:
\begin{equation}\label{eq:residual_accurate_recurrence}
\mathbf{r}_{k+1}=2\mathbf{P}\mathbf{r}_k-\mathbf{r}_{k-1}.
\end{equation}

To derive a local algorithm, we aim for the operations of the three-term recurrence to be independent of $m$ and $n$ while ensuring that the vector $\mathbf{r}_k$ remains sparse at each iteration. To achieve this objective, we propose the following subset Chebyshev recurrence.

\begin{definition}{(Subset Chebyshev Recurrence)}{}\label{def:purning-chebyshev}
Given a source node $s$, the random walk matrix $\mathbf{P}$, a series of node sets $\{S_k\}_k$ with $S_k\subseteq \mathcal{V}$. The subset Chebyshev recurrence is defined as follows:
\begin{eqnarray}\label{eq:subset-cheby-recurrence}
        \left\{
        \begin{aligned}
        \hat{\mathbf{r}}_0&=\mathbf{e}_s;  \hat{\mathbf{r}}_1=\mathbf{Pe}_s; \\
        \hat{\mathbf{r}}_{k+1}&=2\mathbf{P}  (\hat{\mathbf{r}}_k|_{S_k})-\hat{\mathbf{r}}_{k-1}|_{S_{k-1}} +  \hat{\mathbf{r}}_{k-1}|_{\mathcal{V}-S_{k-1}}, & & k\geq 1,
        \end{aligned}
        \right.
    \end{eqnarray}
    where $\hat{\mathbf{r}}_k|_{S_k}\in \mathbb{R}^n$ represents the vector $\hat{\mathbf{r}}_k$ constrained to the node subset $S_k$. That is, for any $u\in S_k$ , $\hat{\mathbf{r}}_k|_{S_k}(u)=\hat{\mathbf{r}}_k(u)$; for any $u\in \mathcal{V}- S_k$, $\hat{\mathbf{r}}_k|_{S_k}(u)=0$.
\end{definition}

Note that by Definition~\ref{def:purning-chebyshev}, the computation of the Subset Chebyshev Recurrence is constrained to the node subsets $S_k$, enabling its local implementation (see Section~\ref{subsec:chebypush}). Compared to the Chebyshev recurrence (Eq.~(\ref{eq:residual_accurate_recurrence})), the subset Chebyshev recurrence includes an additional term $\hat{\mathbf{r}}_{k-1}|_{\mathcal{V}-S_{k-1}}$, which plays a crucial role in bounding the approximation error of our algorithm. Clearly, when $S_k=\mathcal{V}$ for all $k$, the subset Chebyshev recurrence is equivalent to the exact Chebyshev recurrence.  

\stitle{Analysis of the subset Chebyshev recurrence.}
Here, we establish a connection between $\{\hat{\mathbf{r}}_k\}$ derived from the subset Chebyshev recurrence and $\{\mathbf{r}_k\}$ obtained from the exact Chebyshev recurrence. For our analysis, we define the \textit{deviation} produced at step $k$ as $\delta_k \triangleq \hat{\mathbf{r}}_k-\hat{\mathbf{r}}_k|_{S_k}=\hat{\mathbf{r}}_k|_{\mathcal{V}-S_k}$. Then, the three-term subset Chebyshev recurrence can be reformulated as:
\begin{equation} \label{eq:residual-prop}
\hat{\mathbf{r}}_{k+1}=2\mathbf{P}(\hat{\mathbf{r}}_k-\delta_k)-\hat{\mathbf{r}}_{k-1} +  2\delta_{k-1}.   
\end{equation}

Note that by Eq.~(\ref{eq:residual-prop}), the deviation can propagate across different iterations. Figure ~\ref{fig:subset-chebyshev} illustrates the deviation propagation of Eq.~(\ref{eq:residual-prop}). More specifically, the first two terms of the subset Chebyshev recurrence are defined as $\hat{\mathbf{r}}_0=\mathbf{e}_s$ and $\hat{\mathbf{r}}_1=\mathbf{Pe}_s$. For $k=2$, we obtain the three-term recurrence: $\hat{\mathbf{r}}_{2}=2\mathbf{P}(\hat{\mathbf{r}}_1-\delta_1)-\hat{\mathbf{r}}_{0}$ due to $\delta_0=0$. This means that for the deviation $\delta_1$, there are $-2\mathbf{P}\delta_1$ propagated  from $\hat{\mathbf{r}}_1$ to $\hat{\mathbf{r}}_2$. For $k=3$, we derive the three-term recurrence: $\hat{\mathbf{r}}_{3}=2\mathbf{P}(\hat{\mathbf{r}}_2-\delta_2)-\hat{\mathbf{r}}_{1} + 2 \delta_1$. This indicates that for the deviation $\delta_2$, $-2\mathbf{P}\delta_2$ is propagated from $\hat{\mathbf{r}}_2$ to $\hat{\mathbf{r}}_3$, and at the same time we ``\textbf{compensate}'' $2\delta_1$ from $\hat{\mathbf{r}}_1$ to $\hat{\mathbf{r}}_3$. Note that this ``compensate'' operation is important to establish the connection between $\{\hat{\mathbf{r}}_k\}$ and $\{\mathbf{r}_k\}$. The following lemma demonstrates that after performing this step, the difference between $\{\hat{\mathbf{r}}_k\}$ and $\{\mathbf{r}_k\}$ can be characterized by using Chebyshev polynomials.

\begin{figure}
    \centering
    \includegraphics[scale=0.27]{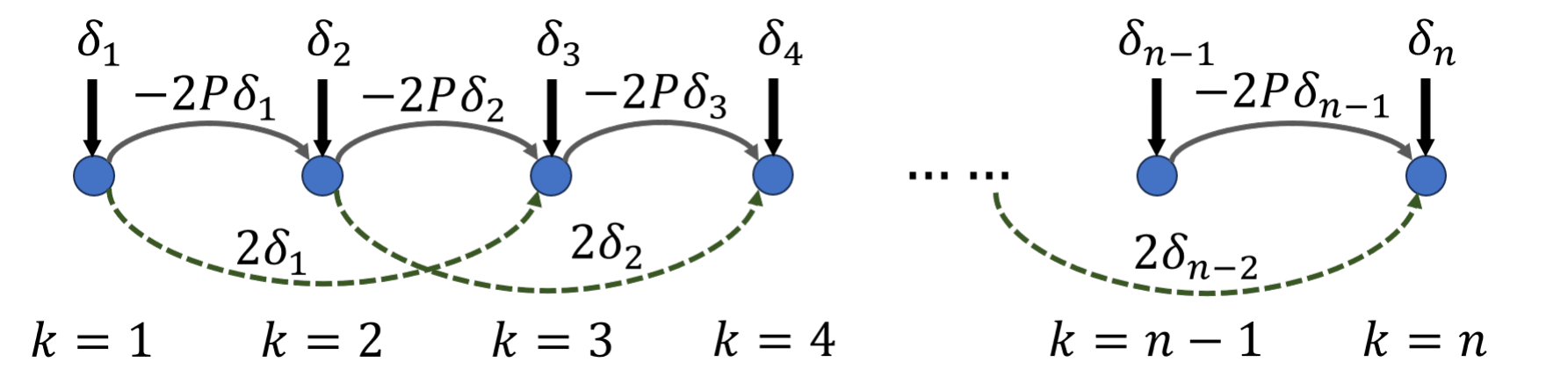} \vspace{-0.3cm}
    \caption{Illustration of the deviation propagation}\vspace{-0.5cm}
    \label{fig:subset-chebyshev}
\end{figure}

\begin{lemma}\label{lem:residual-propagation}
Let $\{\hat{\mathbf{r}}_k\}$ be derived by the subset Chebyshev recurrence (Definition  ~\ref{def:purning-chebyshev}) and $\{\mathbf{r}_k\}$ be obtained by exact Chebyshev recurrence. Then, the relationship between $\mathbf{r}_k$ and $\hat{\mathbf{r}}_k$ is:
\begin{equation}
\mathbf{r}_k=\hat{\mathbf{r}}_k+2\sum_{l=1}^{k-1}{T_{k-l}(\mathbf{P})\mathbf{\delta}_l},
    \end{equation}
    where $\delta_l:=\hat{\mathbf{r}}_l|_{\mathcal{V}-S_l}$ is the deviation produced at step $l$ with $l<k$.
\end{lemma}

\begin{proof}
    We prove this lemma by induction. For the basic case $n=1$, by Definition ~\ref{def:purning-chebyshev}, $\mathbf{r}_1=\hat{\mathbf{r}}_1=\mathbf{Pe}_s$, the lemma clearly holds. For $n=2$, by the three-term subset Chebyshev recurrence, we have:
    \begin{equation}
    \hat{\mathbf{r}}_{2}=2\mathbf{P}(\hat{\mathbf{r}}_1-\delta_1)-\hat{\mathbf{r}}_{0}= \mathbf{r}_2- 2\mathbf{P}\delta_1= \mathbf{r}_2- 2T_1(\mathbf{P})\delta_1.
    \end{equation}
    Thus, the lemma also holds. Similarly, for $n=3$, by the three-term subset Chebyshev recurrence, we have:
   \begin{equation}
   \begin{aligned}
    \hat{\mathbf{r}}_{3}&=2\mathbf{P}(\hat{\mathbf{r}}_2-\delta_2)-\hat{\mathbf{r}}_{1}+2\delta_1\\
    &= 2\mathbf{P}(\mathbf{r}_2-2T_1(\mathbf{P})\delta_1-\delta_2)-\mathbf{r}_{1}+2\delta_1 \\
    &=\mathbf{r}_3-(2\mathbf{P}(2T_1(\mathbf{P}))-2\mathbf{I})\delta_1-2\mathbf{P}\delta_2 \\
    &=\mathbf{r}_3-2T_2(\mathbf{P})\delta_1-2T_1(\mathbf{P})\delta_2.
    \end{aligned}
    \end{equation}
    Thus, the lemma also holds for $n=3$. Suppose for $n\leq k$ with $k\geq 4$, the following equation holds: $\mathbf{r}_k=\hat{\mathbf{r}}_k+2\sum_{l=1}^{k-1}{T_{k-l}(\mathbf{P})\mathbf{\delta}_l}$. Then, for $n=k+1$, according to the Chebyshev recurrence $\mathbf{r}_{k+1}=2\mathbf{P}\mathbf{r}_k-\mathbf{r}_{k-1}$, we have
    \begin{equation}\label{equ:r_k+1}\small
    \begin{aligned}
            \mathbf{r}_{k+1}=2\mathbf{P}\left[\hat{\mathbf{r}}_k+2\sum_{l=1}^{k-1}{T_{k-l}(\mathbf{P})\mathbf{\delta}_l} \right] 
            -\hat{\mathbf{r}}_{k-1}-2\sum_{l=1}^{k-2}{T_{k-l-1}(\mathbf{P})\mathbf{\delta}_l}.\\
        \end{aligned}
    \end{equation}

    We use the subset Chebyshev recurrence $\hat{\mathbf{r}}_{k+1}=2\mathbf{P}(\hat{\mathbf{r}}_k-\delta_k)-\hat{\mathbf{r}}_{k-1} +  2\delta_{k-1}$ to replace the term $\{\hat{\mathbf{r}}_{k}\}$ in the right hand side of Eq.~(\ref{equ:r_k+1}). Then, we have
    \begin{equation}
        \begin{aligned}
            \mathbf{r}_{k+1}&=\hat{\mathbf{r}}_{k+1}+2\mathbf{P}\left[\delta_k+2\sum_{l=1}^{k-1}{T_{k-l}(\mathbf{P})\mathbf{\delta}_l} \right] 
           -2\sum_{l=1}^{k-2}{T_{k-l-1}(\mathbf{P})\mathbf{\delta}_l}-2\delta_{k-1}\\
           &=\hat{\mathbf{r}}_{k+1}+2\mathbf{P}\delta_k + (2\mathbf{P}(2T_1(\mathbf{P}))-2\mathbf{I})\delta_{k-1}\\
           &+2\sum_{l=1}^{k-2}{(2\mathbf{P}T_{k-l}(\mathbf{P})-T_{k-l-1}(\mathbf{P}))\delta_l}\\
           &=\hat{\mathbf{r}}_{k+1}+2T_1(\mathbf{P})\delta_k+2T_2(\mathbf{P})\delta_{k-1}+2\sum_{l=1}^{k-2}{T_{k+1-l}(\mathbf{P}))\delta_l}\\
           &=\hat{\mathbf{r}}_{k+1}+2\sum_{l=1}^{k}{T_{k+1-l}(\mathbf{P})\mathbf{\delta}_l},
        \end{aligned}
    \end{equation}
    where the third equality holds due to $T_1(\mathbf{P})=\mathbf{P}$ and $T_2(\mathbf{P})=2\mathbf{P}T_1(\mathbf{P})-\mathbf{I}$.
    This completes the proof.
\end{proof}

\comment{
\begin{proof}
By Eq.~(\ref{eq:residual-prop}), we can clearly see that $\hat{\mathbf{r}}_k$ is directly related to $\delta_{k-1}$ and $\delta_{k-2}$. By recursively executing  Eq.~(\ref{eq:residual-prop}), it is easy to check that $\hat{\mathbf{r}}_k$ is  dependent on $\delta_{l}$ for every $l< k$. To prove the lemma, we assume without loss of generality that the following equality holds:
     \begin{equation}\label{eq:poly-to-determine}
        \mathbf{r}_k=\hat{\mathbf{r}}_k+\sum_{l=1}^{k-1}{\Psi_{k-l}^{(l)}(\mathbf{P})\mathbf{\delta}_l},
    \end{equation}
    where $\{\Psi_{k-l}^{(l)}(x)\}_{k:k\geq l}$ is some polynomials to be determined. 
It remains to prove that $\Psi_{k-l}^{(l)}(x)$ is exactly a Chebyshev polynomial.
Recall that the exact Chebyshev recurrence is $\mathbf{r}_{k+1}=2\mathbf{P}\mathbf{r}_k-\mathbf{r}_{k-1}$. 
Using Eq.~(\ref{eq:poly-to-determine}) to replace both sides of the exact Chebyshev recurrence, we have
    \begin{equation}\label{eq:take-in}
    \begin{aligned}
        \hat{\mathbf{r}}_{k+1}+\sum_{l=1}^{k}{\Psi_{k+1-l}^{(l)}(\mathbf{P})\mathbf{\delta}_l}=2\mathbf{P}\left[\hat{\mathbf{r}}_k+\sum_{l=1}^{k-1}{\Psi_{k-l}^{(l)}(\mathbf{P})\mathbf{\delta}_l} \right] \\
        - \hat{\mathbf{r}}_{k-1}-\sum_{l=1}^{k-2}{\Psi_{k-l-1}^{(l)}(\mathbf{P})\mathbf{\delta}_l}.
    \end{aligned}
    \end{equation}

    We use the subset Chebyshev recurrence $\hat{\mathbf{r}}_{k+1}=2\mathbf{P}(\hat{\mathbf{r}}_k-\delta_k)-\hat{\mathbf{r}}_{k-1} +  2\delta_{k-1}$ to replace the term $\{\hat{\mathbf{r}}_{k+1}\}$ in the left side of Eq.~(\ref{eq:take-in}). Then, we have

    \begin{equation} \label{eq:poly-equal}
    \begin{aligned}
        \sum_{l=1}^{k}{\Psi_{k+1-l}^{(l)}(\mathbf{P})\mathbf{\delta}_l}=2\mathbf{P}\left[\delta_k+ \sum_{l=1}^{k-1}{\Psi_{k-l}^{(l)}(\mathbf{P})\mathbf{\delta}_l}\right]  \\
        -\sum_{l=1}^{k-2}{\Psi_{k-l-1}^{(l)}(\mathbf{P})\mathbf{\delta}_l}-2\delta_{k-1}.
    \end{aligned}
    \end{equation}
    
    We compare the coefficients of $\delta_l$ from both sides of Eq.~(\ref{eq:poly-equal}), and can obtain the recurrence of $\Psi_{k-l}^{(l)}(\mathbf{P})\delta_l$ as follows:
    \begin{eqnarray}
        \left\{
        \begin{aligned}
        \Psi_1^{(k)}(\mathbf{P})\delta_{k}&=2\mathbf{P}\delta_k;\\
        \Psi_2^{(k-1)}(\mathbf{P})\delta_{k-1}&=2\mathbf{P}\Psi_1^{(k-1)}(\mathbf{P})\delta_{k-1}-2 \delta_{k-1}; \\
        \Psi_{k+1-l}^{(l)}(\mathbf{P})\delta_l&=2\mathbf{P}\Psi_{k-l}^{(l)}(\mathbf{P})\delta_l-\Psi_{k-1-l}^{(l)}(\mathbf{P})\delta_l,\ 1\leq l \leq k-2 .
        \end{aligned}
        \right.
    \end{eqnarray}

    Therefore, the recurrence of $\Psi_{k-l}^{(l)}(x)$ satisfy
    \begin{eqnarray} \label{eq:cheby}
        \left\{
        \begin{aligned}
        \Psi_1^{(l)}(x)&=2x;  \Psi_2^{(l)}(x)=2x\Psi_1^{(l)}(x)-2; \\
        \Psi_{k+1-l}^{(l)}(x)&=2x\Psi_{k-l}^{(l)}(x)-\Psi_{k-1-l}^{(l)}(x), & &  k\geq l+2.
        \end{aligned}
        \right.
    \end{eqnarray}
    Eq.~\ref{eq:cheby} indicates that  $\Psi_{k-l}^{(l)}(x)$ is indeed a Chebyshev polynomial (with a multiplication factor 2), thus the lemma is established.
\end{proof}
}
Lemma~\ref{lem:residual-propagation} provides a useful theoretical result for analyzing the error introduced by the subset Chebyshev recurrence. If the deviation $\delta_l$ is bounded for every $l$, the difference between $\{\hat{\mathbf{r}}_k\}$ and $\{\mathbf{r}_k\}$ can also be bounded, as the norm of the Chebyshev polynomial over the random-walk matrix $\mathbf{P}$ (i.e., $T_{k-l}(\mathbf{P})$) is bounded. 

\subsection{\chebpush: Chebyshev Push Method} \label{subsec:chebypush}

 Recall that for the subset Chebyshev recurrence (Definition ~\ref{def:purning-chebyshev} ), it performs $\hat{\mathbf{r}}_{k+1}=2\mathbf{P}(\hat{\mathbf{r}}_k|_{S_k})-\hat{\mathbf{r}}_{k-1}|_{S_{k-1}} +  \hat{\mathbf{r}}_{k-1}|_{\mathcal{V}-S_{k-1}}$ in the $k^{th}$ iteration. During this iterative process, it seems that we need to explore every node $u\in \mathcal{V}$. However, as we will show below, it suffices to explore only the nodes $u\in S_k$ (no need to traverse the nodes in ${\mathcal{V}-S_{k-1}}$). To achieve this, we maintain two vectors: $\mathbf{r}_{cur}$ and $\mathbf{r}_{new}$. Initially, we set $\mathbf{r}_{cur}=\mathbf{Pe}_s$ and $\mathbf{r}_{new}=-\mathbf{e}_s$. For the $k^{th}$ iteration, we perform the following operations:
\begin{eqnarray}\label{eq:subset-cheby-implement}
        \left\{
        \begin{aligned}
        &\mathbf{r}_{new}=\mathbf{r}_{new}+2\mathbf{P}(\mathbf{r}_{cur}|_{S_k}), \\
        &\mathbf{r}_{cur}|_{S_k}=-\mathbf{r}_{cur}|_{S_k}, \\
        &\mathbf{r}_{cur}.swap(\mathbf{r}_{new}),
        \end{aligned}
        \right.
    \end{eqnarray}
where $\mathbf{r}_{cur}.swap(\mathbf{r}_{new})$ denotes swapping the values between $\mathbf{r}_{cur}$ and $\mathbf{r}_{new}$. By Eq.~(\ref{eq:subset-cheby-implement}), the algorithm only needs to traverse nodes $u \in S_k$ in each iteration, allowing us to implement the subset Chebyshev recurrence locally without the need to search the entire graph. We prove that Eq.~(\ref{eq:subset-cheby-implement}) exactly follows the subset Chebyshev recurrence (i.e., Eq.~(\ref{eq:subset-cheby-recurrence})) in each iteration.

\begin{lemma}\label{lem:accurate-recurrence}
    Initially, set $\mathbf{r}_{cur}=\mathbf{Pe}_s$ and $\mathbf{r}_{new}=-\mathbf{e}_s$. After the $k^{th}$ iteration of Eq.~(\ref{eq:subset-cheby-implement}), the following equation holds:  $\mathbf{r}_{cur}=\hat{\mathbf{r}}_{k+1}$ and $\mathbf{r}_{new}=-\hat{\mathbf{r}}_{k}|_{S_k}+\hat{\mathbf{r}}_{k}|_{\mathcal{V}-S_k}$, where $\{\hat{\mathbf{r}}_k\}$ is obtained by the subset Chebyshev recurrence.
\end{lemma}
\begin{proof}
    We prove this lemma by an induction argument. First, for the initial case, since $\mathbf{r}_{cur}=\mathbf{Pe}_s=\hat{\mathbf{r}}_1,\mathbf{r}_{new}=-\mathbf{e}_s=-\hat{\mathbf{r}}_0$, the lemma clearly holds. Assume that $\mathbf{r}_{cur}=\hat{\mathbf{r}}_{k}$ and $\mathbf{r}_{new}=-\hat{\mathbf{r}}_{k}|_{S_{k-1}}+\hat{\mathbf{r}}_{k}|_{\mathcal{V}-S_{k-1}}$ hold after the $(k-1)^{th}$ iteration. Then, by Eq.~(\ref{eq:subset-cheby-implement}), we perform $\mathbf{r}_{new}= \mathbf{r}_{new}+2\mathbf{P}(\hat{\mathbf{r}}_k|_{S_k})$ and $\mathbf{r}_{cur}|_{S_k}=-\mathbf{r}_{cur}|_{S_k}$ which can obtain $\mathbf{r}_{new}=\hat{\mathbf{r}}_{k+1}$ (by Eq.~(\ref{eq:subset-cheby-recurrence})) and  $\mathbf{r}_{cur}=-\hat{\mathbf{r}}_{k}|_{S_k}+\hat{\mathbf{r}}_{k}|_{\mathcal{V}-S_k}$. Subsequently, we swap  $\mathbf{r}_{cur}$ and $\mathbf{r}_{new}$. As a result, after the $k^{th}$ iteration, the following equation holds:  $\mathbf{r}_{cur}=\hat{\mathbf{r}}_{k+1}$ and $\mathbf{r}_{new}=-\hat{\mathbf{r}}_{k}|_{S_k}+\hat{\mathbf{r}}_{k}|_{\mathcal{V}- S_k}.$ This completes the proof.
\end{proof}

Based on Eq.~(\ref{eq:subset-cheby-implement}) and Lemma~\ref{lem:accurate-recurrence}, we are ready to design a local algorithm to implement the subset Chebyshev recurrence. Our algorithm, namely \chebpush, is detailed in Algorithm~\ref{algo:approx-chebyshev}. Specifically, Algorithm~\ref{algo:approx-chebyshev} first sets a small threshold $\epsilon_a\ll 1$ and sets a $k$-step threshold $\epsilon_k=\frac{1}{\sum_{l=k}^{K}{|c_{l}|}}\frac{ \epsilon_a}{4K}$ for each iteration $k$ (Line~4). For the subset Chebyshev recurrence, the algorithm sets each node subset as $S_k=\{u\in \mathcal{V}:|\hat{\mathbf{r}}_k(u)|>\epsilon_k d_u\}$ for every $k$ (Line~5), and iteratively performs the operations defined in Eq.~(\ref{eq:subset-cheby-implement}) until the truncation step equals $K$ (Lines~3-14). Initially, by Lemma~\ref{lem:accurate-recurrence}, the algorithm sets $\hat{\mathbf{y}}=c_0\mathbf{e}_s$, and $\mathbf{r}_{cur}=\mathbf{Pe}_s,\mathbf{r}_{new}=-\mathbf{e}_s$ corresponding to $\hat{\mathbf{r}}_1,-\hat{\mathbf{r}}_0$ respectively (Line~2). Then, in each iteration $k$, for every $u\in \mathcal{V}$ with $|\mathbf{r}_{cur}(u)|^2>\epsilon_k d_u$ (i.e., $u\in S_k$), the algorithm performs the following \textit{push} operation based on Eq.~(\ref{eq:subset-cheby-implement}): (i) adds $c_k\mathbf{r}_{cur}(u)$ to $\hat{\mathbf{y}}$ (Line 6); (ii) distributes $2\mathbf{r}_{cur}(u)$ to each $u$'s neighbor $v$, i.e., adding $\mathbf{r}_{new}(v)\leftarrow \mathbf{r}_{new}(v)+2\frac{\mathbf{r}_{cur}(u)}{d_u}$ for $v\in \mathcal{N}(u)$ (Lines 7-9); (iii) converts $\mathbf{r}_{cur}(u)$ to $-\mathbf{r}_{cur}(u)$ (Line 10). After that, the algorithm swaps $\mathbf{r}_{cur}$ and $\mathbf{r}_{new}$ (Line 12) following Eq.~(\ref{eq:subset-cheby-implement}). Finally, the algorithm terminates after executing $K$ iterations and outputs $\hat{\mathbf{y}}$ as the approximation of the GP vector $\mathbf{y}=f(\mathbf{P})\mathbf{e}_s$. Note that by Lemma~\ref{lem:accurate-recurrence}, $\mathbf{r}_{cur}$ consistently  equals $\hat{\mathbf{r}}_k$ across all iterations. This guarantees that \chebpush correctly implements the subset Chebyshev recurrence to approximate the GP vector.

\begin{algorithm}[t!]
\small
	\SetAlgoLined
	\KwIn{A graph $\mathcal{G}=(\mathcal{V},\mathcal{E})$, a GP function $f$ with its Chebyshev expansion coefficients$\{c_k\}$, a node $s$, truncation step $K$, threshold $\epsilon_a$}
        $k=1$ \;
        $\hat{\mathbf{y}}=c_0\mathbf{e}_s$ ; $\mathbf{r}_{cur}=\mathbf{Pe}_s$ ; $\mathbf{r}_{new}=-\mathbf{e}_s$\;
		\While{$k< K$}{
        $\epsilon_k=\frac{1}{\sum_{l=k}^{K}{|c_{l}|}}\frac{ \epsilon_a}{4K}$ \;
        \For{$u\in \mathcal{V}$ with $|\mathbf{r}_{cur}(u)|> \epsilon_k d_u$}{
         $\hat{\mathbf{y}}(u)\leftarrow \hat{\mathbf{y}}(u)+ c_k \mathbf{r}_{cur}(u)$ \;
            \For{$v\in N(u)$}{
		$\mathbf{r}_{new}(v)\leftarrow \mathbf{r}_{new}(v)+2\frac{\mathbf{r}_{cur}(u)}{d_u}$\;
                }
                $\mathbf{r}_{cur}(u)\leftarrow -\mathbf{r}_{cur}(u)$
        }
        $\mathbf{r}_{cur}.swap(\mathbf{r}_{new})$ \;
        $k\leftarrow k+1$ \;
		}
	\KwOut{$\hat{\mathbf{y}}$ as the approximation of $\mathbf{y}=f(\mathbf{P})\mathbf{e}_s$}
	\caption{\chebpush}\label{algo:approx-chebyshev}
\end{algorithm}

\comment{
\stitle{Local Implementation of \chebpush.} Note that the implementation of \chebpush does not require searching the whole graph. Specifically, we define the invoking set $S_k$ for iteration $k$ with $S_1=\{\mathcal{N}(s)\}$. At each iteration $k$, when executing Line 8 of algorithm ~\ref{algo:approx-chebyshev}, we check: if $|\mathbf{r}_{new}(v)|^2> \epsilon_{k+1} d_v$ and $v$ is not yet in $S_{k+1}$, then add $v$ to $S_{k+1}$; if $|\mathbf{r}_{new}(v)|^2< \epsilon_{k+1} d_v$ and $v$ is already in $S_{k+1}$, then remove $v$ from $S_{k+1}$. Moreover, when executing Line 10 of algorithm ~\ref{algo:approx-chebyshev}, if $|\mathbf{r}_{cur}(u)|^2> \epsilon_{k+2} d_u$, then we add $u$ to $S_{k+2}$. So to implement \chebpush, we just need to invoke $S_k$ at each iteration $k$ without searching the whole graph.
}

\stitle{An illustrative example.} Consider a toy graph shown in Figure  ~\ref{fig:exp-chebpush} with $\mathcal{V}=\{v_0,v_1,v_2,v_3\}$, $\mathcal{E}=\{(v_0,v_1),(v_0,v_2),(v_0,v_3)\}$. Suppose that our goal is to compute $f(\mathbf{P})\mathbf{e}_s=T_3(\mathbf{P})\mathbf{e}_{v_1}$, we set the threshold $\epsilon_k=1/3$ for $k\geq 1$. Initially, we have $\mathbf{r}_{cur}=\mathbf{Pe}_{v_1}=[1,0,0,0]$ and $\mathbf{r}_{new}=-\mathbf{e}_{v_1}=[0,-1,0,0]$. For the first iteration $k=1$, we perform a \textit{push} operation on $v_0$. That is, for each $v_0$'s neighbor: $v_1,v_2,v_3$, we update $\mathbf{r}_{new}$ as: $\mathbf{r}_{new}(v_i)\leftarrow \mathbf{r}_{new}(v_i)+2\frac{\mathbf{r}_{cur}(v_0)}{d_{v_0}}$ for $i=1,2,3$. Then, we convert $\mathbf{r}_{cur}(v_0)$ to $-\mathbf{r}_{cur}(v_0)$. After the \textit{push} operation, we have $\mathbf{r}_{cur}=[-1,0,0,0]$ and $\mathbf{r}_{new}=[0,-\frac{1}{3},\frac{2}{3}, \frac{2}{3}]$. Then, we swap $\mathbf{r}_{cur}$ and  $\mathbf{r}_{new}$, and turn to the next iteration. For the next iteration $k=2$, initially we have $\mathbf{r}_{cur}=[0,-\frac{1}{3},\frac{2}{3}, \frac{2}{3}]$ and $\mathbf{r}_{new}=[-1,0,0,0]$. Since $\epsilon_k=1/3$, we only perform \textit{push} operation on $v_2$ and $v_3$. Similarly, after the \textit{push} operation, we have $\mathbf{r}_{cur}=[0,-\frac{1}{3},-\frac{2}{3}, -\frac{2}{3}]$ and $\mathbf{r}_{new}=[\frac{5}{3},0,0,0]$. Then, we swap $\mathbf{r}_{cur}$ and $\mathbf{r}_{new}$, and turn to the next iteration $k=3$. Thus, $\mathbf{r}_{cur}=[\frac{5}{3},0,0,0]$ is $\hat{\mathbf{r}}_3$ obtained by the subset Chebyshev recurrence, which is an approximation of $\mathbf{r}_3=T_3(\mathbf{P})\mathbf{e}_{v_1}$. Note that the exact value is $T_3(\mathbf{P})\mathbf{e}_{v_1}=[1,0,0,0]$, suggesting that \chebpush obtains a reasonable approximate solution.

\begin{figure}
    \centering
    \includegraphics[scale=0.22]{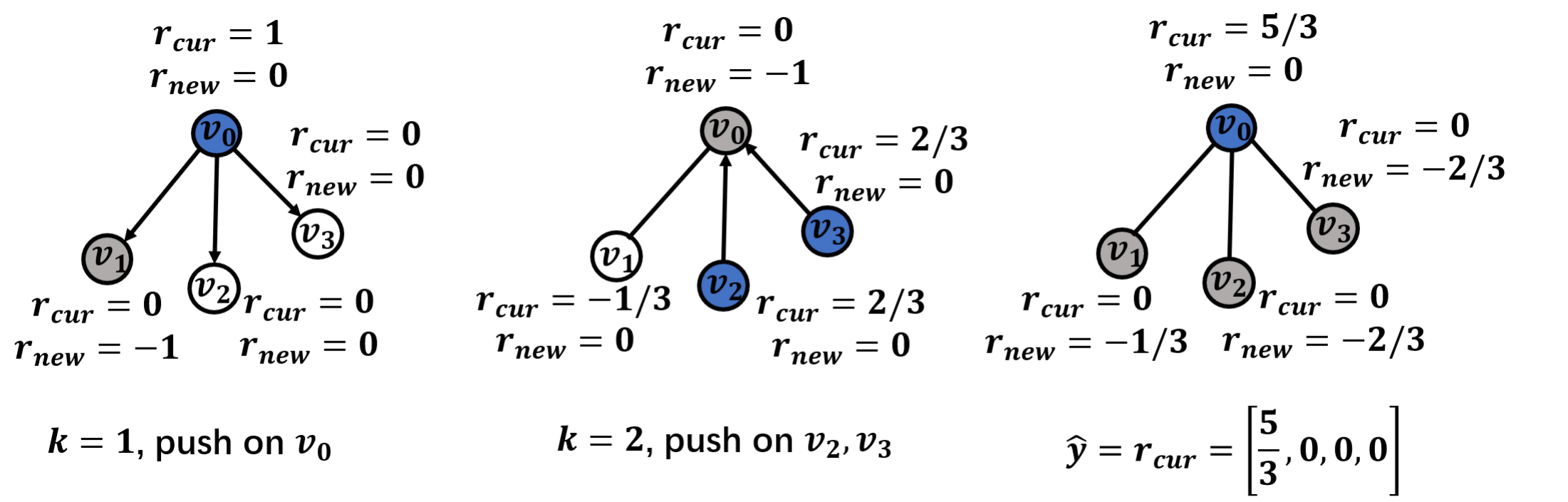} \vspace{-0.4cm}
    \caption{Illustration of \chebpush to compute $T_3(\mathbf{P})\mathbf{e}_{v_1}$} \vspace{-0.3cm}
    \label{fig:exp-chebpush}
\end{figure}

\stitle{Comparison between \push and \chebpush.} Similar to traditional \push \cite{andersen2006local}, \chebpush is also a local algorithm that explores only a small portion of the graph. However, the key mechanism of \chebpush is fundamentally different from that of \push. As illustrated in Figure ~\ref{fig:push-chebpush}, the traditional \push operation on a node $u$ distributes $\mathbf{r}_{cur}(u)$ uniformly to $u$'s neighbors, and sets $\mathbf{r}_{cur}(u)$ to $0$. In contrast, \chebpush distributes $2\mathbf{r}_{cur}(u)$ uniformly to $u$'s neighbors, and updates $\mathbf{r}_{cur}(u)$ to $-\mathbf{r}_{cur}(u)$. Besides, traditional \push can be viewed as a local and asynchronous variant of \powermethod \cite{wu2021unifying}, whereas \chebpush implements \ltwocheb locally on the graph using a novel subset Chebyshev recurrence technique and also shares its rapid convergence properties. 

\begin{figure}
    \centering
    \includegraphics[scale=0.22]{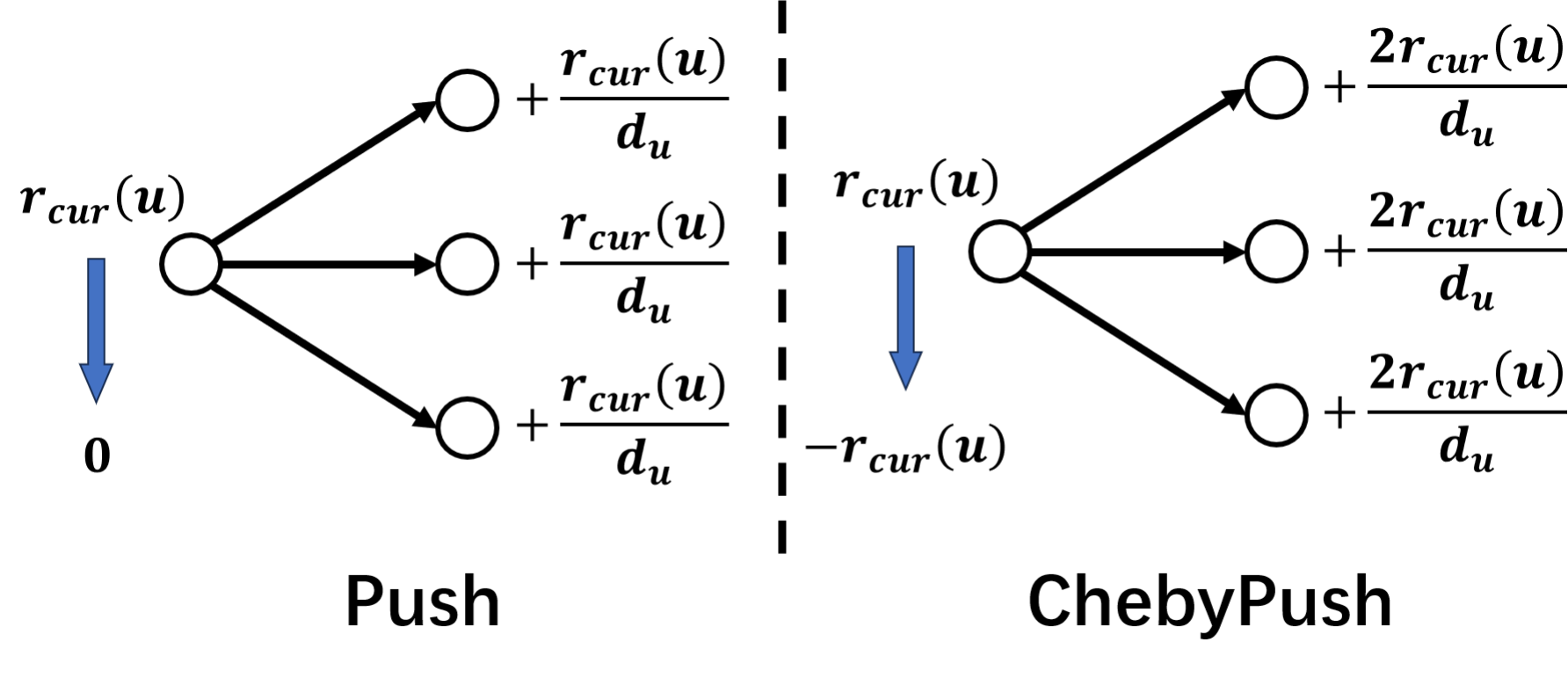} \vspace{-0.4cm}
    \caption{The difference between \truncatepush v.s. \chebpush} \vspace{-0.3cm}
    \label{fig:push-chebpush}
\end{figure}

\subsection{Theoretical Analysis of \chebpush} \label{subsec:analysis-cheby-push}

To analyze the time complexity of \chebpush, it is essential to determine the number of operations required to execute Lines 5-11 of Algorithm~\ref{algo:approx-chebyshev} in each iteration. Our analysis is based on the following stability assumption.

\begin{assumption}{(Stability assumption)}
\label{assump:stability_assump}
    For any $k\in \mathbb{N}^+$, $\Vert T_k(\mathbf{P}^T) \Vert_\infty = \Vert T_k(\mathbf{P}) \Vert_1 \leq C$ for some universal constant $C$.
\end{assumption}

The rationale for assumption \ref{assump:stability_assump} is as follows: recall that by the $k$ step truncation of the Chebyshev expansion, we have $\hat{\mathbf{y}}=\sum_{k=0}^{K}{c_k T_k(\mathbf{P})\mathbf{e}}_s$. So assumption ~\ref{assump:stability_assump} gives us that from $k$ step truncation to $k+1$ step truncation, the value $\hat{\mathbf{y}}$ does not produce much perturbation under $l_1$ error, which is often true in practice. Let $supp\{x\}\triangleq \{u:x(u)\neq 0\}$ be the support of the vector $x$. Denote by ${\hat{\mathbf{r}}_k}$ a vector obtained by the subset Chebyshev recurrence. We define $S_k \triangleq supp\{\hat{\mathbf{r}}_{k}|_{|\hat{\mathbf{r}}_{k}(u)|>\epsilon_kd_u}\}$ (i.e., $S_k$ denotes the node set such that each node $u\in S_k$ has $|\hat{\mathbf{r}}_{k}(u)|>\epsilon_k d_u$) and $vol(S_k)= \sum_{v\in S_k} d_v$. Then, to bound the number of operations executed in Lines 5-11, it is equivalent to bound $vol(S_k)$. Below, we derive an upper bound of $vol(S_k)$.

\begin{lemma}\label{lem:residual_sparse}
    $vol(S_k) \leq O\left(\frac{1} {\epsilon_k}\right)$.
\end{lemma}

\begin{proof}
  By the definition of $\mathbf{r}_k=T_k(\mathbf{P})\mathbf{e}_s$ and Assumption ~\ref{assump:stability_assump}, we have $\Vert\mathbf{r}_k\Vert_1\leq O(1)$. By Assumption ~\ref{assump:stability_assump} and Lemma \ref{lem:residual-propagation}, we have: $\Vert\hat{\mathbf{r}}_k\Vert_1=\Vert\mathbf{r}_k+2\sum_{l=1}^{k-1}{T_{k-l}(\mathbf{P})\mathbf{\delta}_l}\Vert_1\leq O(1)$ by setting the threshold $\epsilon_l$ sufficiently small (i.e., $\epsilon_l<\frac{1}{K}$). Therefore, we have:
       $$ \sum_{u\in S_k}{\epsilon_k d_u}\leq \sum_{u\in S_k}{\hat{\mathbf{r}}_k(u)^2}\leq \sum_{u\in \mathcal{V}}{\hat{\mathbf{r}}_k(u)^2} = \Vert\hat{\mathbf{r}}_k\Vert_2^2 \leq O(1).$$
    
As a result, we have $vol(S_k)\leq O\left(\frac{1}{\epsilon_k}\right)$.
\end{proof}

Based on Lemma ~\ref{lem:residual-propagation} and Lemma ~\ref{lem:residual_sparse}, we can derive the following error bound and time complexity of \chebpush.

\begin{theorem}\label{thm:main-runtime-bound}
    The approximation vector $\hat{\mathbf{y}}$ obtained by \chebpush satisfies the degree-normalized error, that is $\mathop{\max}\limits_{u\in \mathcal{V}}{\frac{|\mathbf{y}(u)-\hat{\mathbf{y}}(u)|}{d_u}} < \epsilon_a $. Besides, the time complexity of \chebpush is $O\left( \min \left\{\frac{K^2}{\epsilon_a}, K m  \right\} \right)$, where $K$ is the truncation step of the Chebyshev expansion.
\end{theorem}

\begin{proof}
    First, we prove $\mathop{\max}\limits_{u\in \mathcal{V}}{\frac{|\mathbf{y}(u)-\hat{\mathbf{y}}(u)|}{d_u}} < \epsilon_a $. To this end, we define another vector $\tilde{\mathbf{y}}$ as the $K$-step truncation of the Chebyshev expansion: $\tilde{\mathbf{y}}=\sum_{k=0}^{K}{c_kT_k(\mathbf{P})\mathbf{e}_s}$, such that $\Vert\tilde{\mathbf{y}}-\mathbf{y}\Vert_2< \epsilon_a/2$. By the definition of $\mathbf{r}_k$, $\tilde{\mathbf{y}}$ can also be reformulated as $\tilde{\mathbf{y}}=\sum_{k=0}^{K}{c_k\mathbf{r}_k}$. 

    Since the approximation $\hat{\mathbf{y}}$ output by \chebpush satisfy: $\hat{\mathbf{y}}=\sum_{k=0}^{K}{c_k(\hat{\mathbf{r}}_k|_{|\hat{\mathbf{r}}_k(u)|>\epsilon_kd_u})}$. Therefore, by Lemma ~\ref{lem:residual-propagation} and the definition of $\delta_k=\hat{\mathbf{r}}_k|_{|\hat{\mathbf{r}}_k(u)|\leq \epsilon_k d_u}$, we can derive the relationship between $\hat{\mathbf{y}}$ and $\tilde{\mathbf{y}}$ as follows:
    \begin{equation}\label{equ:computation}
    \begin{aligned}
        \Vert\mathbf{D}^{-1/2}(\tilde{\mathbf{y}}-\hat{\mathbf{y}})\Vert_\infty
        &=\left\Vert\mathbf{D}^{-1/2}\left(
        \sum_{k=1}^{K}{
        c_k(\mathbf{r}_k-\hat{\mathbf{r}}_k+\delta_k)
        }
        \right)\right\Vert_\infty \\
    &\leq \left\Vert\mathbf{D}^{-1/2}\left(
    \sum_{k=1}^{K}{
        2c_k\left(
        \sum_{l=1}^{k}{T_{k-l}(\mathbf{P})\delta_l}
         \right)
        }
        \right)\right\Vert_\infty \\
        &\leq \sum_{k=1}^{K}{
        2|c_k|\left(
        \sum_{l=1}^{k}{\epsilon_l}
        \right)
        } 
        =\sum_{l=1}^{K}{
        \left(
        \sum_{k=l}^{K}{2|c_k|}
        \right)\epsilon_l.
        } 
        \end{aligned}
    \end{equation}
    The reason for the second inequality is as follows. Since $\delta_l\leq \epsilon_l \mathbf{D}\mathbf{1}_{supp\{\delta_l\}}$, where $\mathbf{1}_{supp\{\delta_l\}}$ is the indicator vector takes value 1 for $u\in supp\{\delta_l\}$ and 0 otherwise, we have:
    $$\Vert \mathbf{D}^{-1}T_{k-l}(\mathbf{P})\delta_l\Vert_\infty \leq \epsilon_l \Vert  \mathbf{D}^{-1}T_{k-l}(\mathbf{P})\mathbf{D}\mathbf{1}_{supp\{\delta_l\}}\Vert_\infty. $$
 Note that $\mathbf{P}=\mathbf{AD}^{-1}=\mathbf{D}\mathbf{P}^T\mathbf{D}^{-1}$ holds for undirected graph. By the assumption that $\Vert T_k(\mathbf{P}^T)\Vert_\infty\leq 1$, we have: $\epsilon_l\Vert  \mathbf{D}^{-1}T_{k-l}(\mathbf{P})\mathbf{D}\mathbf{1}_{supp\{\delta_l\}}\Vert_\infty  \leq \epsilon_l \Vert T_{k-l}(\mathbf{P}^T)\Vert_\infty \leq \epsilon_l$. Thus, the second inequality holds. Since $\epsilon_k=\frac{1}{\sum_{l=k}^{K}{|c_{l}|}}\frac{ \epsilon_a}{4K}$ (Line~4 of Algoirthm~\ref{algo:approx-chebyshev}), the last step of Eq.~(\ref{equ:computation}) can be scaled as follows:
    $$\sum_{l=1}^{K}{
        \left(
        \sum_{k=l}^{K}{2|c_l|}
        \right)\epsilon_l
        } \leq \sum_{l=1}^{K}{\frac{\epsilon_a}{2K}}<\epsilon_a/2.$$
        
        Let $\mathbf{y}$ be the exact GP vector. Then, the relationship between $\mathbf{y}$ and $\hat{\mathbf{y}}$ is as follows:
        $$
        \begin{aligned}
            \Vert\mathbf{D}^{-1}(\mathbf{y}-\hat{\mathbf{y}})\Vert_\infty 
            &\leq 
            \Vert\mathbf{D}^{-1}(\mathbf{y}-\tilde{\mathbf{y}})\Vert_\infty + \Vert\mathbf{D}^{-1}(\hat{\mathbf{y}}-\tilde{\mathbf{y}})\Vert_\infty \\
            &\leq
            \Vert \mathbf{y}-\tilde{\mathbf{y}} \Vert_2 + \Vert\mathbf{D}^{-1}(\hat{\mathbf{y}}-\tilde{\mathbf{y}})\Vert_\infty \\
            &<\epsilon_a/2+\epsilon_a/2=\epsilon_a.
            \end{aligned}
        $$
    Therefore, $|\mathbf{y}(u)-\hat{\mathbf{y}}(u)| < \epsilon d_u $ holds for any $u\in \mathcal{V}$. Therefore $\hat{\mathbf{y}}$ achieves the $l_2$ degree-normalized error. 

        For the time complexity of \chebpush, recall that in each iteration of \chebpush, we do the following subset Chebyshev recurrence:
    $$\hat{\mathbf{r}}_{k+1}=2\mathbf{P}(\hat{\mathbf{r}}_k-\delta_k)-\hat{\mathbf{r}}_{k-1} +  2\delta_{k-1}.$$
At the $k^{th}$ iteration, the operations for the matrix-vector multiplication $\mathbf{P}(\hat{\mathbf{r}}_k|_{|\hat{\mathbf{r}}_k(u)|>\epsilon_k d_u})$ is at most $O(vol(S_k))\leq O\left(\frac{1}{\epsilon_k}\right)$. Since \chebpush executes $K$ iterations, the total operations is no more than $O\left(\sum_{k=1}^{K}{\frac{1}{\epsilon_k}}\right)\leq O\left(\frac{K}{\epsilon} \left(\sum_{k=1}^{K}{(\sum_{l=k}^{K}{|c_l|})}\right)\right)$. By $\sum_{l=k}^{K}{|c_l|}< 1$ ( $\sum_{l=0}^{+\infty}{|c_l|}=1$ for GP function $f$), the total operations is $O\left(\frac{K^2}{\epsilon_a} \right)$. On the other hand, when we set $\epsilon_a = 0$, \chebpush is equivalent to \ltwocheb. Therefore, the time complexity of \chebpush can also be bounded by the time complexity of \ltwocheb $O\left(K m \right)$. Putting it all together, the time complexity of \chebpush can be bounded by $O\left( \min \left\{\frac{K^2}{\epsilon_a}, K m \right\} \right)$. 
\end{proof}

Theorem~\ref{thm:main-runtime-bound} shows that our \chebpush algorithm can achieve a degree normalized error guarantee while using computation time independent of the graph size to compute general GP vectors. Note that the state-of-the-art algorithm for computing general GP vectors is the randomized push algorithm \agp \cite{wang2021approximate}, which has a time complexity of $O(\frac{N^3}{\epsilon_a})$, where $N$ denotes the truncation step of the Taylor expansion. In comparison, \chebpush has a time complexity of $O\left(\frac{K^2 }{\epsilon_a}\right)$, with $K$ denoting the truncation step of the Chebyshev expansion. Given that $K$ is roughly $O(\sqrt{N})$ (see Lemma ~\ref{lem:converge_rate_cheb}), the computational complexity of  \chebpush is lower than \agp. It is worth mentioning that for general GP vectors, such as \hkpr, $N$ often needs to be set to a large value to achieve a good accuracy, making \chebpush significantly faster than \agp, as evidenced by our experimental results.     

\comment{
\stitle{Discussions.} 
Now let's go back to \ssppr and \hkpr. By lemma ~\ref{lem:converge_rate_cheb}, the runtime bound of \chebpush for \ssppr is $\tilde{O}\left(\frac{\sqrt{|\bar{S}|}}{\alpha \epsilon}\right)$, and for \hkpr is $\tilde{O}\left(\frac{t\sqrt{|\bar{S}|}}{ \epsilon}\right)$, where $\tilde{O}(.)$ is the notation $O(.)$ omitting $\log$ terms. The runtime of \push for \ssppr is $O(\frac{1}{\alpha\epsilon})$ (by the special structure of \ssppr), so \chebpush is theoretically slightly worse than \push. For \hkpr, the runtime of the SOTA algorithm: \agp \cite{wang2021approximate} is $O(\frac{N^3}{\epsilon_a})=\tilde{O}(\frac{t^3}{\epsilon_a})$ to reach the same error bound as \chebpush, where $N$ is the truncation step of Taylor expansion. So our algorithm is faster when $|S|<t^4$. Generally speaking, as we discussed before, the main idea of \push is to locally implement \powermethod, while the main idea of \chebpush is to locally implement Chebyshev recurrence. So it is not surprising that \chebpush naturally converges faster than \push (for general GP vector computation). However, since there is an additional $\sqrt{|\bar{S}|}$ term that exists in runtime bound, this partially offsets the effect. We believe this happens mainly by the property of Chebyshev polynomials since we only have $\Vert T_k(\mathbf{P})\mathbf{e}_s \Vert_2\leq 1$ and $\sum_{u\in \mathcal{V}}{T_k(\mathbf{P}) \mathbf{e}_s(u)}=1$. But for Taylor expansion, we have $\Vert \mathbf{P}^k \mathbf{e}_s\Vert_1 =1$ and $\mathbf{P}^k \mathbf{e}_s$ is a probability distribution. This makes purning $T_k(\mathbf{P})\mathbf{e}_s$ slightly harder than $\mathbf{P}^k \mathbf{e}_s$.  In our experiments, \chebpush performs similarly to \push for larger $\epsilon$, but significantly better than \push when $\epsilon$ is small.
}

\section{Generalizations}\label{sec:problem_variants}
In this section, we first show that our \chebpush algorithm can be further extended to compute a more general graph propagation matrix. Then, similar to traditional \push, we show that \chebpush can also be combined with simple Monte Carlo sampling techniques to obtain a bidirectional algorithm. Moreover, compared to the state-of-the-art \push-based bidirectional algorithm for \ssppr vector computation \cite{wu2021unifying}, our solution achieves a complexity reduction by a factor of $O(1/\sqrt{\alpha})$. 

\subsection{Computing More General GP Matrix} \label{subsec:more-general-gp}
Instead of using the random-walk matrix $\mathbf{P}=\mathbf{AD}^{-1}$ to define the graph propagation function, in the graph machine learning community, a more general graph propagation function is frequently utilized~\cite{wang2021approximate}, defined as follows.

\begin{definition} \label{def:general-gp}
    Given a feature matrix $\mathbf{X}\in \mathbb{R}^{n\times k}$ and a graph propagation function $f$, the generalized graph propagation is defined as: $\mathbf{Z}=f(\mathbf{D}^{-a}\mathbf{A}\mathbf{D}^{-b})\mathbf{X}$ with $a+b=1$.
\end{definition}

Clearly, when $a=0$ and $b=1$, the graph propagation (GP) function $f(\mathbf{D}^{-a}\mathbf{A}\mathbf{D}^{-b})$ defined in Definition~\ref{def:general-gp} is equivalent to $f(\mathbf{P})$. Therefore, Definition~\ref{def:general-gp} represents a more generalized form compared to our previous GP vector. This generalized GP model is commonly used in graph machine learning~\cite{bojchevski2019pagerank,gasteiger2019diffusion,wu2019simplifying}.

To compute the generalized GP matrix $\mathbf{Z}$, we first focus on a simple \textit{vector-based} version: $\mathbf{z}=f(\mathbf{D}^{-a}\mathbf{A}\mathbf{D}^{-b})\mathbf{x}$ with $\mathbf{x}\in \mathbb{R}^n$. The following theorem suggests that \textit{vector-based} GP vector $\mathbf{z}$ can be computed by our \chebpush algorithm.  
\begin{theorem} \label{thm:general-gp-vector-compute}
  Given a GP function $f$, the generalized GP vector $\mathbf{z}=f(\mathbf{D}^{-a}\mathbf{A}\mathbf{D}^{-b})\mathbf{x}$ can be expressed as: $\mathbf{z}=\mathbf{D}^{-a}\mathbf{y}$ with $\mathbf{y}=f(\mathbf{P})\mathbf{D}^a\mathbf{x}$.
\end{theorem}
\begin{proof}
    By the definition of the random-walk matrix $\mathbf{P}=\mathbf{AD}^{-1}$ and $a+b=1$, we derive that
$(\mathbf{D}^{-a}\mathbf{A}\mathbf{D}^{-b})^k\mathbf{x}=\mathbf{D}^{-a}(\mathbf{AD}^{-1})^k\mathbf{D}^{1-b}\mathbf{x}=\mathbf{D}^{-a}(\mathbf{P}^k \mathbf{D}^{a}\mathbf{x})$. 
By the Taylor expansion
$$\mathbf{z}=f(\mathbf{D}^{-a}\mathbf{A}\mathbf{D}^{-b})\mathbf{x}=\sum_{k=0}^{\infty}{\zeta_k(\mathbf{D}^{-a}\mathbf{A}\mathbf{D}^{-b})^k\mathbf{x}}$$
we have:
$$\mathbf{z}=\sum_{k=0}^{\infty}{\zeta_k(\mathbf{D}^{-a}\mathbf{A}\mathbf{D}^{-b})^k\mathbf{x}}=\sum_{k=0}^{\infty}{\zeta_k \mathbf{D}^{-a}\mathbf{P}^k\mathbf{D}^{a}\mathbf{x} }=\mathbf{D}^{-a}f(\mathbf{P})\mathbf{D}^{a}\mathbf{x}$$
\end{proof}

Based on Theorem~\ref{thm:general-gp-vector-compute}, we can first use \chebpush($\mathcal{G},f,\mathbf{D}^a\mathbf{x},K,\epsilon_a$) to approximately compute $\mathbf{y}=f(\mathbf{P})\mathbf{D}^a\mathbf{x}$ and then obtain the generalized GP vector by $\mathbf{z}=\mathbf{D}^{-a}\mathbf{y}$. Furthermore, to compute the generalized GP matrix $\mathbf{Z}$, we can express $\mathbf{Z}=[\mathbf{z}_1,...,\mathbf{z}_k]$ and $\mathbf{X}=[\mathbf{x}_1,...,\mathbf{x}_k]$. For each $i\in [k]$, we compute the approximation $\hat{\mathbf{z}}_k=\mathbf{D}^{-a}\hat{\mathbf{y}}_k$ with $\hat{\mathbf{y}}_k$ computed by \chebpush($\mathcal{G},f,\mathbf{D}^a\mathbf{x}_k,K,\epsilon_a$). We repeat this \textit{vector} GP vector computation procedure $O(k)$ times to obtain an approximation $\hat{\mathbf{Z}}=[\hat{\mathbf{z}}_1,...,\hat{\mathbf{z}}_k]$.

\subsection{Generalizing to a Bidirectional Method} \label{subsec:bidirect-method}

As we discussed in Section~\ref{sec:preliminaries}, \push is a very basic operator in many tasks. Especially, for \ssppr, many existing algorithms ~\cite{wang2017fora,wu2021unifying,liao2022efficient} follow a  bidirectional framework: (i) use \push to obtain a rough approximation; (ii) refine this approximation through random walks (\rw) sampling. Our proposed algorithm can also fit within this framework by simply replacing \push with \chebpush. For completeness, we will first briefly introduce random walk sampling and then present the bidirectional algorithm that combines \chebpush with \rw.

The basic idea of \rw sampling is to generate $\alpha$-random walks to estimate the \ssppr vector $\bm{\pi}_s$ \cite{wang2017fora}. Specifically, we set the estimate vector $\hat{\bm{\pi}}_s=0$, and generate $W$ $\alpha$-random walks. Here the $\alpha$-random walk performs as follows: (i) we start the random walk from node $s$; (ii) suppose that the current node of the random walk is $u$; In each step, with probability $(1-\alpha)$ we uniformly choose a neighbor of $u$, with probability $\alpha$ we stop the random walk; (iii) when the random walk stops, we output the terminate node $t$ and update the approximation $\hat{\bm{\pi}}_s(t)=\hat{\bm{\pi}}_s(t)+\frac{1}{W}$. By the Chernoff bound ~\cite{chung2006concentration}, we can obtain that by setting $W=O(\frac{\log n}{\epsilon_r^2 \delta})$, the approximate \ssppr vector satisfies an $(\epsilon_r,\delta,1/n)$-approximation. That is, for each $\bm{\pi}_s(v)>\delta$, the approximation follows $|\bm{\pi}_s(v)-\hat{\bm{\pi}}_s(v)|<\epsilon_r \bm{\pi}_s(v)$ with probability at least $1-\frac{1}{n}$. Below, we show how to combine \chebpush with \rw. Denote by $\mathbf{r}=\alpha^{-1}[\mathbf{I}-(1-\alpha)\mathbf{P}](\bm{\pi}_s-\hat{\bm{\pi}}_s)$ the residual vector. Similar to \push, we first discuss the residual bound for \chebpush. The following lemma shows that after performing \chebpush, the residual is bounded by $|\mathbf{r}(u)|\leq \epsilon_a d_u$ for all $u\in \mathcal{V}$, which is consistent with \push.

\begin{lemma}
    Let $\bm{\pi}_s$ be the accurate \ssppr vector with source $s$, and $\hat{\bm{\pi}}_s$ be the approximation vector. Then, after performing $\hat{\bm{\pi}}_s\leftarrow$ \chebpush($\mathcal{G},f,s,K,\epsilon_a^2$), the residual vector $\mathbf{r}$ satisfies $|\mathbf{r}(u)|\leq \epsilon_a d_u$ for all $u\in \mathcal{V}$.
\end{lemma}\label{lem:residual-bound}
\begin{proof}
    According to Theorem ~\ref{thm:main-runtime-bound}, the approximation vector $\hat{\bm{\pi}}_s$ returned by \chebpush satisfies the $l_2$ degree normalized relative error $\Vert \mathbf{D}^{-1}(\bm{\pi}_s-\hat{\bm{\pi}}_s) \Vert_\infty <\epsilon_a$. Let $\mathbf{d}$ be the degree vector, then $|\bm{\pi}_s-\hat{\bm{\pi}}_s|< \epsilon_a \mathbf{d}$. By the definition of residual, we have:
    \begin{equation*}
        \begin{aligned}
    |\mathbf{r}|&=\alpha^{-1}[\mathbf{I}-(1-\alpha)\mathbf{P}]|\bm{\pi}_s-\hat{\bm{\pi}}_s| \\
            &\leq \alpha^{-1}[\mathbf{I}-(1-\alpha)\mathbf{P}] \epsilon_a \mathbf{d}\\
            &=\alpha^{-1} [\epsilon_a \mathbf{d}-\epsilon_a(1-\alpha) \mathbf{P}\mathbf{d}]\\
            &=\alpha^{-1} [\epsilon_a \mathbf{d}-\epsilon_a(1-\alpha) \mathbf{d}]=\epsilon_a \mathbf{d}.
        \end{aligned}
    \end{equation*}
\end{proof}
On top of that, the following lemma shows that the residual maintains an invariant, a crucial aspect in the design of bidirectional methods~\cite{andersen2006local,yang2024efficient}.

\begin{lemma} \label{lem:invariant}
    Let $\mathbf{r}$ be the residual vector along with the approximation $\hat{\bm{\pi}}_s$. Then, the following equation holds:
    \begin{equation}
      \bm{\pi}_s(u) = \hat{\bm{\pi}}_s(u)+\sum_{v\in \mathcal{V}}{\mathbf{r}(v)\bm{\pi}_v(u)}, \forall u\in \mathcal{V}. 
    \end{equation}
\end{lemma}\label{lem:invarient}

Based on Lemma~\ref{lem:invariant}, we present our bidirectional method for computing \ssppr in Algorithm~\ref{algo:chebpush+rw}. The detailed procedure is as follows: (i) in the first phase, performing \chebpush with threshold $r_{max}$ to get a rough approximation $\hat{\bm{\pi}}_s$ along with its residual $\mathbf{r}$; (ii) in the second phase, for each $v\in \mathcal{V}$, we perform $W_v=\lceil |\mathbf{r}(v)| W\rceil$ $\alpha$-random walks, and updating $\hat{\bm{\pi}}_s(u) \leftarrow \hat{\bm{\pi}}_s(u)+\frac{\mathbf{r}(v)}{W_v}$ if the random walk stops at $u$. 

\begin{algorithm}[t!]
\small
	\SetAlgoLined
	\KwIn{$\mathcal{G}$, $\alpha$, $s$, $K$, $\epsilon_r$, $\delta$}
        $W\leftarrow \frac{2(2\epsilon_r/3+2)\log n}{\epsilon_r^2 \delta}$\;
        $r_{max}\leftarrow  \frac{\sqrt{\alpha}}{W}$\;
        $\hat{\bm{\pi}}_s\leftarrow $ \chebpush($\mathcal{G},f=\alpha(1-(1-\alpha)x)^{-1},s,K,r_{max}$) \;
        $\mathbf{r}\leftarrow \mathbf{e}_s-\alpha^{-1}[\mathbf{I}-(1-\alpha)\mathbf{P}]\hat{\bm{\pi}}_s$ \;
        \For{$v\in \mathcal{V}$ with $|\mathbf{r}(v)|>0$}{
        $W_v\leftarrow \lceil |\mathbf{r}(v)| W\rceil$ \;
        Perform $W_v$ $\alpha$-random walks from the node $v$ \;
        \For{each random walk terminates at node $u$}{
            $\hat{\bm{\pi}}_s(u) \leftarrow \hat{\bm{\pi}}_s(u)+\frac{\mathbf{r}(v)}{W_v}$ \;
        }
        }
	\KwOut{$\hat{\bm{\pi}}_s$ as the approximation of $\bm{\pi}_s$}
	\caption{\chebpush+\rw}\label{algo:chebpush+rw}
\end{algorithm}

\stitle{Analysis of Algoirthm~\ref{algo:chebpush+rw}.} Since the runtime and error analysis of Algorithm ~\ref{algo:chebpush+rw} is similar to those of the traditional bidirectional methods~\cite{wang2017fora,wu2021unifying}, we simplify the discussions here. For the runtime bound, we consider the time complexity for the \chebpush phase and the \rw phase respectively. For \chebpush, since the threshold $r_{max}$ is often set to a very small number for bidirectional methods (Line~2 of Algorithm~\ref{algo:chebpush+rw}), and the truncation step for Chebyshev expansion is $K=\frac{1}{\sqrt{\alpha}}\log{(\frac{1}{r_{max}})}$, we use Theorem ~\ref{thm:main-runtime-bound} with upper bound $O(Km)=O\left(\frac{m}{\sqrt{\alpha}}\log{(\frac{1}{r_{max}})}\right)$ to bound the time cost for \chebpush. The time cost for \rw is $O(\frac{1}{\alpha}\sum_{v\in \mathcal{V}}{W_v})$, because the expected time cost for each $\alpha$-random walk is $\frac{1}{\alpha}$ and we perform $W_v$ times for each $v\in \mathcal{V}$. Additionally, since $W_v=\lceil |\mathbf{r}(v)| W\rceil$ and $|\mathbf{r}(u)|\leq r_{max} d_u$ by Lemma ~\ref{lem:residual-bound}, we can obtain $O(\frac{1}{\alpha}\sum_{v\in \mathcal{V}}{W_v})\leq O(\frac{1}{\alpha}mr_{max}W)$. Putting it together, the total runtime is $O\left(\frac{1}{\sqrt{\alpha}}\left(m\log{(\frac{1}{r_{max}})}+\frac{1}{\sqrt{\alpha}}mr_{max}W\right)\right)$. By setting $r_{max}=  \frac{\sqrt{\alpha}}{W}$, $W= \frac{2(2\epsilon_r/3+2)\log n}{\epsilon_r^2 \delta}$, this runtime bound is simplified by $\tilde{O}\left(\frac{1}{\sqrt{\alpha}}m\right)$, which improves \speedppr ~\cite{wu2021unifying} by a factor $\frac{1}{\sqrt{\alpha}}$. Since $W= \frac{2(2\epsilon_r/3+2)\log n}{\epsilon_r^2 \delta}$, the error bound of $\hat{\bm{\pi}}_s$ output by algorithm ~\ref{algo:chebpush+rw} is the same as those of traditional bidirectional algorithms ~\cite{wang2017fora,wu2021unifying}: $|\bm{\pi}_s(u)-\hat{\bm{\pi}}_s(u) |<\epsilon_r \bm{\pi}_s(u)$ for any $\bm{\pi}_s(u)>\delta$ with high probability. The case is similar when using loop-erased random walk (\lv) to substitute \rw ~\cite{liao2022efficient}, thereby we omit the details for brevity. 

\comment{
\subsection{Single Node Queries}

Computing single-node graph centrality is a basic task for modern network analysis. Shortly speaking, the single-node graph centrality is defined as follows ~\cite{bressan2018sublinear}:

\begin{definition}
    Given a node $u\in \mathcal{V}$ and the graph propagation function $f$, the single-node graph centrality is defined as $P(u)=\frac{1}{n}f(\mathbf{P})\vec{\mathbf{1}}(u)=\frac{1}{n}\sum_{k=0}^{\infty}{\zeta_k\mathbf{P}^k\vec{\mathbf{1}}}(u)$, where $\vec{\mathbf{1}}$ is the all-one vector.
\end{definition}

The main observation is that for undirected graphs, we can express the single-node graph centrality in a single-source version.

\begin{theorem}
    The single-node graph centrality can be expressed as: $P(u)=\frac{d_u}{n}\sum_{v\in \mathcal{V}}{\frac{1}{d_v}\mathbf{y}(v)}$, where $\mathbf{y}=f(\mathbf{P})\mathbf{e}_u$ is the (single-source) graph propagation vector as defined by Equation \ref{equ:graph diffusion}.
\end{theorem}

\begin{proof}
    Observe that the above definition can also be expressed as follows: $P(u)=\frac{1}{n}f(\mathbf{P})\vec{\mathbf{1}}(u)=\frac{1}{n}\sum_{v\in \mathcal{V}}{f(\mathbf{P})\mathbf{e}_v(u)}$. Now we denote $f(\mathbf{P})\mathbf{e}_v(u)$ by $(f(\mathbf{P}))_{u,v}$, i.e. the $(u,v)$-th entry of matrix $f(\mathbf{P})$. The key observation here is that $(f(\mathbf{P}))_{u,v}=\frac{d_u}{d_v}(f(\mathbf{P}))_{v,u}$ holds for undirected graphs, which is proved as follows:

    According to the Taylor expansion of the function $f$, we have $(f(\mathbf{P}))_{u,v}=\sum_{k=0}^{\infty}{\zeta_k(\mathbf{P}^k)_{u,v}}$. By the definition of walk matrix $\mathbf{P}=\mathbf{AD}^{-1}$ and $\mathcal{G}$ is an undirected graph, we further have
$(\mathbf{P}^k)_{u,v}=(\mathbf{AD}^{-1})^k_{u,v}=((\mathbf{AD}^{-1})^T)^k_{v,u}=(\mathbf{D}^{-1}\mathbf{P}^k\mathbf{D})_{v,u}=\frac{d_u}{d_v}(\mathbf{P}^k)_{v,u}$.
So we have
$(f(\mathbf{P}))_{u,v}=\sum_{k=0}^{\infty}{\zeta_k(\mathbf{P}^k)_{u,v}}=\sum_{k=0}^{\infty}{\zeta_k\frac{d_u}{d_v}(\mathbf{P}^k)_{v,u}}=\frac{d_u}{d_v}f(\mathbf{P})_{v,u}$. As a result, the following equation holds 
$P(u)=\frac{1}{n}\sum_{v\in \mathcal{V}}{(f(\mathbf{P}))_{u,v}}=\frac{1}{n}\sum_{v\in \mathcal{V}}{\frac{d_u}{d_v}(f(\mathbf{P}))_{v,u}}=\frac{d_u}{n}\sum_{v\in \mathcal{V}}{\frac{1}{d_v}\mathbf{y}(v)}$
with $\mathbf{y}=f(\mathbf{P})\mathbf{e}_u$. 
\end{proof}

Based on the above theoretical analysis, we design Algorithm ~\ref{algo:chebpush_single_node} to compute the single-node graph centrality. Specifically, it first inputs a parameter $\epsilon$, a node $u\in \mathcal{V}$, and a propagation function $f$. Then, it computes the approximate graph propagation $\hat{\mathbf{y}}\leftarrow$ \chebpush($\mathcal{G},f,u,K,\epsilon/d_u$). Finally, it outputs $\hat{P}(u)=\frac{d_u}{n}\sum_{v\in \mathcal{V}}{\frac{1}{d_v}\hat{\mathbf{y}}(v)}$ as the approximate single-node value. Time complexity and error bound of Algorithm ~\ref{algo:chebpush_single_node} are stated as follows. 

\begin{algorithm}
\small
	\SetAlgoLined
	\KwIn{$\mathcal{G}$, $f$, $u$, $K$, $\epsilon$}
        $\hat{\mathbf{y}}\leftarrow $ \chebpush($\mathcal{G},f,u,K,\epsilon/d_u$) \;
       $\hat{P}(u)\leftarrow \frac{d_u}{n}\sum_{v\in \mathcal{V}}{\frac{1}{d_v}\hat{\mathbf{y}}(v)}$ \;
	\KwOut{$\hat{P}(u)$ as the approximation of $P(u)$}
	\caption{\chebpush for single-node graph centrality}\label{algo:chebpush_single_node}
\end{algorithm}

\begin{theorem}
   Algorithm ~\ref{algo:chebpush_single_node} outputs $\hat{P}(u)$ such that $|P(u)-\hat{P}(u)|<\epsilon$ with runtime bound $O\left(\frac{Cd_u\sqrt{|\bar{S}|}}{\epsilon}\right)$, in which $C$ and $\bar{S}$ are the same as Theorem ~\ref{thm:main-runtime-bound}.
\end{theorem}

\begin{proof}
    By Theorem ~\ref{thm:main-runtime-bound}, $\hat{\mathbf{y}}$ returned by \chebpush satisfies the degree-normalized relative error: $\Vert \mathbf{D}^{-1}(\mathbf{y}-\hat{\mathbf{y}})\Vert_\infty<\epsilon/d_u$. Therefore, 
    $|P(u)-\hat{P}(u)|\leq\frac{d_u}{n}\sum_{v\in \mathcal{V}}{\frac{1}{d_v}|\hat{\mathbf{y}}(v)-\mathbf{y}(v)|}< \frac{d_u}{n}\sum_{v\in \mathcal{V}}{\frac{1}{d_v}\left(\epsilon \frac{d_v}{d_u}\right)}=\epsilon$. As a result,
   the time complexity is $O\left(\frac{Cd_u\sqrt{|\bar{S}|}}{\epsilon}\right)$. 
\end{proof}

}

\begin{table}[t!]\vspace{-0.3cm}
\small
	\centering
	\caption{Statistics of datasets} \vspace{-0.2cm}
        \label{tab:dataset}
	{
		\label{Datasets for vertex classification}
		\centering
		
		\begin{tabular}{cccccc}
			\toprule
			Datasets&$n$&$m$& $m/n$  & type \\
			\midrule
			\dblp & 317,080 & 1,049,866  & 3.31 &Undirected\\
			\youtube&1,134,890&2,987,624 &2.63 &Undirected \\
			\livejournal&4,846,609&42,851,237&8.84&Undirected\\
			\orkut&3,072,626&11,718,5083&38.13&Undirected\\
            \friendster&65,608,366&1,806,067,135&27.53&Undirected\\
			\bottomrule
		\end{tabular}
	}\vspace{-0.4cm}
\end{table}

\begin{figure*}[!t]
\vspace{-0.3cm}
	\centering
 \subfigure[\dblp]{
            \centering
		\includegraphics[scale=0.14]{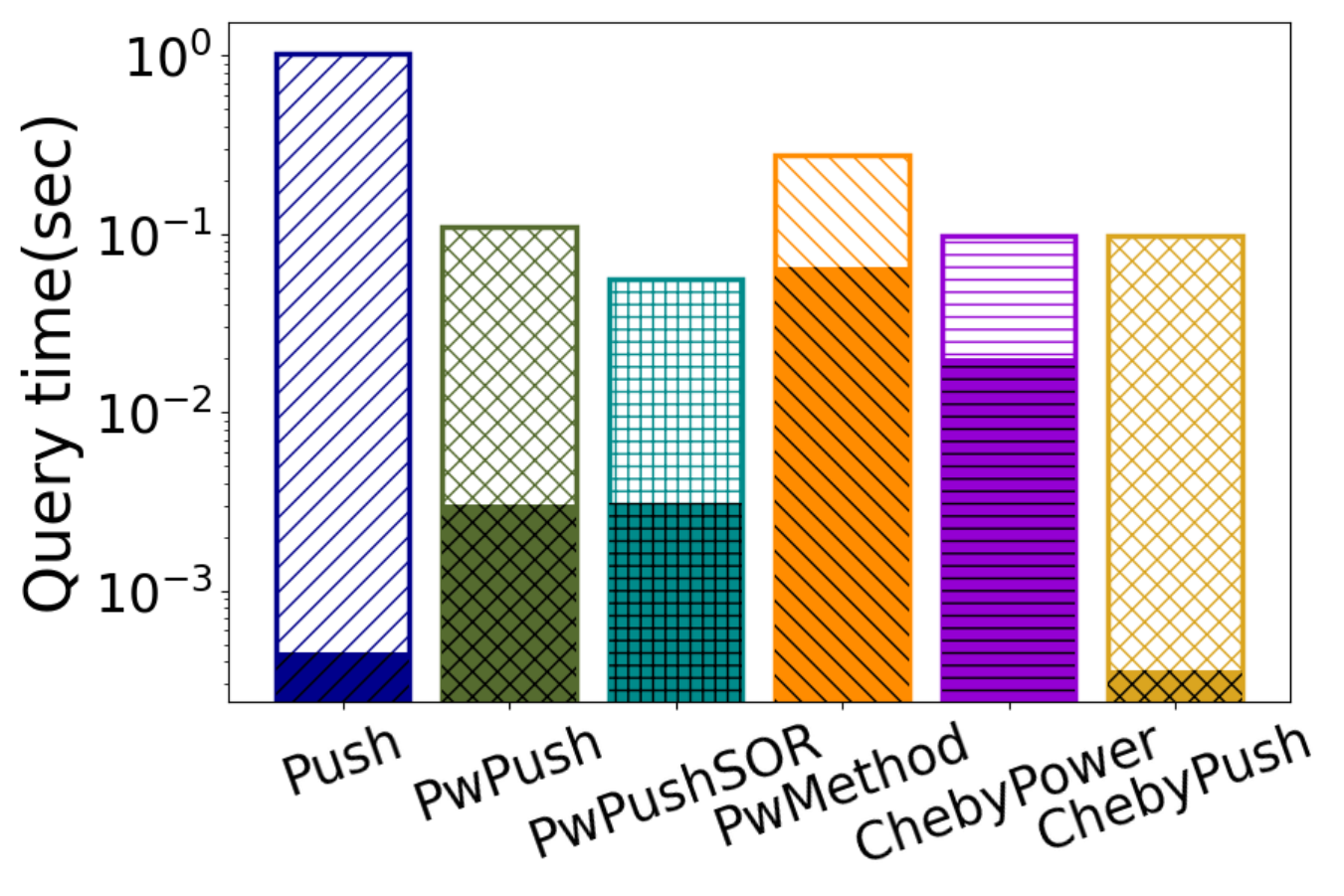}\hspace{-5mm} \label{1}
	}
	\quad
        \subfigure[\youtube]{
        \centering
		\includegraphics[scale=0.14]{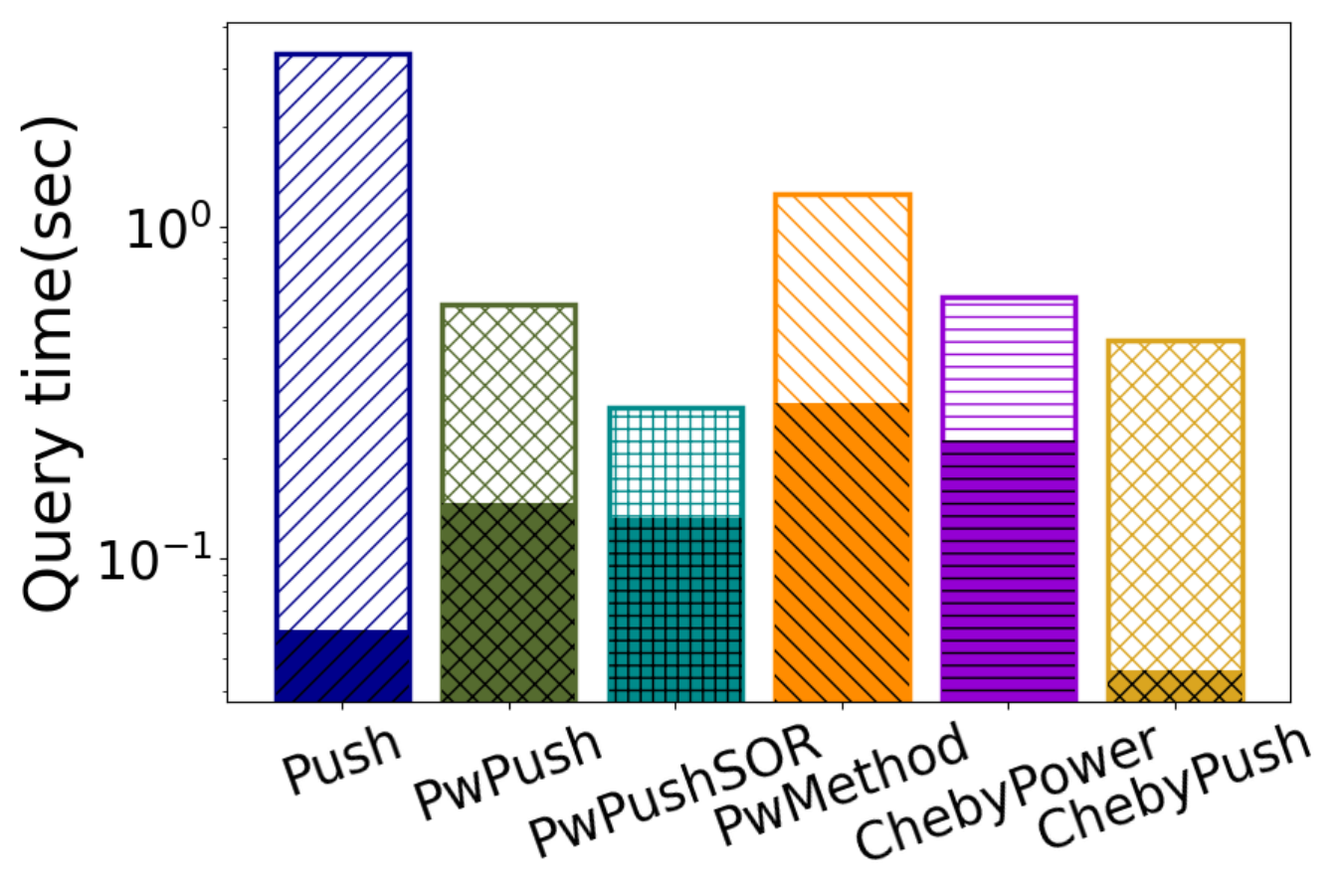}\hspace{-5mm}  \label{2}
	}
	\quad
	\subfigure[\livejournal]{
        \centering
		\includegraphics[scale=0.14]{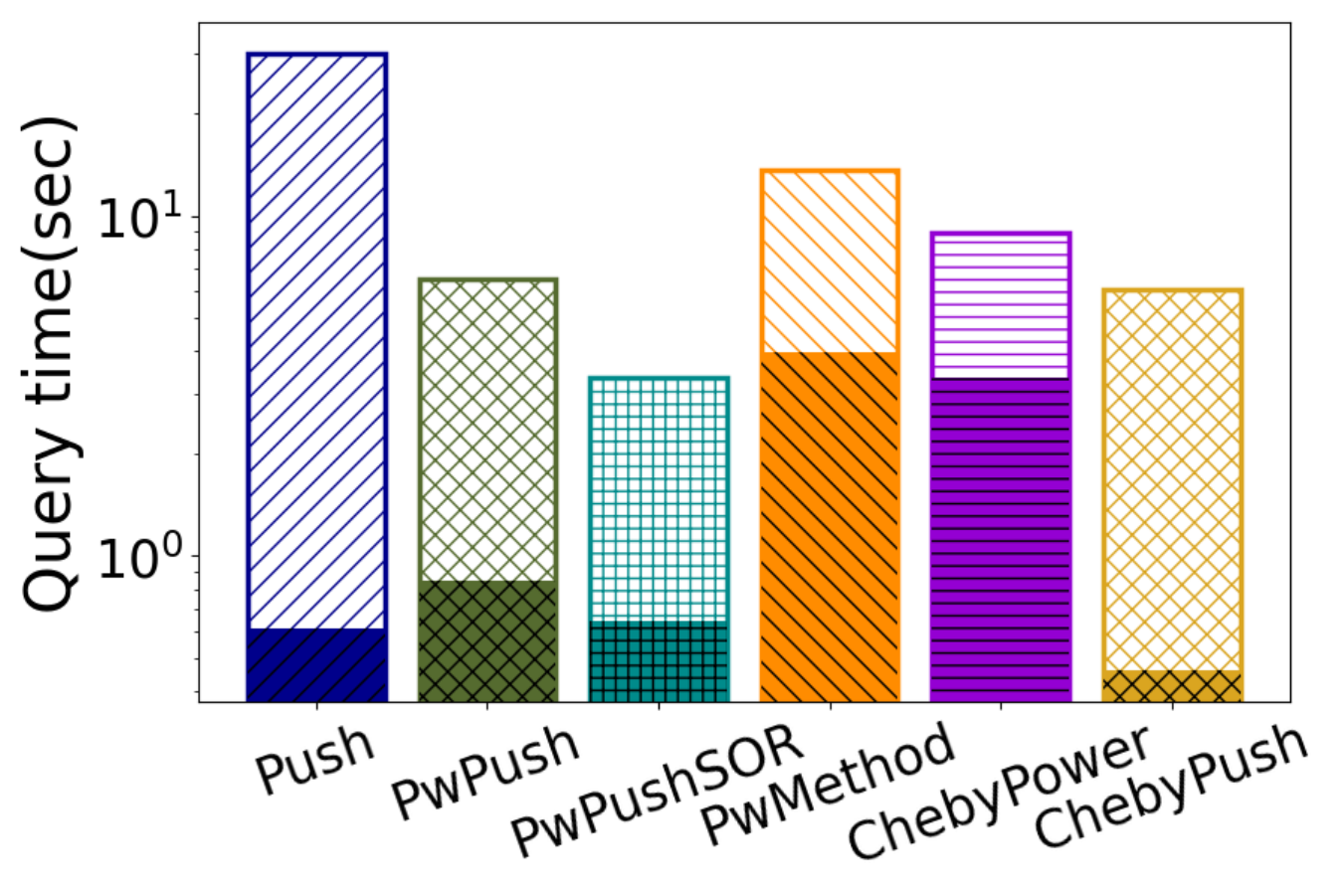}\hspace{-5mm}  \label{3}
	}
	\quad
	\subfigure[\orkut]{
        \centering
		\includegraphics[scale=0.14]{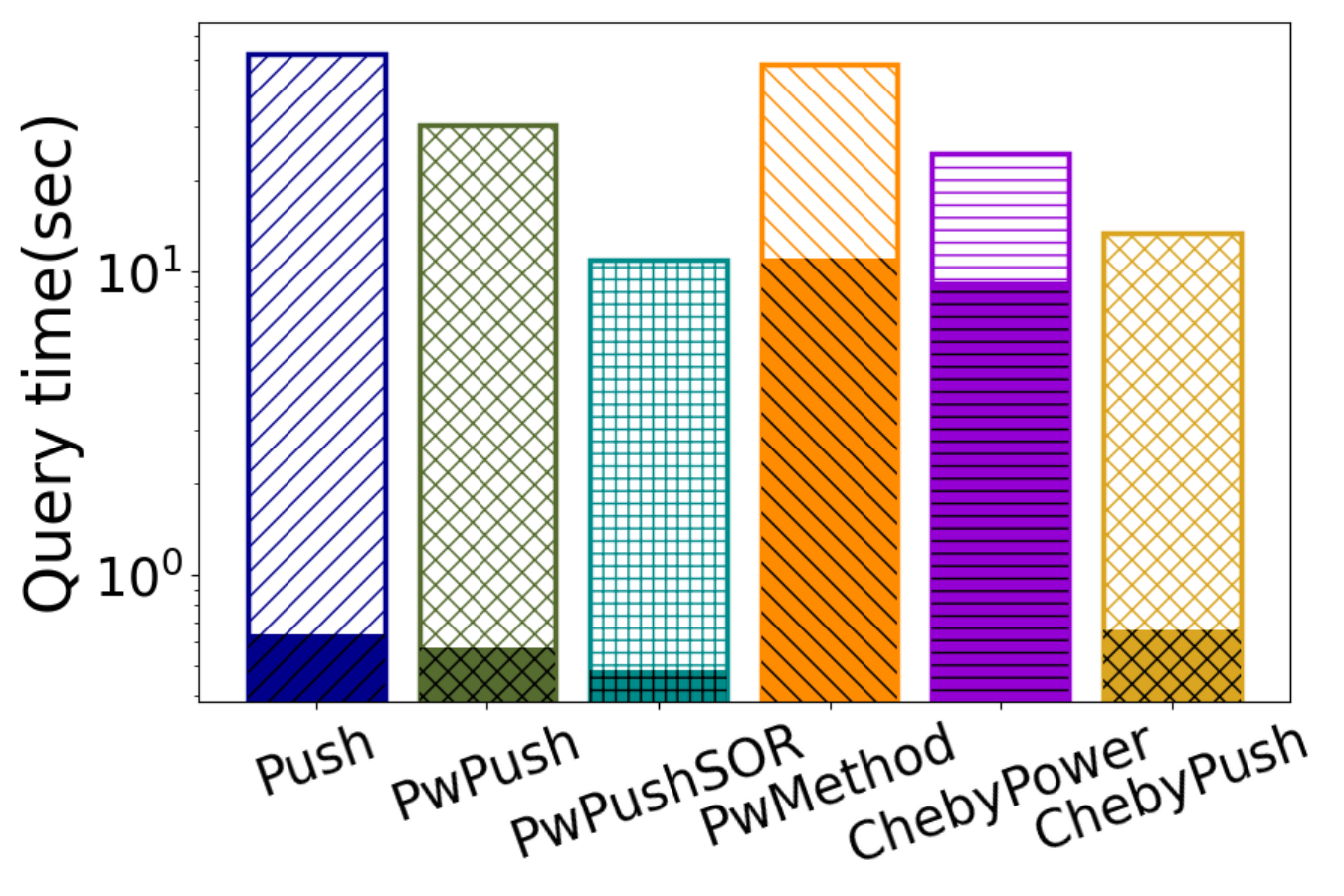}\label{4}
	}
 \subfigure[\friendster]{
        \centering
		\includegraphics[scale=0.14]{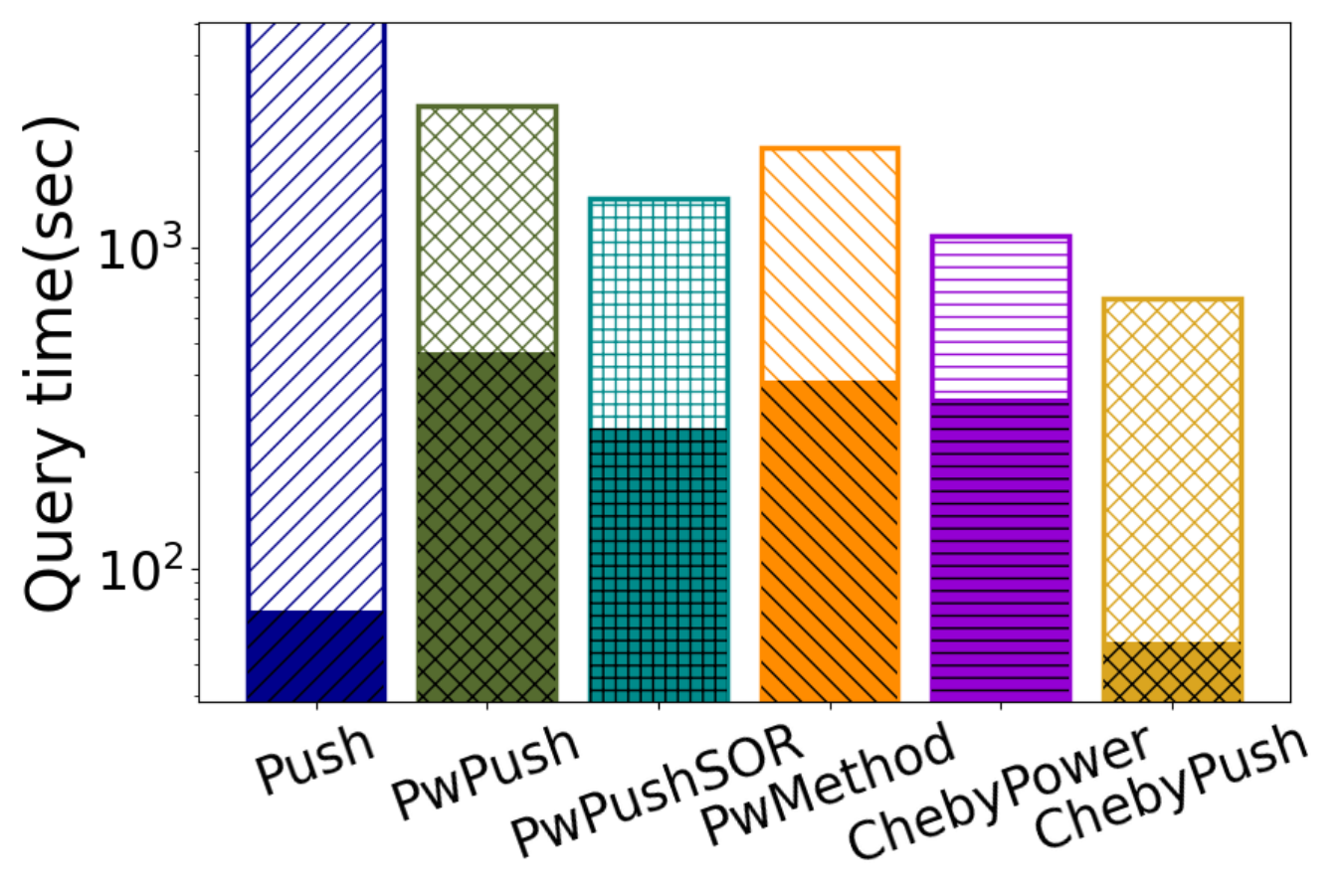}\label{4}
	}
 \vspace{-0.3cm}
	\caption{Query time of different \ssppr algorithms under $l_1$-error. The lower (resp., upper) bar of each figure represents the query time of each algorithm to reach low-precision (resp., high-precision) $l_1$-error.}\label{fig:ssppr-l1err}
\vspace{-0.2cm}
\end{figure*}

\begin{figure*}[!t]
\vspace{-0.3cm}
	\centering
	\subfigure[\dblp]{
            \centering
		\includegraphics[scale=0.14]{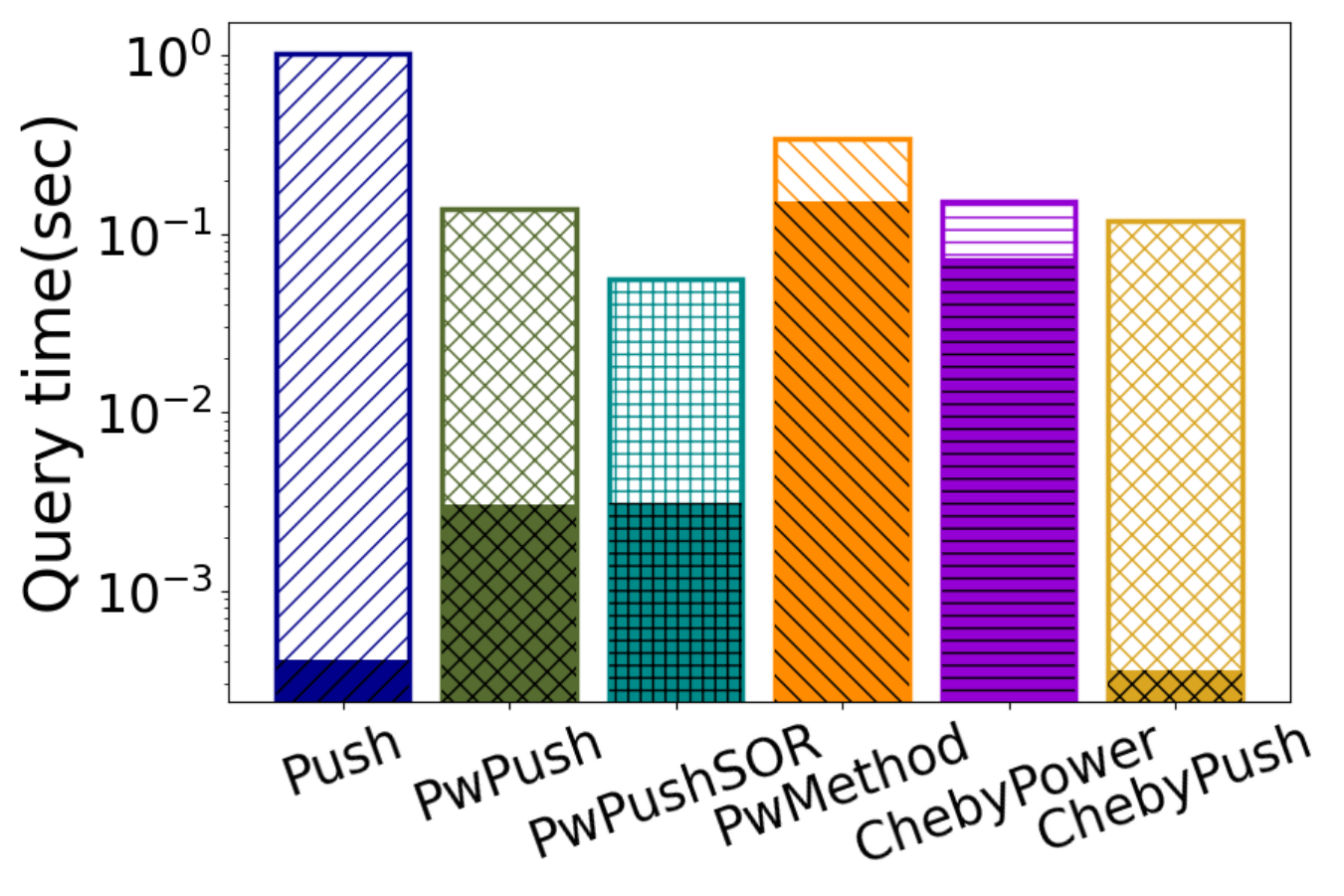}\hspace{-5mm} \label{1}
	}
	\quad
        \subfigure[\youtube]{
        \centering
		\includegraphics[scale=0.14]{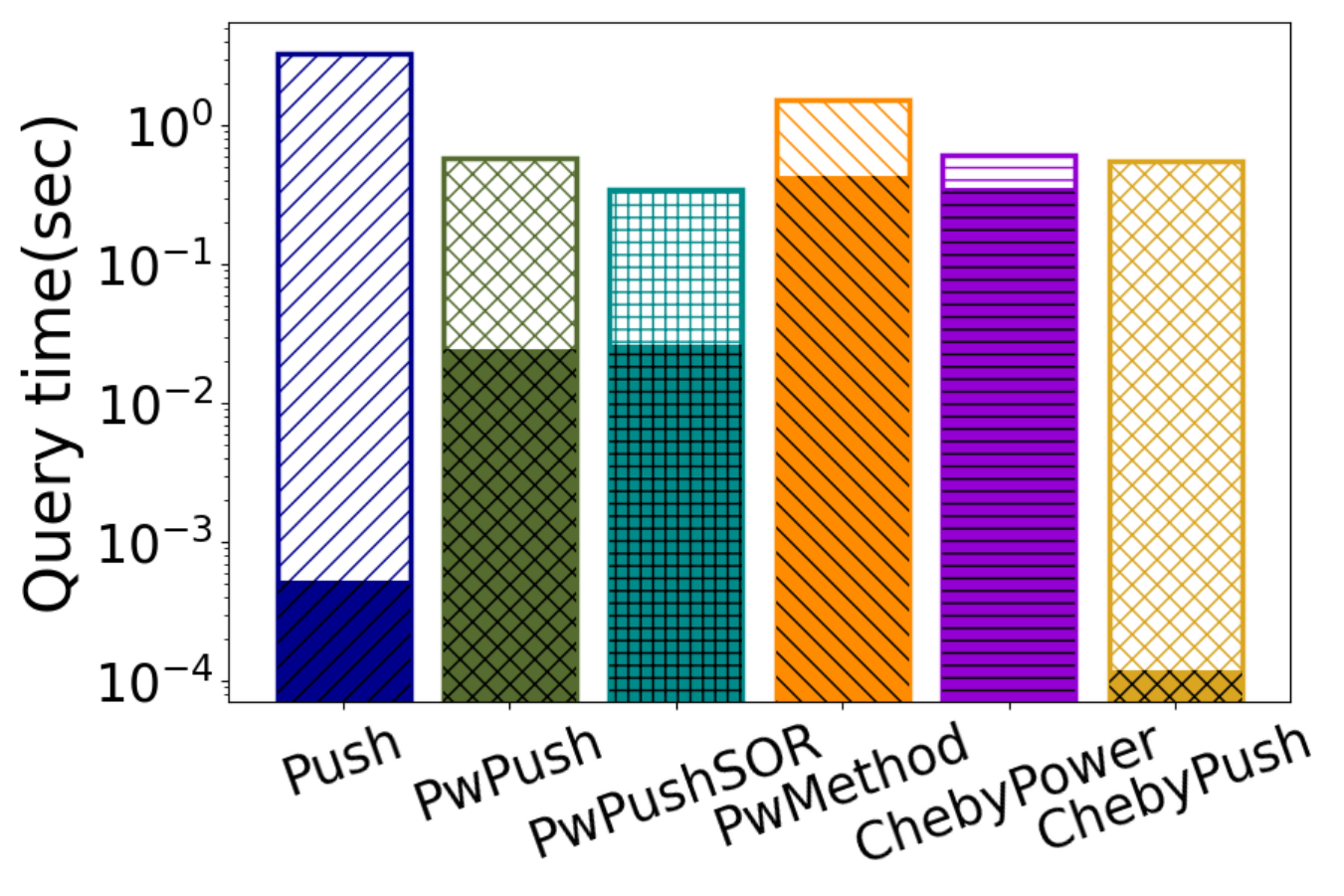}\hspace{-5mm}  \label{2}
	}
	\quad
	\subfigure[\livejournal]{
        \centering
		\includegraphics[scale=0.14]{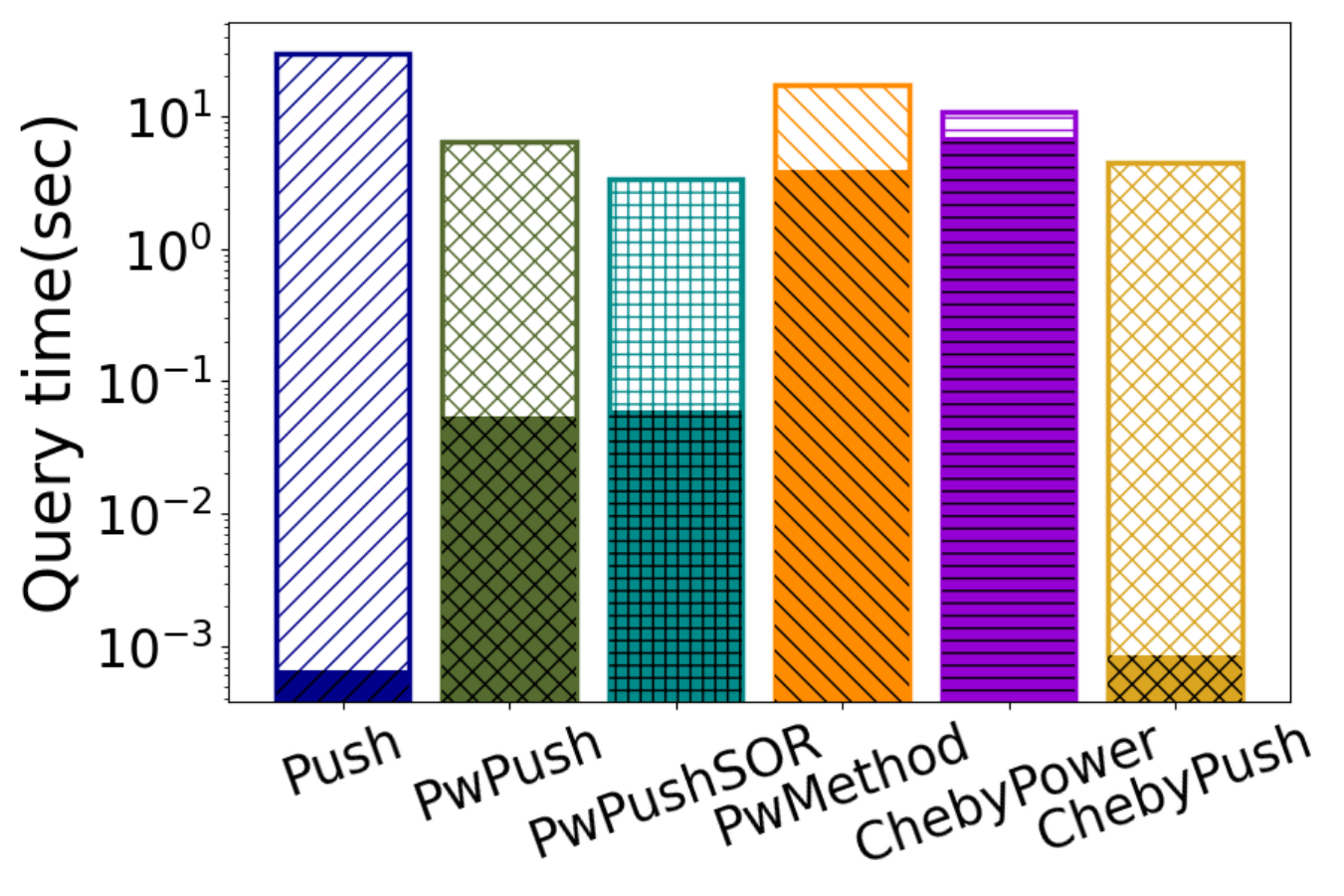}\hspace{-5mm}  \label{3}
	}
	\quad
	\subfigure[\orkut]{
        \centering
		\includegraphics[scale=0.14]{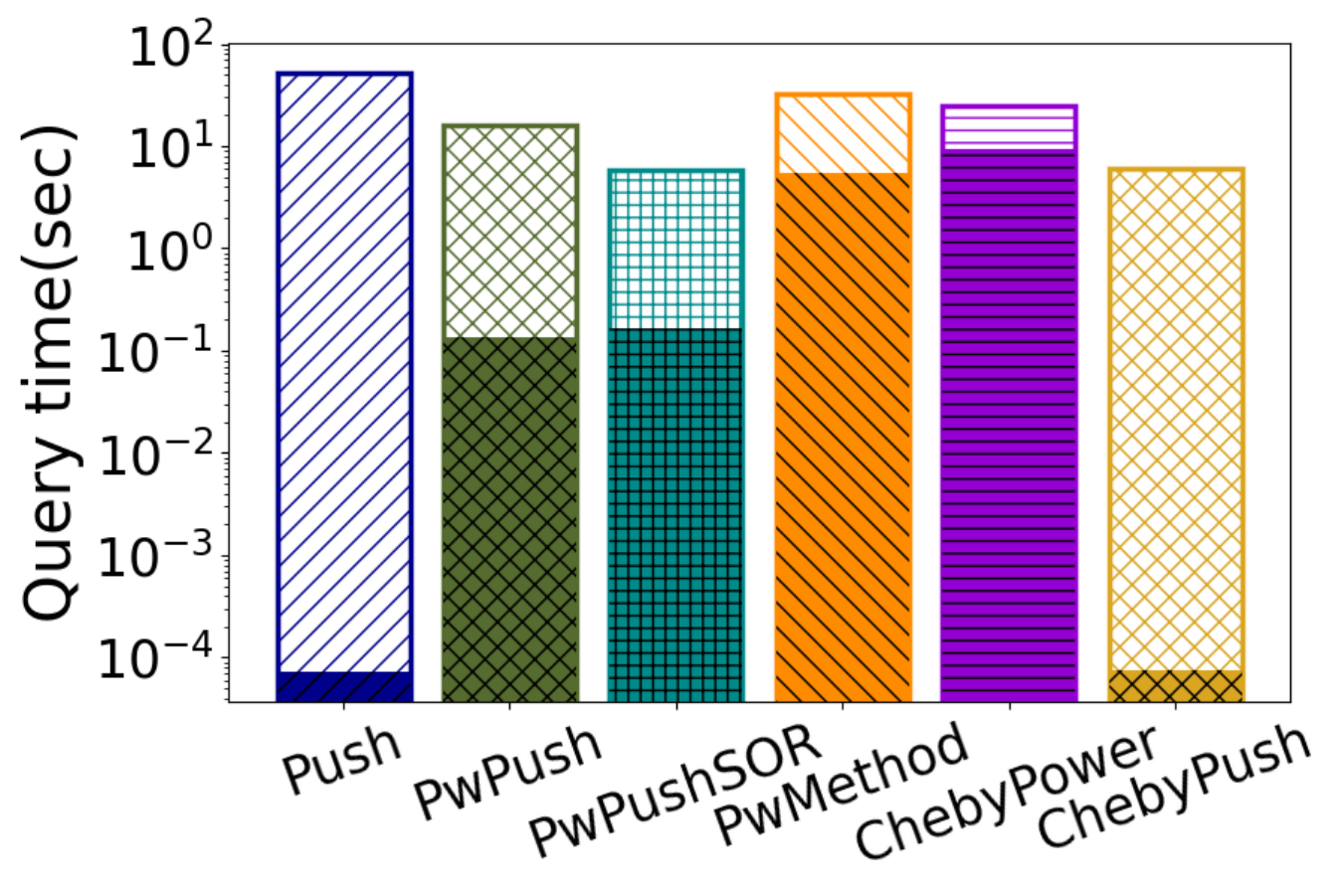}\label{4}
	}
    \subfigure[\friendster]{
        \centering
		\includegraphics[scale=0.14]{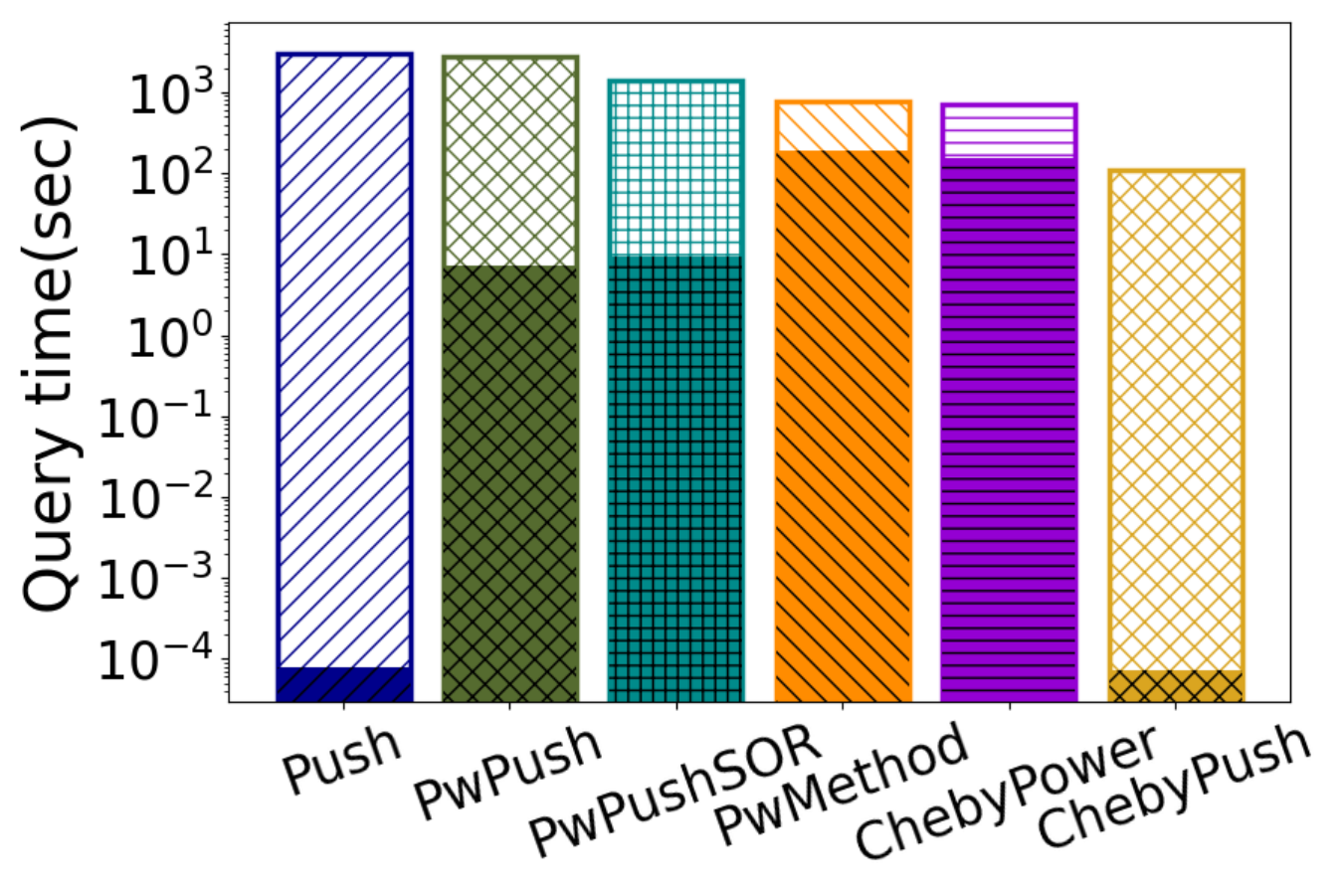}\label{4}
	}
  \vspace{-0.3cm}
	\caption{Query time of different \ssppr algorithms under Normalized RelErr. The lower (resp., upper) bar of each figure represents the query time of each algorithm to reach low-precision (resp., high-precision) Normalized RelErr.}\label{fig:ssppr-relerr}
\vspace{-0.2cm}
\end{figure*}

\begin{figure}[t!]
    \centering
    \subfigure[\dblp]{
		\includegraphics[scale=0.24]{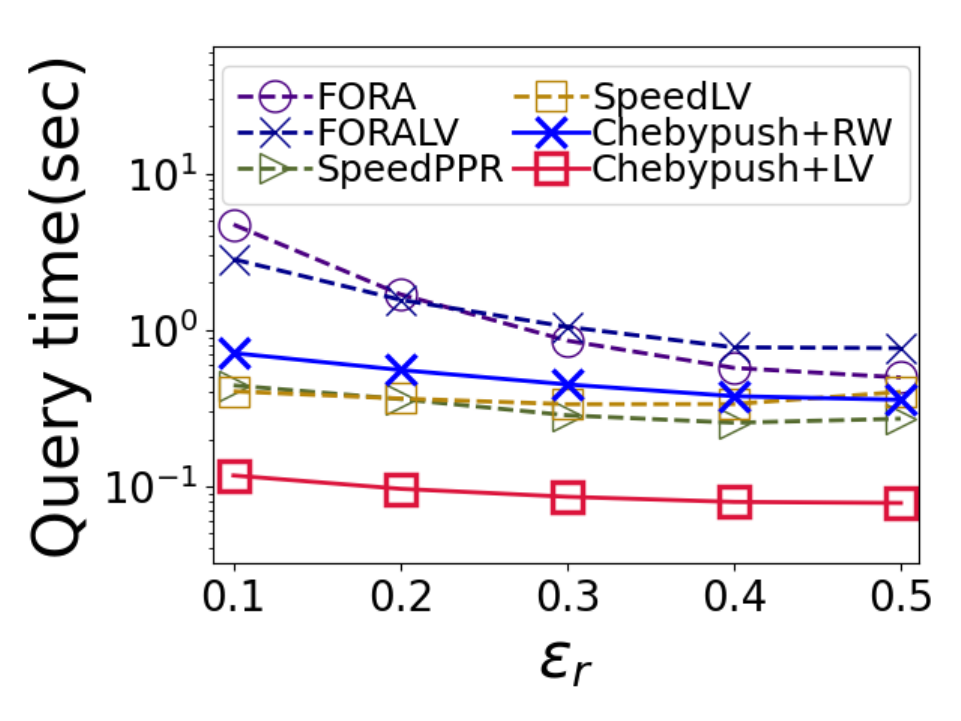}}
     \subfigure[\livejournal]{
		\includegraphics[scale=0.24]{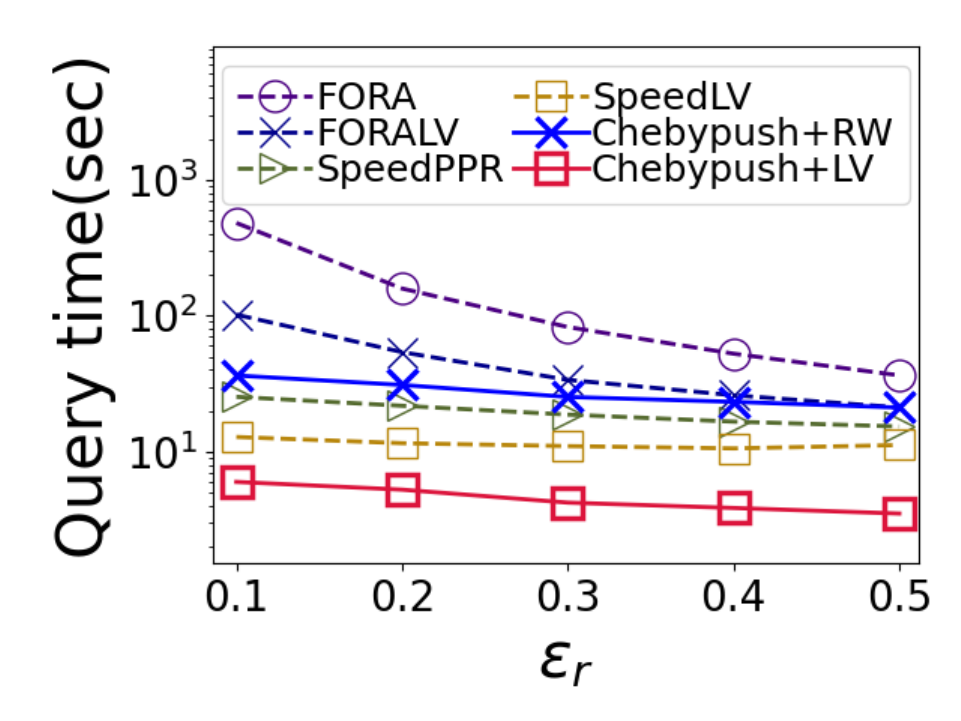}}
\vspace{-0.3cm}
    \caption{Comparison of various bidirectional methods}
    \label{fig:bidirectional-runtime}
\end{figure}
\section{Experiments}\label{sec:experiment}
We use 5 publicly-available datasets\footnote{All datasets can be downloaded from http://snap.stanford.edu/} with varying sizes (Table~\ref{tab:dataset}), including a graph (i.e., Friendster) with billions of edges, which are popular benchmarks for graph propagation (GP) computation \cite{wang2021approximate,wang2023singlenode}. To evaluate the approximation errors of different algorithms, we compute \textit{ground-truth}  GP vectors using \powermethod with a sufficiently large truncation step $N$. Specifically, to compute the \textit{ground-truth} GP vector, following previous studies \cite{andersen2006local, wang2017fora, kloster2014heat, wang2021approximate}, we set the truncation step $N=\frac{1}{\alpha}\log{(10^{20})}$ for \ssppr and $N=2t\log{(10^{20})}$ for \hkpr. Following \cite{wang2021approximate,wang2023singlenode}, we randomly generate 10 source nodes as query sets and report the average performance over them for different algorithms. We  will study the effect of different source node selection strategies in Section~\ref{subsec:source-node-distri}.  By Definition \ref{def:matrics}, 
we use the $l_1$-error and degree-normalized $l_\infty$-error (denoted as Normalized RelErr) to evaluate the estimation error of different algorithms. Following~\cite{wang2022edge}, we evaluate all algorithms under two different cases: low-precision and high-precision cases. For the low-precision (resp., high-precision) case, we set the error threshold as Normalized RelErr=$10^{-5}$ or  $l_1$-error$=10^{-1}$  (resp., Normalized RelErr=$10^{-10}$  or $l_1$-error$=10^{-5}$). We conduct all experiments on a Linux 20.04 server with an Intel 2.0 GHz CPU and 128GB memory. All algorithms are implemented in C++ and compiled using GCC 9.3.0 with -O3 optimization. the source code of this paper is available at \url{https://anonymous.4open.science/r/ChebyPush-2E94}.

\subsection{Results of \ssppr Vector Computation} \label{exp:ssppr}
\stitle{Comparison of various power-iteration and push algorithms.} In this experiment, we compare the proposed  \ltwocheb and \chebpush algorithms with four SOTA baselines for \ssppr vector computation, including \powermethod~\cite{page1999pagerank}, \push~\cite{andersen2006local}, \powerpush~\cite{wu2021unifying}, and \powerpushsor ~\cite{chen2023accelerating}. Note that both \powerpush and \powerpushsor are two highly-optimized \textit{push-style} algorithms designed for \ssppr computation. For \powermethod and \ltwocheb, by Lemma~\ref{lem:converge_rate_cheb}, we set their truncation steps as $N=\frac{1}{\alpha}\log{(\frac{1}{\epsilon})}$ and $K=\frac{1}{\sqrt{\alpha}}\log{(\frac{1}{\epsilon})}$, respectively, with varying $\epsilon\in [10^{-1},10^{-9}]$ to achieve different error value. For \chebpush, we set the truncation step $K=\frac{1}{\sqrt{\alpha}}\log{(10^5)}$ and vary the threshold $\epsilon_a\in [10^{-3},10^{-11}]$ to achieve different error values. For \push, \powerpush, and \powerpushsor, we adopt the same parameter settings in their original studies \cite{wu2021unifying,chen2023accelerating}. Unless specified otherwise, the damping factor $\alpha$ in \ssppr is set to 0.2, a value widely employed in previous studies \cite{page1999pagerank,wu2021unifying,chen2023accelerating}.

Figure~\ref{fig:ssppr-l1err} and Figure~\ref{fig:ssppr-relerr} report the query time of various algorithms for \ssppr vector computation, under the $l_1$-error and Normalized RelErr metrics respectively. In Figure~\ref{fig:ssppr-l1err} and Figure~\ref{fig:ssppr-relerr}, the lower bar of each figure represents the query time of each algorithm to reach the low-precision approximation, while the upper bar represents the query time to reach the high-precision approximation. From these figures, we have the following observations: (i) 
For the $l_1$-error, \chebpush performs much better than all competitors on most datasets under the low-precision case. On the other hand,  under the high precision case, \chebpush consistently outperforms all competitors except  \powerpushsor. Note that \chebpush is slightly worse than \powerpushsor at high precision on small-sized graphs, but it significantly outperforms \powerpushsor on the largest dataset \friendster. (ii) For the Normalized RelErr, \chebpush is significantly faster than all baseline methods on large graphs under both high precision and low precision scenarios. On relatively-small graphs, it still demonstrates comparable performance to the baseline methods. (iii) For both $l_1$-error and Normalized RelErr, \ltwocheb  is consistently faster than its counterpart \powermethod under both high and low precision cases. These results demonstrate the high efficiency of the proposed algorithms, and also confirm our theoretical analysis in Sections~\ref{sec:interpretation} and ~\ref{sec:local-algo}. 

Additionally, the results of Figure ~\ref{fig:ssppr-l1err} 
 and Figure  ~\ref{fig:ssppr-relerr} also suggest that \chebpush is very suitable for applications where the required estimation error is not excessively high, as \chebpush is extremely fast under the low-precision case. Therefore, \chebpush can be highly effective and efficient for designing bidirectional algorithms, because bidirectional algorithms often do not necessitate achieving high accuracy in \ssppr estimation during the push phase (Section \ref{subsec:bidirect-method}). Besides, we also observe that \powerpushsor performs well for high-precision \ssppr vector computation. However, unlike our \ltwocheb and \chebpush algorithm, \powerpushsor is limited to computing \ssppr vector and cannot be used for general GP vector computations (e.g., \hkpr), resulting in limited applications.

\stitle{Comparison of various bidirectional algorithms.} As discussed in Section \ref{subsec:bidirect-method}, our \chebpush algorithm can also be combined with Monte-Carlo methods to generalize into bidirectional algorithms. We apply the following widely used methods for the Monte-Carlo phase of bidirectional algorithms: random walk (\rw) based methods \cite{wang2017fora,wu2021unifying} and loop-erased random walk (\lv) based methods \cite{liao2022efficient}. We use the SOTA variance reduction version for \lv-based algorithms \cite{liao2022efficient}. We compare our bidirectional methods \chebpush+\rw and \chebpush+\lv (\chebpush for the first phase, \rw or \lv for the second phase) with four SOTA competitors \fora (\push+\rw) ~\cite{wang2017fora}, \speedppr (\pwpush+\rw) ~\cite{wu2021unifying},  \speedlv (\pwpush + \lv) ~\cite{liao2022efficient}, \foralv (\push+\lv). For \chebpush based bidirectional algorithms (i.e., \chebpush+\rw and \chebpush+\lv), we set $K=\frac{1}{\sqrt{\alpha}}\log{(10^5)}$. The parameter settings of the other bidirectional algorithms follow their original studies \cite{wang2017fora, wu2021unifying, liao2022efficient}. Note that all bidirectional algorithms are designed to achieve an $\epsilon_r$-relative error guarantee. Thus, we evaluate the query time of different bidirectional algorithms to achieve the given $\epsilon_r$-relative errors with varying $\epsilon_r$ from $0.1$ to $0.5$. 

The results on \dblp and \livejournal are shown in Figure \ref{fig:bidirectional-runtime}. Similar results can also be observed on the other datasets. As can be seen, \chebpush+\rw performs better than \fora and similarly to \speedppr, while \chebpush+\lv significantly outperforms all bidirectional algorithms. For instance, \chebpush+\lv achieves query times $2\times$ to $6\times$ better than the best baseline algorithm (\speedlv) to reach the same $\epsilon_r$-relative error. This is because \lv is designed to accelerate \rw computation, whereas our \chebpush accelerates \push and \powermethod computations. Thus, it is unsurprising that \chebpush+\lv achieves the best results. We also note that the query time of each algorithm decreases with an increasing $\epsilon_r$. The reason is as follows. The results obtained by bidirectional algorithms are more accurate as $\epsilon_r$  decreases. This increased accuracy comes at the cost of longer query costs for various bidirectional algorithms. These results further demonstrate the high efficiency of the proposed algorithms. 


  \begin{figure*}[t!]
    \vspace{-0.3cm}
    \centering
 	\subfigure[\dblp]{
            \centering
		\includegraphics[scale=0.14]{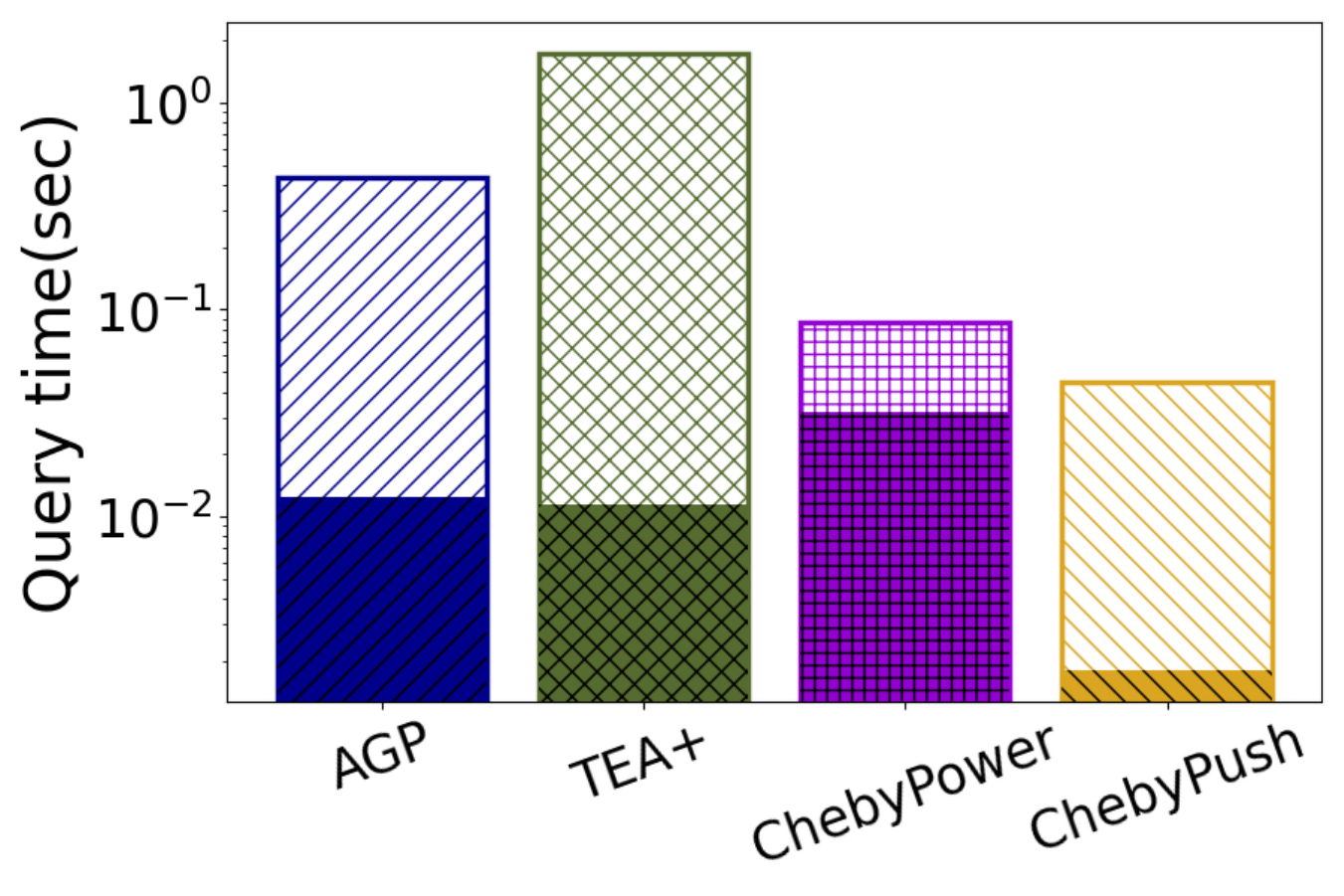}\hspace{-5mm} \label{1}
	}
	\quad
        \subfigure[\youtube]{
        \centering
		\includegraphics[scale=0.14]{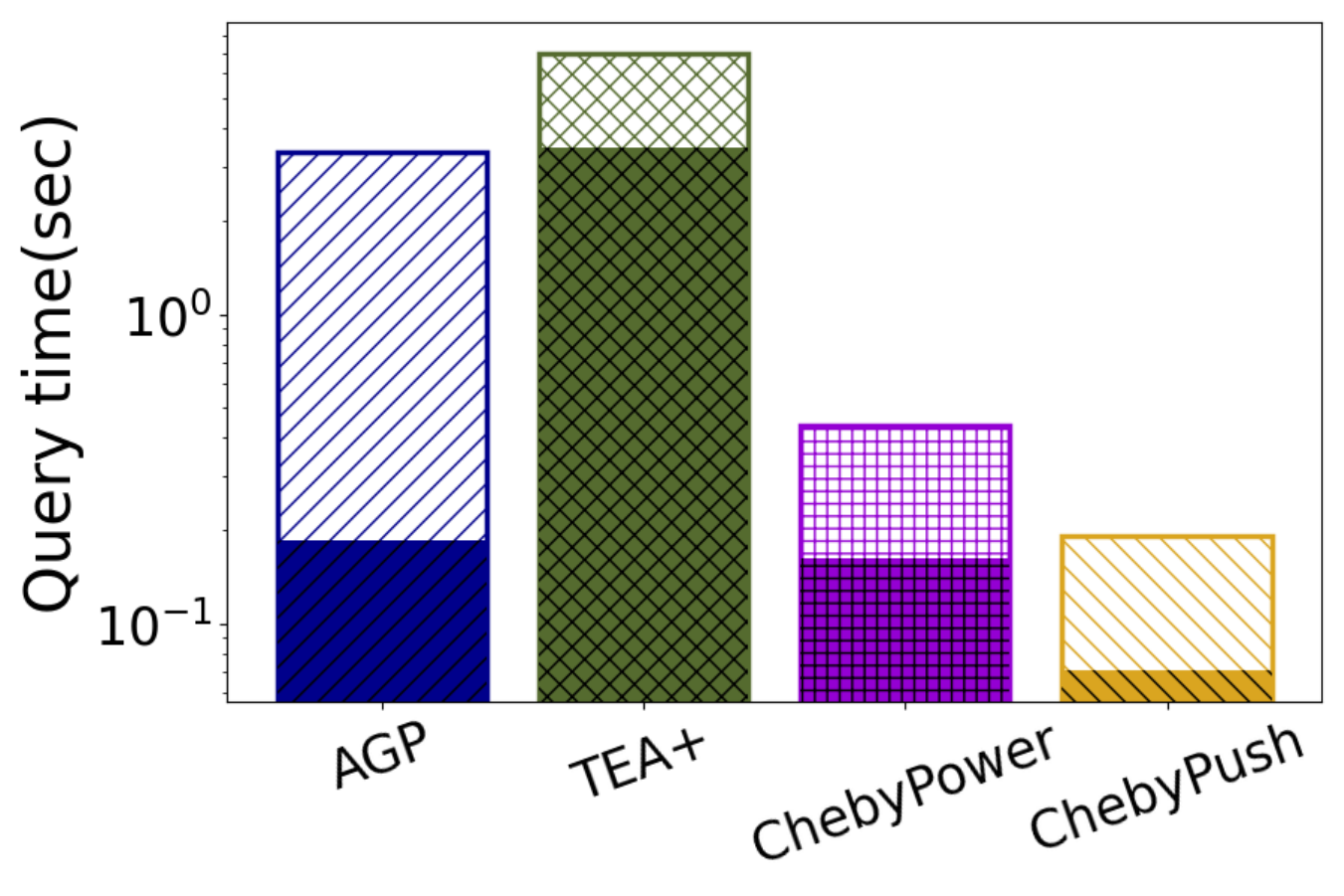}\hspace{-5mm}  \label{2}
	}
	\quad
	\subfigure[\livejournal]{
        \centering
		\includegraphics[scale=0.14]{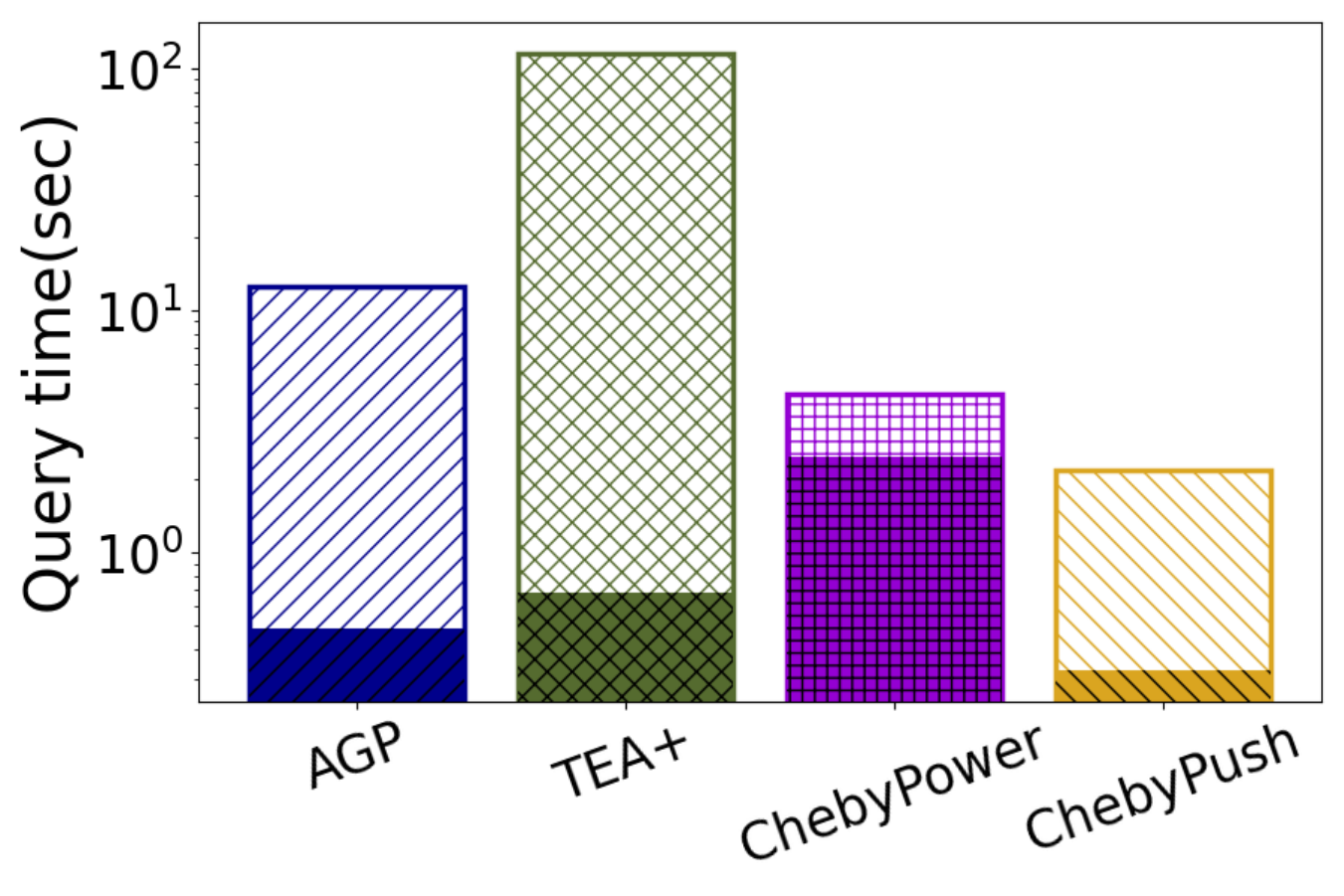}\hspace{-5mm}  \label{3}
	}
	\quad
	\subfigure[\orkut]{
        \centering
		\includegraphics[scale=0.14]{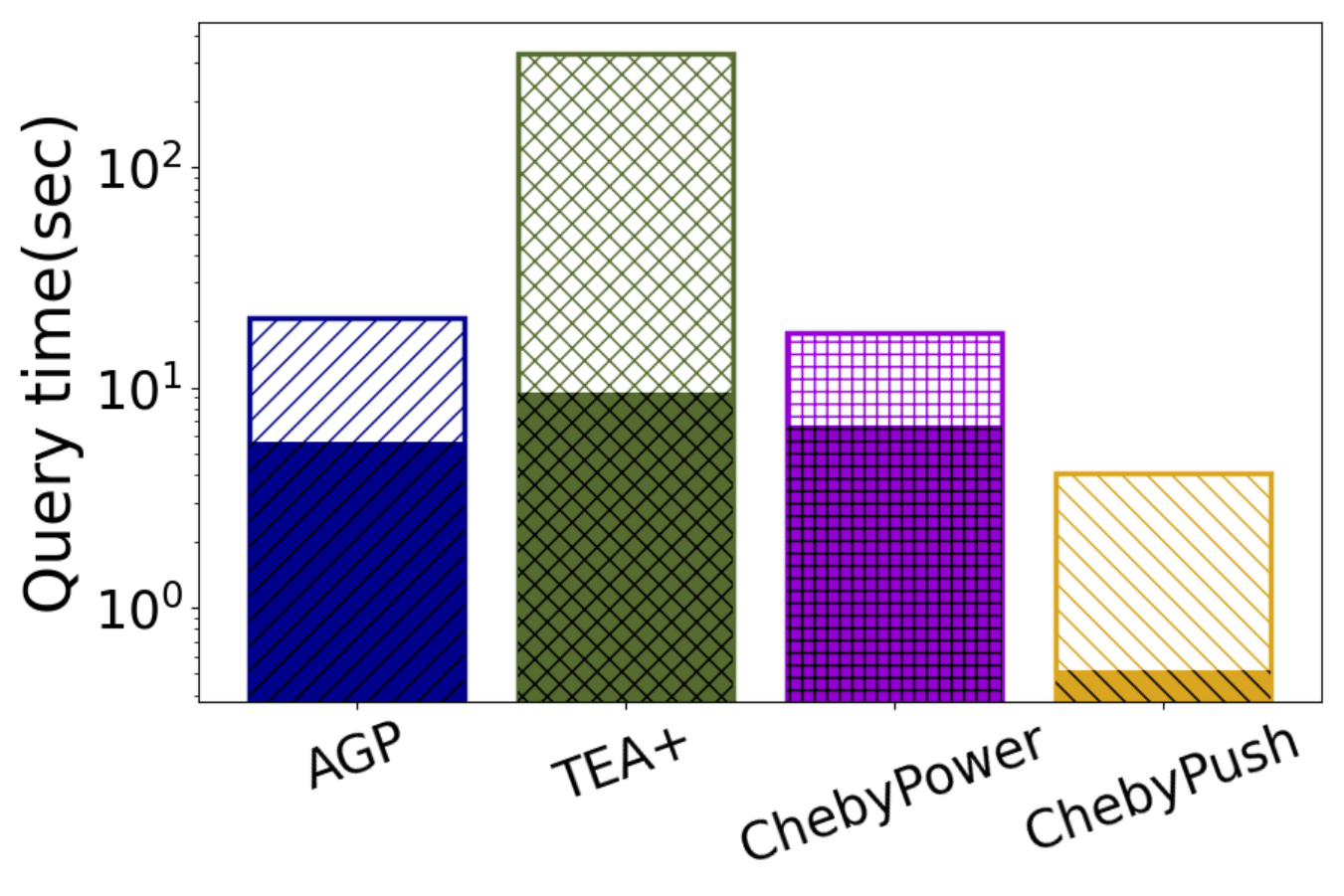}\label{4}
	}
 \subfigure[\friendster]{
        \centering
		\includegraphics[scale=0.14]{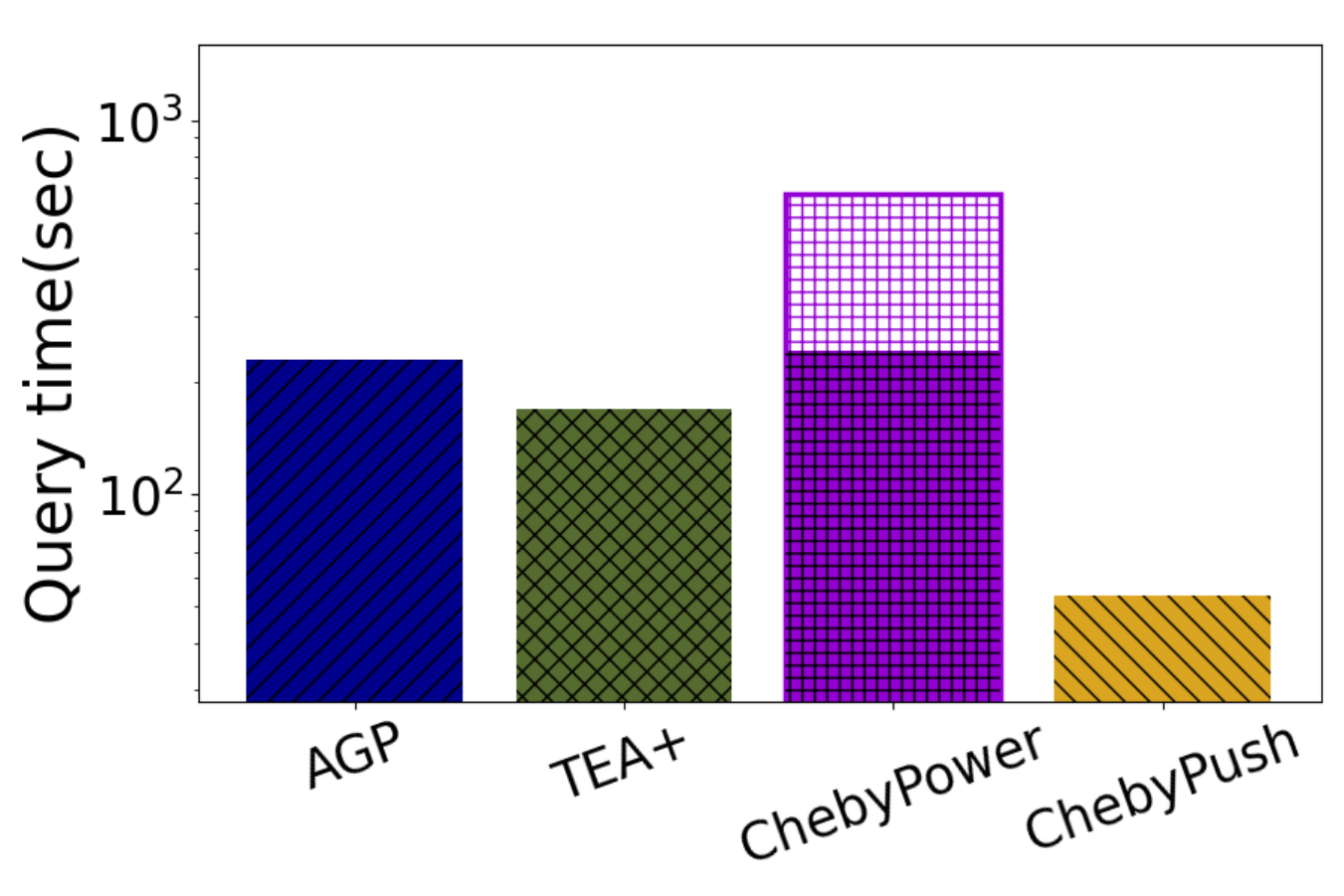}\label{4}
	}\vspace{-0.3cm}
	\caption{Query time of different \hkpr algorithms under $l_1$-error. The lower (resp., upper) bar of each figure represents the query time of each algorithm to reach low-precision (resp., high-precision) $l_1$-error.}\label{fig:hkpr-l1err}
\vspace{-0.2cm}
\end{figure*}

\begin{figure*}[t!]
\vspace{-0.3cm}
	\centering
	\subfigure[\dblp]{
            \centering
		\includegraphics[scale=0.14]{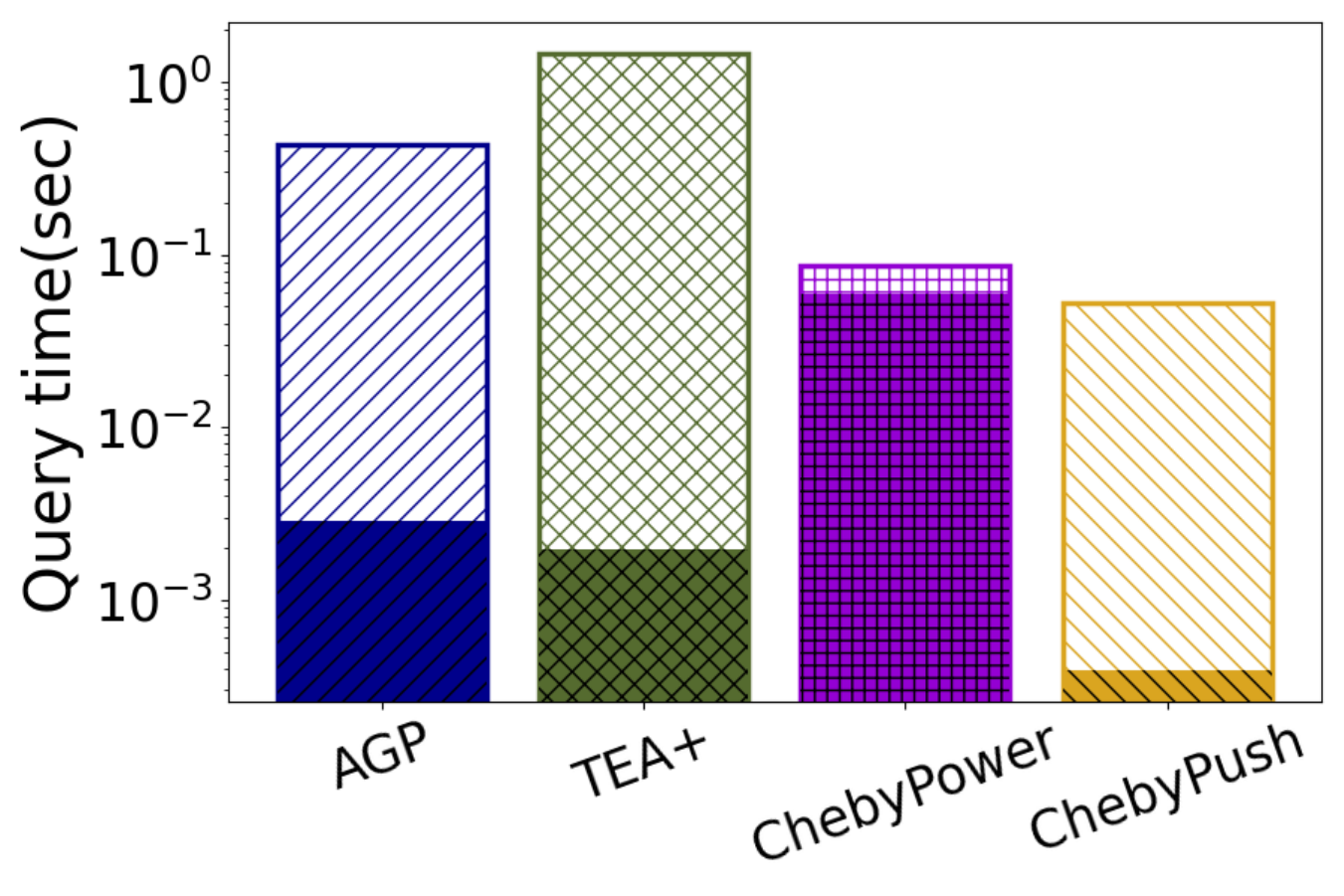}\hspace{-5mm} \label{1}
	}
	\quad
        \subfigure[\youtube]{
        \centering
		\includegraphics[scale=0.14]{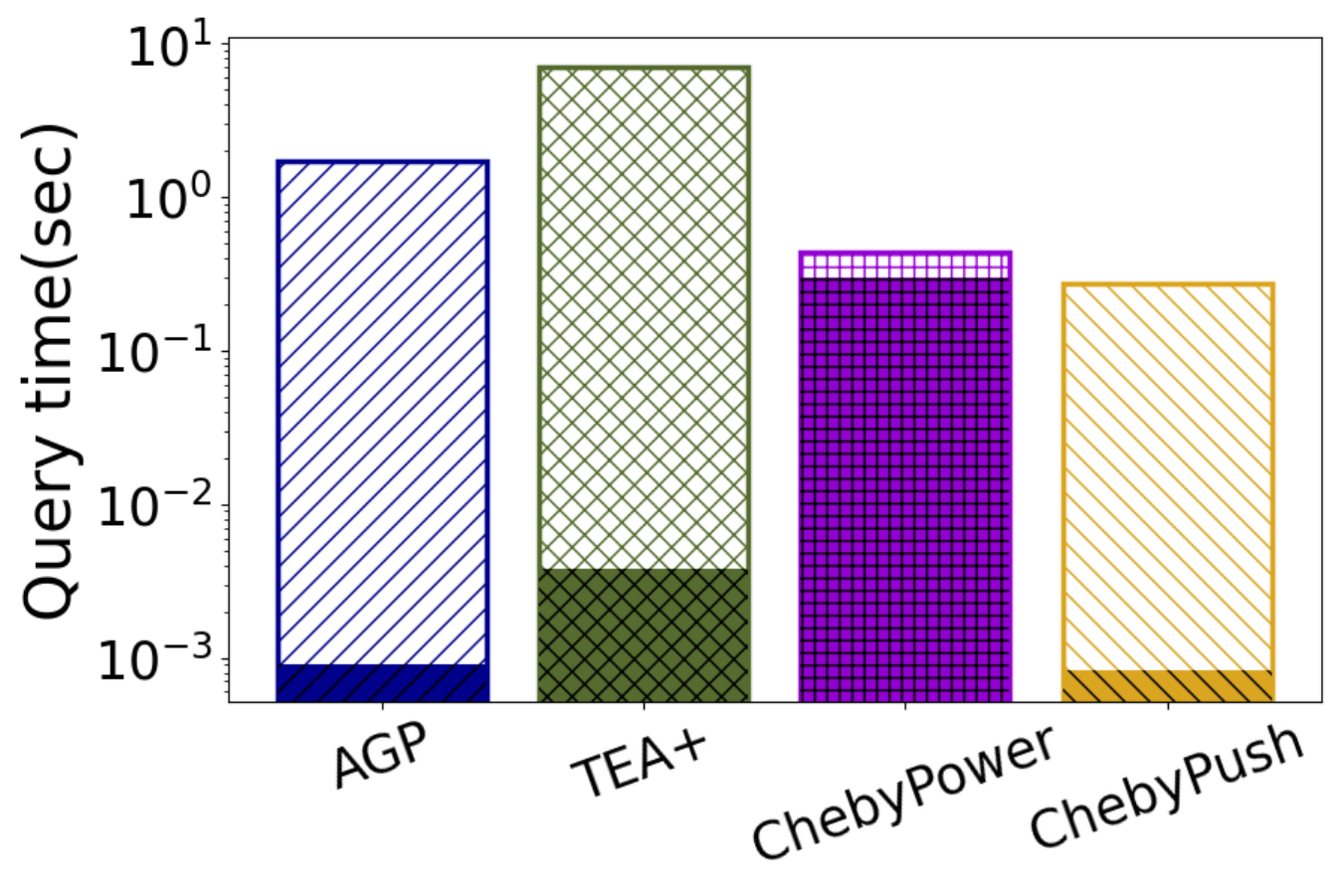}\hspace{-5mm}  \label{2}
	}
	\quad
	\subfigure[\livejournal]{
        \centering
		\includegraphics[scale=0.14]{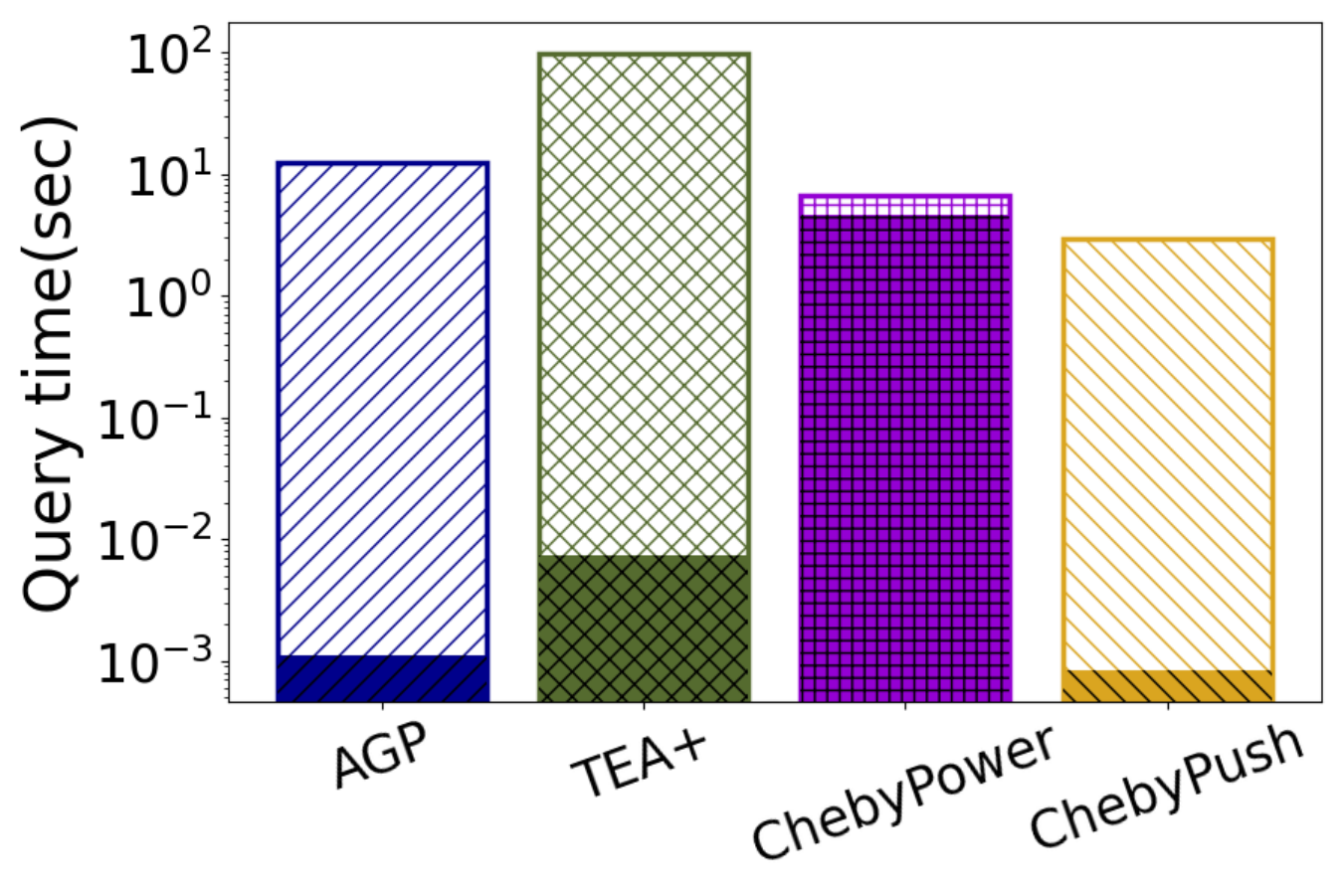}\hspace{-5mm}  \label{3}
	}
	\quad
	\subfigure[\orkut]{
        \centering
		\includegraphics[scale=0.14]{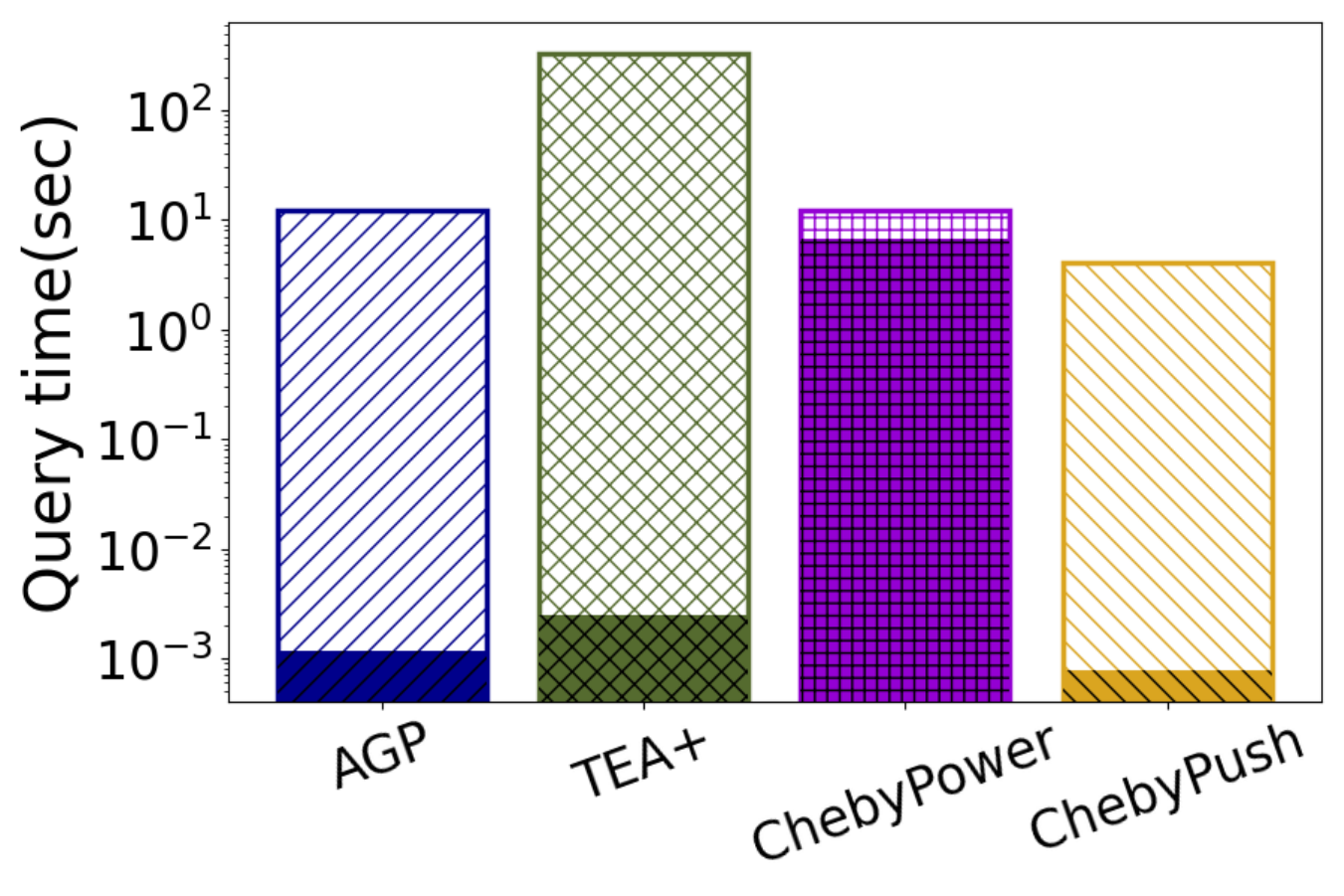}\label{4}
	}
    \subfigure[\friendster]{
        \centering
		\includegraphics[scale=0.14]{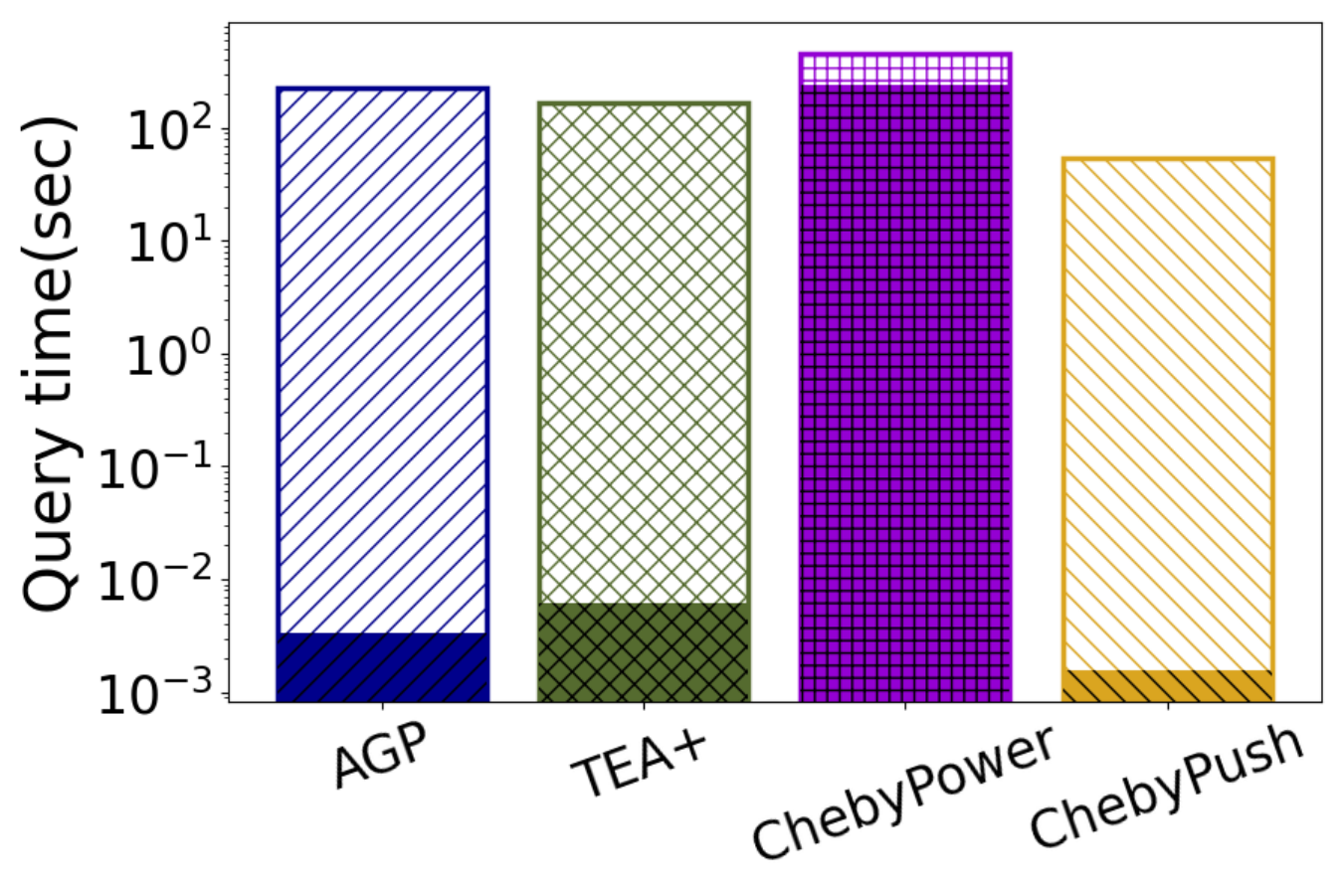}\label{4}
	}\vspace{-0.3cm}
    \caption{Query time of different \hkpr algorithms under Normalized RelErr. The lower (resp., upper) bar of each figure represents the query time of each algorithm to reach low-precision (resp., high-precision) Normalized RelErr.}\label{fig:hkpr-relerr}
    \vspace{-0.2cm}
    \end{figure*}

\subsection{Results of \hkpr Vector Computation} \label{subsec:exp-hkpr}


In this experiment, we compare our algorithms \ltwocheb and \chebpush with the state-of-the-art competitors \agp \cite{wang2021approximate} and \teaplus \cite{yang2019efficient} for computing \hkpr vectors. It is important to note that both \agp and \teaplus have shown significant performance improvements over \hkrelax \cite{kloster2014heat} and ClusterHKPR \cite{chung2018computing}, as reported in \cite{wang2021approximate} and \cite{yang2019efficient}. Therefore, in this experiment, we exclude \hkrelax and ClusterHKPR from comparison. Similar to \ssppr, for \ltwocheb we set $K=\sqrt{t}\log{(\frac{1}{\epsilon})}$ with $\epsilon$ in $[10^{-1}, 10^{-9}]$. For \chebpush, we set $K=\sqrt{t}\log{(10^5)}$, and vary the threshold $\epsilon_a\in [10^{-3},10^{-11}]$. The parameter settings of \agp and \teaplus follow their original papers~\cite{yang2019efficient,wang2021approximate}. Unless specified otherwise, we set $t=5$ following previous studies ~\cite{kloster2014heat,yang2019efficient,wang2021approximate}. 

Figure~\ref{fig:hkpr-l1err} and Figure~\ref{fig:hkpr-relerr} show the query time of various algorithms for \hkpr vector computation under the $l_1$-error and Normalized RelErr metrics, respectively. Similar to Figure ~\ref{fig:ssppr-l1err} and Figure~\ref{fig:ssppr-relerr}, the lower bar of each figure represents the query time for low-precision computation and the upper bar represents the query time for high-precision computation. We have the following important observations: (i) For the $l_1$-error metric, our \chebpush algorithm substantially outperforms all other methods under the low-precision case. Under the high-precision case, both of our \ltwocheb and \chebpush consistently outperform the baselines on all datasets. In general, \chebpush is $3\times$ to $8\times$ faster than the best baseline algorithm. Note that for \friendster with billions of edges, none of the existing algorithms can output the results within a few seconds under the high-precision case. (ii) For the Normalized RelErr metric, our \chebpush algorithm performs much better than the baseline methods under the low-precision case. For the high-precision case, both \ltwocheb and \chebpush are better than baselines on most datasets (on Friendster, \chebpush is the champion algorithm and \ltwocheb is slightly worse than baselines). These results indicate that our algorithms are indeed significantly accelerating existing algorithms for \hkpr vector computation, which is consistent with our theoretical analysis.

\comment{
\subsection{Experiments for single-node PageRank}

For single-node PageRank, we compare our algorithms \ltwocheb and \chebpush with SOTA sublinear algorithms: \setpush ~\cite{wang2023singlenode} and \bippr~\cite{lofgren16bidirection,wang2024revisiting}. For \ltwocheb, we vary the parameter $\epsilon$ in $[10^{-1},10^{-9}]$. For \chebpush, we set $K=\frac{1}{\sqrt{\alpha}}\log{(10^5)}$, and vary $\epsilon$ in $[10^{-3},10^{-11}]$. For \setpush, we vary the parameter $c$ in $[\frac{1}{2}\times 10^{-4},\frac{1}{2}]$ and other parameters follow the setting in ~\cite{wang2023singlenode}. For \bippr, we follow the optimal setting of ~\cite{wang2024revisiting}. First, we set the number of random walks $n_r=n^{1/2}\min{\{\Delta^{1/2},m^{1/4}\}}c^{-1}$ and run \push with threshold $\epsilon=1/2,1/4,1/8...$ until the total cost of push is $\Theta(n_r)$, then we perform $n_r$ random walks. We vary the parameter $c$ in $[\frac{1}{2}\times 10^{-4},\frac{1}{2}]$ and other parameters follow the setting of ~\cite{wang2024revisiting}. To evaluate the performance of different algorithms, we use the Relative Error metric, which is defined as $\frac{|P(u)-\hat{P}(u)|}{P(u)}$, where $P(u)$ is the ground-truth answer, and $\hat{P}(u)$ is the approximation.
Figure ~\ref{fig:single-node} reports the performance of different algorithms for computing single-node PageRank under the Relative Error metric. The observation is that our algorithm \chebpush performs $3\times$ to $10\times$ faster than \setpush and \bippr under higher Relative Error.

\begin{figure*}[!t]
\vspace{-0.3cm}
	\centering
	\subfigure[\dblp]{
            \centering
		\includegraphics[scale=0.2]{experiment-figure/single-node/dblp.png}\hspace{-5mm} \label{1}
	}
	\quad
        \subfigure[\youtube]{
        \centering
		\includegraphics[scale=0.2]{experiment-figure/single-node/youtube.png}\hspace{-5mm}  \label{2}
	}
	\quad
	\subfigure[\livejournal]{
        \centering
		\includegraphics[scale=0.2]{experiment-figure/single-node/livejournal.png}\hspace{-5mm}  \label{3}
	}
	\quad
	\subfigure[\orkut]{
        \centering
		\includegraphics[scale=0.2]{experiment-figure/single-node/orkut.png}\label{4}
	}
 \subfigure[\orkut]{
        \centering
		\includegraphics[scale=0.2]{experiment-figure/single-node/orkut.png}\label{4}
	}\vspace{-0.3cm}
	\caption{Trade-offs between running time v.s. Relative Error for single-node PageRank}\label{fig:single-node}
\vspace{-0.2cm}
\end{figure*}
}

\subsection{Results with Various Query Distributions} \label{subsec:source-node-distri}
In this experiment, we explore two strategies for selecting source nodes: (1) choosing the 10 nodes with the highest degrees, and (2) randomly selecting 10 nodes from the graph. We set the parameter $\alpha=0.2,0.02$ and $t=5,20$ for \ssppr and \hkpr respectively. For our \ltwocheb and \chebpush algorithms, we fix $\epsilon,\epsilon_a = 10^{-5}$. We present the results using box plots to depict the distribution of query times, as shown in Figure~\ref{fig:query-distribution}.

From Figure~\ref{fig:query-distribution}, we observe the following: (i) \ltwocheb exhibits insensitivity to different source node selections; the query costs for various source nodes are nearly identical. This behavior stems from \ltwocheb being a \textit{global} algorithm, where its runtime remains consistent regardless of the source node. (ii) \chebpush demonstrates greater sensitivity to the choice of source nodes. Generally, the query costs for \chebpush are higher when using high-degree source nodes compared to randomly selected ones. This disparity arises because \chebpush operates as a \textit{local} algorithm, where its time complexity depends on the local graph structure. High-degree nodes involve numerous neighbors contributing significant \ssppr and \hkpr values thus often cannot be pruned in the \chebpush procedure, resulting in higher computational costs. These findings align with the theoretical properties of our algorithms.

\begin{figure}
    \centering
    \subfigure[\dblp,$\alpha=0.2$,$t=5$]{
		\includegraphics[scale=0.15]{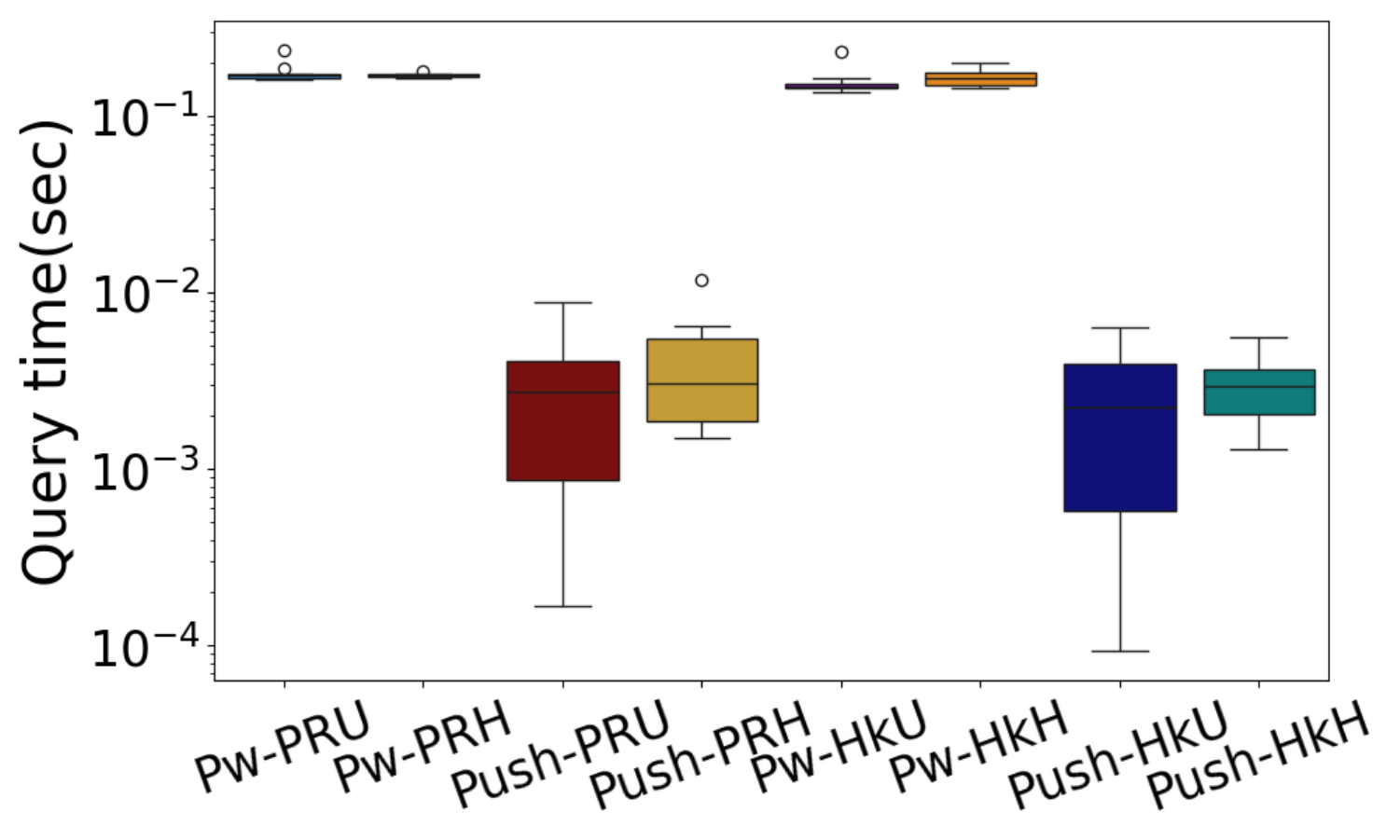}}
     \subfigure[\dblp,$\alpha=0.02$,$t=20$]{
		\includegraphics[scale=0.15]{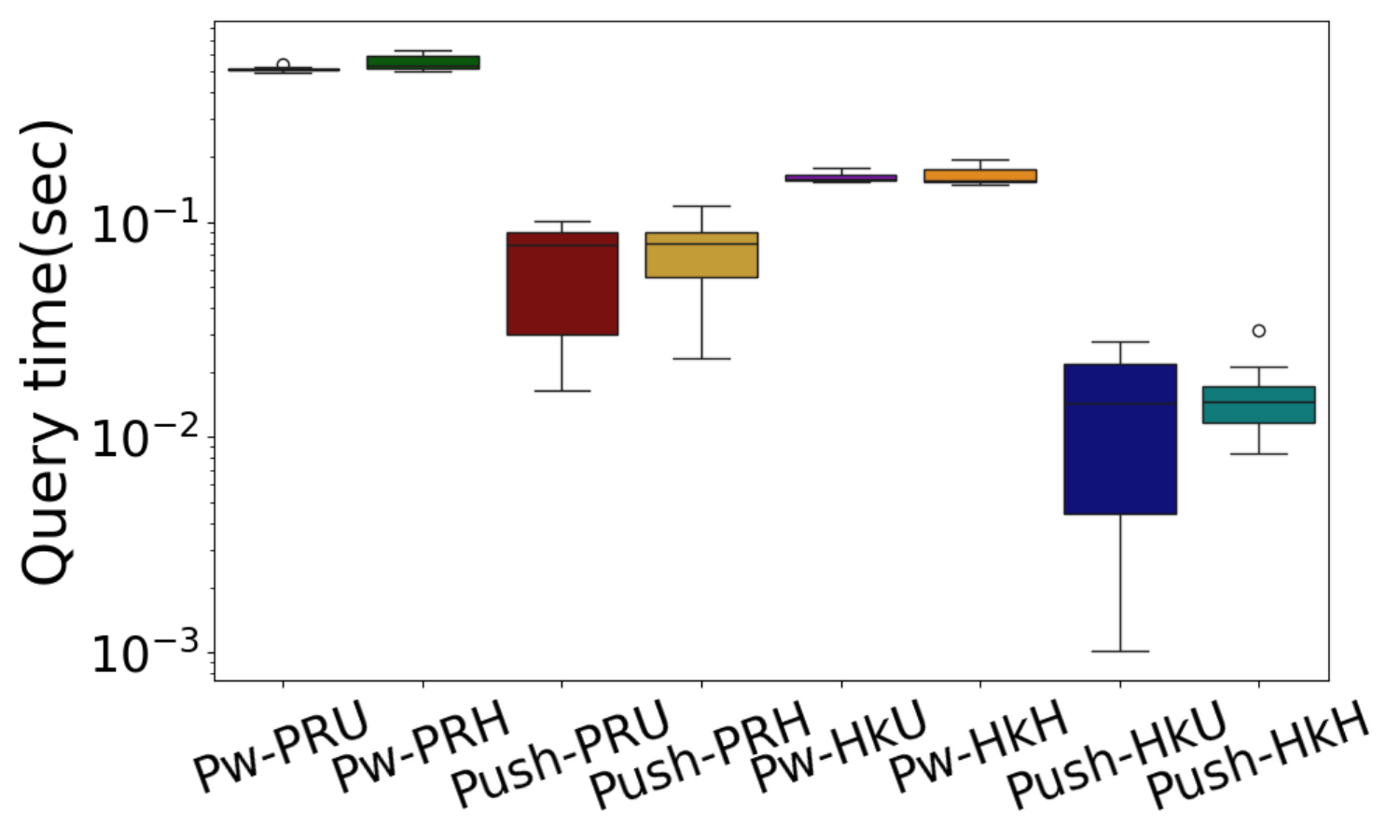}}

      \subfigure[\livejournal,$\alpha=0.2$,$t=5$]{
		\includegraphics[scale=0.15]{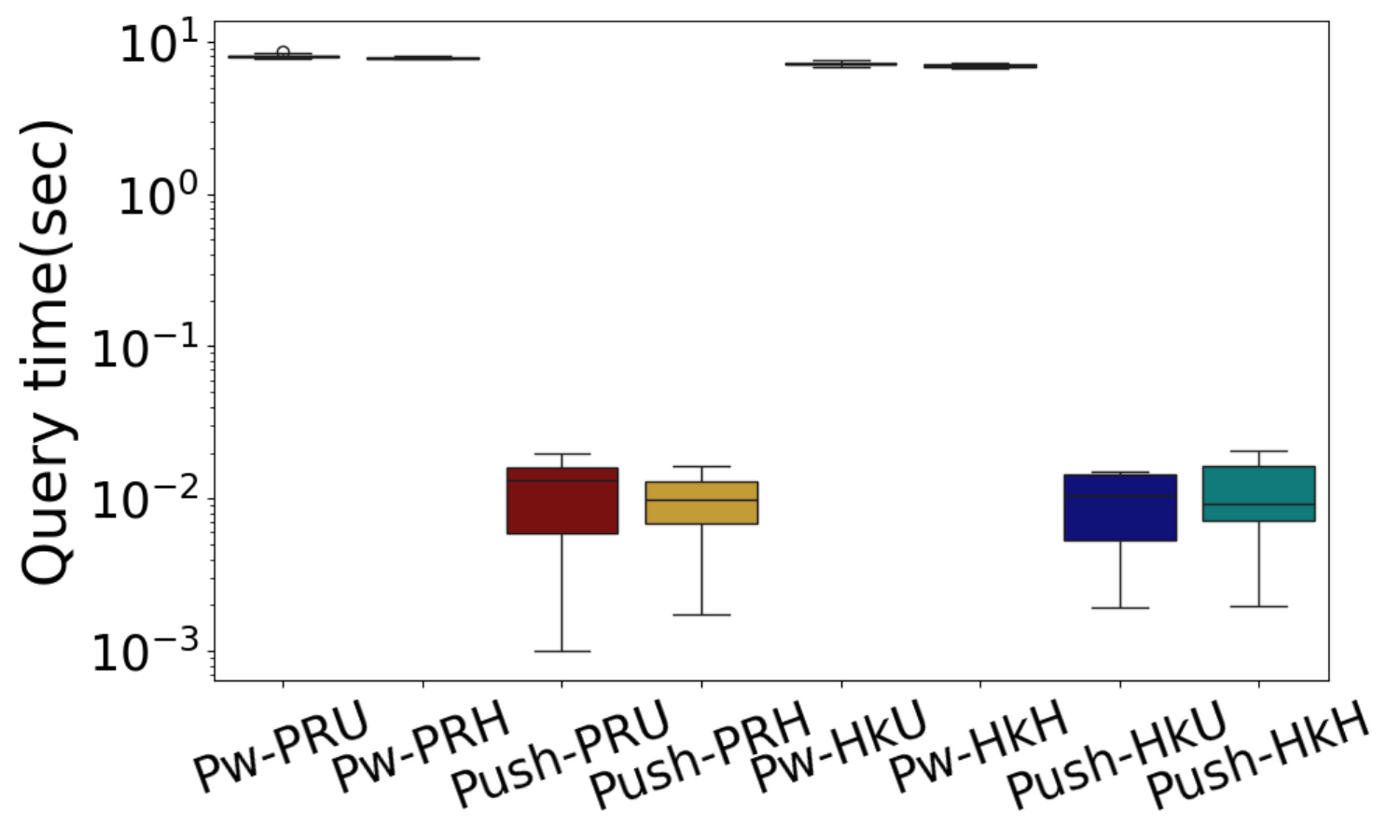}}
     \subfigure[\livejournal,$\alpha=0.02$,$t=20$]{
		\includegraphics[scale=0.15]{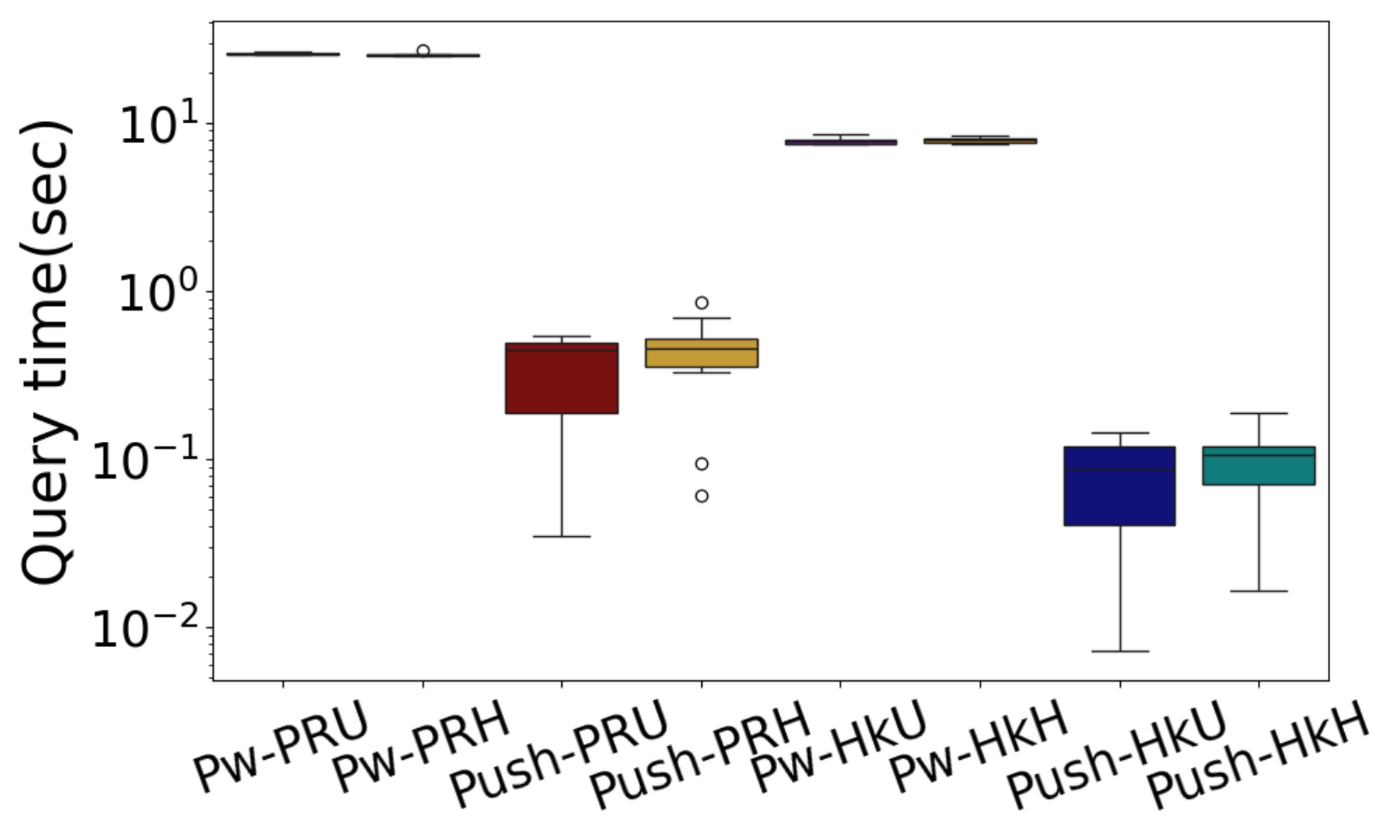}}
    \caption{Query time distribution for our \ltwocheb and \chebpush algorithms in computing \ssppr and \hkpr vectors. \kw{Pw-PRU} and \kw{Pw-HkU} denote \ltwocheb selecting source nodes uniformly for \ssppr and \hkpr, respectively. \kw{Pw-PRH} and \kw{Pw-HkH} represent \ltwocheb selecting source nodes with the highest degree for \ssppr and \hkpr, respectively. Similar naming conventions apply to \chebpush.}\label{fig:query-distribution}\vspace{-0.2cm}
\end{figure}

\comment{

\subsection{Results with Various Parameters $\alpha$ and $t$} \label{subsec:exp-parameter-effect}
In this experiment, we explore the effect of parameters $\alpha$ and $t$ on our \ltwocheb and \chebpush algorithms for computing \ssppr and \hkpr vectors respectively. 

\begin{figure}[t!]
    \centering
    \subfigure[\dblp, Normalized RelErr=$10^{-5}$]{
		\includegraphics[scale=0.24]{experiment-figure/ssppr/var-alpha-high-error.png}}
     \subfigure[\dblp, Normalized RelErr=$10^{-10}$]{
		\includegraphics[scale=0.24]{experiment-figure/ssppr/var-alpha-low-error.png}}
\vspace{-0.5cm}
    \subfigure[\dblp, $l_1$-error=$10^{-1}$]{
		\includegraphics[scale=0.24]{experiment-figure/ssppr/var-alpha-high-l1error.png}}
     \subfigure[\dblp, $l_1$-error=$10^{-5}$]{
		\includegraphics[scale=0.24]{experiment-figure/ssppr/var-alpha-low-l1error.png}}
    \caption{The effect of $\alpha$ on our \ssppr algorithms.}\label{fig:ssppr-var-alpha} \vspace{-0.3cm}
\end{figure}

\stitle{The effect of parameter $\alpha$.} For \ssppr vector computation, we vary $\alpha=0.2$ to $\alpha=0.02$ and evaluate the effect of $\alpha$ on the performance of our \ltwocheb and \chebpush algorithm. The results on \dblp are reported in Figure ~\ref{fig:ssppr-var-alpha}. Similar results can also be obtained on the other datasets. As can be seen, the query time of each algorithm increases as $\alpha$ decreases. This is because the time complexity both of \ltwocheb and \chebpush 

This experiment studies the performance of various algorithms with varying $\alpha$. Specifically, we vary $\alpha=0.2,0.1,0.02,0.002,0.0002$ and compare the performance of various algorithm under different $\alpha$ values. The results on \dblp are reported in Figure ~\ref{fig:ssppr-var-alpha}. Similar results can also be obtained on the other datasets. As can be seen, the query time of each algorithm increases as $\alpha$ decreases.

The conclusions are as follows. (i) For Normalized RelErr,  under the low-precision case (i.e., Figure ~\ref{fig:ssppr-var-alpha}(a)), \push and \chebpush are the champion and runners-up, respectively. This is because \push and \chebpush are local algorithms, they perform significantly better than other algorithms.  Note that \push is is last one under the high-precision case by Figure ~\ref{fig:ssppr-var-alpha}(b). Besides, our \ltwocheb  performs consistently better than its counterpart \powermethod on different $\alpha$.  For high-precision case (i.e., Figure ~\ref{fig:ssppr-var-alpha}(b)), the performances of \ltwocheb and \chebpush  are significantly better than all competitors except \powerpushsor (but \powerpushsor  performs mediocrely under low precision). This is because \ltwocheb and \chebpush are acceleration algorithms for computing \ssppr, thus they perform  better than their counterparts. (ii) For the $l_1$-error (i.e., Figure ~\ref{fig:ssppr-var-alpha}(c)(d)), the results are similar to Figure ~\ref{fig:ssppr-var-alpha}(b) for both low precision \& high precision cases. Thus, for the sake of brevity, we omit explanations for the $l_1$-error. (iii) The query time of each algorithm increases as $\alpha$ decreases. Besides, \chebpush, \ltwocheb, and \powerpushsor are less sensitive to $\alpha$ compared to other power or push-based algorithms under high precision case (i.e., Figure ~\ref{fig:ssppr-var-alpha}(b)(d)), as they are accelerating algorithms and converge faster than \powermethod (\ltwocheb is asymptotic $(\frac{1}{\sqrt{\alpha}})\times$ faster than \powermethod). These experiments show that, overall, our algorithm outperforms all baselines under different $\alpha$.

\stitle{The effect of parameter $t$.} Similar to \ssppr, we study the performance of various \hkpr algorithms with varying $t$. We vary $t=5,10,20,30,40$ and compare the trade-off curves of $t$ v.s. running time, which is illustrated  in Figure  ~\ref{fig:hkpr-var-t}. For the  low-precision case (Figure ~\ref{fig:hkpr-var-t}(a)(c)), our algorithms \ltwocheb and \chebpush are faster than \agp and \teaplus  in most cases. Besides, they also are less sensitive to parameter $t$ compared to \agp and \teaplus. Moreover, since \chebpush is a local algorithm, it performs consistently better than \ltwocheb. For the high-precision case (Figure ~\ref{fig:hkpr-var-t}(b)(d)), our algorithms botha \ltwocheb and \chebpush  perform $10\times$ better than \agp and \teaplus, and the performance of \ltwocheb and \chebpush are similar. In addition, the query time of each algorithm increases as $t$ increases. However, since \ltwocheb and \chebpush are accelerating algorithms, they are less sensitive to $t$ compared to other algorithms (\ltwocheb is asymptotically $(\sqrt{t})\times$ faster than \powermethod.)
Thus, these experiments show that our algorithms outperform all baselines under
different $t$, which are consistent with our theoretical results (the runtime bound for \chebpush is $\tilde{O}(\frac{t\sqrt{|S|}}{\epsilon})$ and this bound is better than SOTA \agp). 

\begin{figure}
    \centering
    \subfigure[\dblp, Normalized RelErr=$10^{-5}$]{
		\includegraphics[scale=0.24]{experiment figure/hkpr/var-t-high-error.png}}
     \subfigure[\dblp, Normalized RelErr=$10^{-10}$]{
		\includegraphics[scale=0.24]{experiment figure/hkpr/var-t-low-error.png}}

    \subfigure[\dblp, $l_1$-error=$10^{-1}$]{
		\includegraphics[scale=0.24]{experiment figure/hkpr/var-t-high-l1error.png}}
     \subfigure[\dblp, $l_1$-error=$10^{-5}$]{
		\includegraphics[scale=0.24]{experiment figure/hkpr/var-t-low-l1error.png}}
    \caption{The effect of $t$ on our \hkpr algorithms.}\label{fig:hkpr-var-t}
\end{figure}
}

\section{Further Related Work}
\stitle{Personalized PageRank (PPR) Computation.} There are roughly  four different PPR computation problems~\cite{yang2024efficient}: single-source PPR, single-target PPR, single-pair PPR, and top-$k$ PPR. The primary techniques for computing PPR include Monte-Carlo sampling (MC), \powermethod (PM), Forward Push (FP), and Reverse Push (RP), which form the foundation of most existing PPR solutions. For single-source PPR, which involves determining the PPR value from a given source vertex $s$ to every target vertex $t\in V$, state-of-the-art practical algorithms~\cite{wang2017fora,wu2021unifying,lin2020index,liao2022efficient,23sigmodepvr} often utilize the FP+MC or PM+MC frameworks. Notably, Wei et al.~\cite{wei2024approximating} theoretically prove that single-source PPR can be computed in sublinear time with a relative error guarantee using the MC+RP+MC framework. Wang et al.~\cite{wang2020personalized,wang2023singlenode} propose the randomized push technique specifically for addressing single-target PPR, which calculates the PPR value from every node $s \in V$ to a given target node $t$. Liao et al.~\cite{liao2022efficient} and Lofgren et al.~\cite{lofgren16bidirection} employed the RP+MC framework to tackle single-pair PPR, determining the PPR value between a specific source node $s$ and a specific target node $t$. For top-$k$ PPR, which identifies the top-$k$ PPR values between a specific source node$s$ and all nodes in $V$, Wei et al.~\cite{wei2018topppr} use FP+BP+MC to achieve a $\rho$-precision guarantee. As our objective is to enhance the efficiency of push operations (FP and BP) and power method (PI), our solutions complement these existing works.

\stitle{Chebyshev Polynomial Methods.} Chebyshev Polynomials are powerful tools that are widely used to accelerate function approximation in numerical analysis. For example, Bautista and Latapy~\cite{bautista2022local} employ Chebyshev polynomials to design distributed Personalized PageRank algorithms. Shuman et al.\ ~\cite{shuman2018distributed,shuman2011chebyshev} apply Chebyshev polynomials to approximate the distributed graph signal processing.  Defferrard et al.\ ~\cite{defferrard2016convolutional} and He et al.\ ~\cite{he2022convolutional}  decompose graph signals into linear combinations of Chebyshev polynomials. Braverman et al.~\cite{braverman2022sublinear} utilize the stability of Chebyshev polynomials ~\cite{clenshaw1955note,musco2018stability} to design  fast algorithms for spectral density approximation of the transition matrix $\mathbf{P}$. However, since their problem differs from ours, their techniques are fundamentally distinct from ours and cannot be directly applied to our GP vectors computation problem.

\section{Conclusion} \label{sec:conclusion}
In this paper, we propose two innovative and efficient algorithms  for general GP computation based on Chebyshev polynomials. Specifically, we first propose a novel Chebyshev expansion formula for general GP functions, based on which new power-iteration and push-style algorithms are proposed. We prove that our new power-iteration method, namely \ltwocheb, can achieve roughly $O(\sqrt{N})$ acceleration over the state-of-the-art power iteration methods for two well-studied GP computation problems (\ssppr and \hkpr). For the push-style algorithm, we develop a novel subset Chebyshev recurrence technique to achieve a local algorithm with lower time complexity compared to the state-of-the-art push algorithms for general GP computation. Extensive experimental results demonstrate that our algorithms significantly outperform existing state-of-the-art methods for both \ssppr and \hkpr computations in large graphs.


\balance

\bibliographystyle{plain}
\bibliography{ref}

\comment{
\appendix

\section{The generalized subset Chebyshev recurrence}\label{sec:generalize-subset-chebyshev}

 In this section, we introduce the generalized subset Chebyshev recurrence and show that Definition ~\ref{def:purning-chebyshev} is actually the optimal  generalized subset Chebyshev recurrence. These theoretical results may not be helpful for the proposed algorithm, \chebpush, but they are the process by which we get the subset Chebyshev recurrence. For our analysis, we define  three kinds of Chebyshev polynomials.

\begin{definition}{~\cite{mason2002chebyshev}}\label{def:three-kind-chebyshev}
    For any $x\in \mathbb{R}$, the first kind of Chebyshev polynomials $\{T_k(x)\}$, the second kind of Chebyshev polynomials $\{U_k(x)\}$, the third kind of Chebyshev polynomials $\{V_k(x)\}$ are obtained by the following process, respectively: (i) $T_0(x)=U_0(x)=V_0(x)=1$; (ii) $T_1(x)=x, U_1(x)=2x, V_1(x)=2x-1$; (iii) $P_{k+1}(x)=2xP_k(x)-P_{k-1}(x)$ holds for  $k\geq 1$, in which $P_{k+1}$ refers to $T_{k+1}(x)$, or $U_{k+1}(x)$, or $V_{k+1}(x)$.
\end{definition}

The following lemma shows that for any polynomial series $\{P_k(x)\}$ following the three term recurrence $P_{k+1}(x)=2xP_k(x)-P_{k-1}(x)$ can be expressed by the linear combination of the three kinds of Chebyshev polynomials

\begin{lemma}\label{lem:combination-chebyshev}
    Consider the polynomial series $\{P_k(x)\}$ derived by the following process:
\begin{eqnarray} \label{eq:18}
        \left\{
        \begin{aligned}
        P_0(x)&=c_1, P_1(x)=c_2x+c_3 \\
        P_{k+1}(x)&=2xP_k(x)-P_{k-1}(x) \\
        \end{aligned}
        \right.
    \end{eqnarray}
    where $c_1,c_2,c_3\in \mathbb{R}$ are some constants. Then $\{P_k(x)\}$ can be expressed by the linear combination of the three kinds of Chebyshev polynomials.
\begin{equation}
 P_k(x)=(2c_1-c_2)T_k(x)+(c_2-c_1+c_3)U_k(x)-c_3V_k(x),k\geq 0   
\end{equation}   
\end{lemma}

\begin{proof}
    By Definition ~\ref{def:three-kind-chebyshev}, we can know that the recurrence of three kinds of Chebyshev polynomials same as $P_k(x)$  (i.e., Equation \ref{eq:18}) for $k\geq 1$. Therefore, we suppose $P_k(x)=\lambda_1 T_k(x)+\lambda_2 U_k(x)+\lambda_3V_k(x),k\geq 0$ with parameters $\lambda_1,\lambda_2,\lambda_3$ to be determined.
    Taking this into the first two terms $P_0(x)$ and $P_1(x)$, then $\lambda_1,\lambda_2,\lambda_3$ satisfy the linear equation:
\begin{eqnarray}
        \left\{
        \begin{aligned}
        \lambda_1+\lambda_2+\lambda_3&=c_1\\
        \lambda_1+2\lambda_2+2\lambda_3&=c_2\\
        -\lambda_3&=c_3
        \end{aligned}
        \right.
    \end{eqnarray}

    Solving this linear equation, we have $\lambda_1=2c_1-c_2$, $\lambda_2=c_2-c_1+c_3$, $\lambda_3=-c_3$, which proves the lemma.
\end{proof}

Next, we define the generalized subset Chebyshev recurrence with parameters $\lambda_k,\phi_k$ to be determined.

\begin{definition}{(generalized subset Chebyshev recurrence)}\label{def:generalize-subset-chebyshev}
Given source node $s$, matrix $\mathbf{P}$, a series of sets $\{S_k\}_k$ with $S_k\subset \mathcal{V}$, and parameters $\{\lambda_k,\phi_k\}$ to be determined. The generalized subset Chebyshev recurrence is defined as follows:
\begin{eqnarray}\label{equ:subset-cheby-recurrence}
        \left\{
        \begin{aligned}
  \hat{\mathbf{r}}_0&=\mathbf{e}_s;  \hat{\mathbf{r}}_1=\mathbf{Pe}_s \\
        \hat{\mathbf{r}}_{k+1}&=2\mathbf{P}(\hat{\mathbf{r}}_k|_{S_k})-\hat{\mathbf{r}}_{k-1} +\lambda_k \hat{\mathbf{r}}_{k}|_{\mathcal{V}/S_{k}} +  \phi_{k-1}\hat{\mathbf{r}}_{k-1}|_{\mathcal{V}/S_{k-1}}
        \end{aligned}
        \right.
    \end{eqnarray}

    Where $\hat{\mathbf{r}}_k|_{S_k}\in \mathbb{R}^n$ represents the vector $\hat{\mathbf{r}}_k$ constrained on $S_k$. Namely, for any $u\in S_k$, $\hat{\mathbf{r}}_k|_{S_k}(u)=\hat{\mathbf{r}}_k(u)$; for any $u\in \mathcal{V}/ S_k$, $\hat{\mathbf{r}}_k|_{S_k}(u)=0$.
\end{definition}

For this definition, $2\mathbf{P}(\hat{\mathbf{r}}_k|_{S_k})$ is to partially implement the matrix-vector oracle, and $\lambda_k \hat{\mathbf{r}}_{k}|_{\mathcal{V}/S_{k}} +  \phi_{k-1}\hat{\mathbf{r}}_{k-1}|_{\mathcal{V}/S_{k-1}}$ is the "compensate" operation. The following lemma tells us that the residual propagation of the generalized subset Chebyshev recurrence follows the linear combination of three kinds of Chebyshev polynomials.

\begin{lemma}{(the behavior of residual propagation)}\label{lem:residual-propagation-generalize}
    Let $\{\hat{\mathbf{r}}_k\}$ derived by generalized subset Chebyshev recurrence, $\{\mathbf{r}_k\}$ derived by the Chebyshev recurrence. Then the relationship between $\{\hat{\mathbf{r}}_k\}$ and $\{\mathbf{r}_k\}$ is stated  as follows:
    \begin{equation}
        \footnotesize\mathbf{r}_k=\hat{\mathbf{r}}_k-\sum_{l=1}^{k-1}{\left[(2\phi_l-2)T_{k-l}(\mathbf{P})+(2-\phi_l-\lambda_l)U_{k-l}(\mathbf{P})+\lambda_l V_{k-l}(\mathbf{P})\right]\mathbf{\delta}_l}
    \end{equation}
    Where $\delta_l:=\hat{\mathbf{r}}_l|_{V/S_l}$ is the residual produced at step $l$ with $l<k$, and $\{T_k(x)\}$, $\{U_k(x)\}$, $\{V_k(x)\}$ are three kinds of Chebyshev polynomials.
\end{lemma}

\begin{proof}
    By the definition of $\delta_k$, we write the subset Chebyshev recurrence as follows:
    $$\hat{\mathbf{r}}_{k+1}=2\mathbf{P}(\hat{\mathbf{r}}_k-\delta_k)-\hat{\mathbf{r}}_{k-1}+\lambda_k\delta_k+\phi_{k-1} \delta_{k-1}$$
    
    Let's take a closer look at this equation. Note that for every $l< k$, the $l$-step residual $\delta_l$ will have an impact on $\hat{\mathbf{r}}_k$. Without loss of generality, we assume this impact is $\Psi_{k-l}^{(l)}(\mathbf{P})\mathbf{\delta}_l$, where $\Psi_{k-l}^{(l)}(x)$ is a polynomial to be determined. Equivalently speaking, we assume the following equality holds:
     \begin{equation}\label{equ:poly-to-determine}
\mathbf{r}_k=\hat{\mathbf{r}}_k-\sum_{l=1}^{k-1}{\Psi_{k-l}^{(l)}(\mathbf{P})\mathbf{\delta}_l}
    \end{equation}
    Where $\{\Psi_{k-l}^{(l)}(x)\}_{k:k\geq l}$ is some polynomials to be determined. Now, we recall the exact Chebyshev recurrence: $\mathbf{r}_{k+1}=2\mathbf{P}\mathbf{r}_k-\mathbf{r}_{k-1}$. Using Equation~\ref{equ:poly-to-determine} to replace the two sides of the exact Chebyshev recurrence, we have:
    \begin{equation}\label{equ:take-in}
    \begin{aligned}
        \hat{\mathbf{r}}_{k+1}-\sum_{l=1}^{k}{\Psi_{k+1-l}^{(l)}(\mathbf{P})\mathbf{\delta}_l}=2\mathbf{P}\left[\hat{\mathbf{r}}_k-\sum_{l=1}^{k-1}{\Psi_{k-l}^{(l)}(\mathbf{P})\mathbf{\delta}_l} \right] \\
        - \hat{\mathbf{r}}_{k-1}+\sum_{l=1}^{k-2}{\Psi_{k-l-1}^{(l)}(\mathbf{P})\mathbf{\delta}_l}
    \end{aligned}
    \end{equation}

    Then, we use the subset Chebyshev recurrence: $\hat{\mathbf{r}}_{k+1}=2\mathbf{P}(\hat{\mathbf{r}}_k-\delta_k)-\hat{\mathbf{r}}_{k-1}+\lambda_k\delta_k+\phi_{k-1} \delta_{k-1}$ to take place the term with respect to $\{\hat{\mathbf{r}}_k\}$ in two sides of Equation ~\ref{equ:take-in}, we have:

    \begin{equation*}
    \begin{aligned}
        \sum_{l=1}^{k}{\Psi_{k+1-l}^{(l)}(\mathbf{P})\mathbf{\delta}_l}=2\mathbf{P}\left[\delta_k+ \sum_{l=1}^{k-1}{\Psi_{k-l}^{(l)}(\mathbf{P})\mathbf{\delta}_l}\right]  \\
        -\sum_{l=1}^{k-2}{\Psi_{k-l-1}^{(l)}(\mathbf{P})\mathbf{\delta}_l}-2\delta_{k-1}
    \end{aligned}
    \end{equation*}
    
    We compare the coefficients of $\delta_l$ from the two sides of the equation and get the recurrence of $\Psi_{k-l}^{(l)}(\mathbf{P})\delta_l$ as follows:
    \begin{eqnarray*}
        \left\{
        \begin{aligned}
        \Psi_1^{(k)}(\mathbf{P})\delta_{k}&=2\mathbf{P}\delta_k -\lambda_k \delta_k \\
        \Psi_2^{(k-1)}(\mathbf{P})\delta_{k-1}&=2\mathbf{P}\Psi_1^{(k-1)}(\mathbf{P})\delta_{k-1}-\phi_{k-1} \delta_{k-1} \\
        \Psi_{k+1-l}^{(l)}(\mathbf{P})\delta_l&=2\mathbf{P}\Psi_{k-l}^{(l)}(\mathbf{P})\delta_l-\Psi_{k-1-l}^{(l)}(\mathbf{P})\delta_l, & & 1\leq l \leq k-2 
        \end{aligned}
        \right.
    \end{eqnarray*}

    Therefore, the recurrence of $\Psi_{k-l}^{(l)}(x)$ satisfies:
    \begin{eqnarray*}
        \left\{
        \begin{aligned}
        \Psi_1^{(l)}(x)&=2x-\lambda_l;  \Psi_2^{(l)}(x)=2x\Psi_1^{(l)}(x)-\phi_l \\
        \Psi_{k+1-l}^{(l)}(x)&=2x\Psi_{k-l}^{(l)}(x)-\Psi_{k-1-l}^{(l)}(x), & &  k\geq l
        \end{aligned}
        \right.
    \end{eqnarray*}
    
    By Lemma ~\ref{lem:combination-chebyshev}, the recurrence of $\Psi_{k-l}^{(l)}(x)$ is the linear combination of three kinds of Chebyshev polynomials, and the coefficients as we stated in Lemma ~\ref{lem:combination-chebyshev} with $c_1=\phi_l$, $c_2=2$, $c_3=-\lambda_l$. This completes the proof.
\end{proof}
By the property of Chebyshev polynomials ~\cite{mason2002chebyshev}, we can know that the value bound for three kinds of Chebyshev polynomials $\mathop{\max}\limits_{x\in [-1,1]}{|T_k(x)|}=1$, $\mathop{\max}\limits_{x\in [-1,1]}{|U_k(x)|}=k+1$ and $\mathop{\max}\limits_{x\in [-1,1]}{|V_k(x)|}=2k-1$. Figure ~\ref{fig:chebyshev-polynomials} illustrates the different behavior of these three kinds of Chebyshev polynomials. As can be shown, since the value bound for $U_k(x)$, $V_k(x)$ is worse than $T_k(x)$, they make the residual propagation of subset Chebyshev recurrence more sensitive. Thus, we want to leave only $T_k(x)$ by setting reasonable parameters $\lambda_k$ and $\phi_k$. When setting $\lambda_k=0,\phi_k=2$, Lemma ~\ref{lem:invarient} is actually a corollary of Lemma ~\ref{lem:residual-propagation-generalize}. Based on these theoretical analyses, we can restate Lemma ~\ref{lem:invarient} as the following Corollary ~\ref{coro:residual-propagation}.

\begin{figure}
    \centering
    \subfigure[$k=5$, Chebyshev polynomials]{
		\includegraphics[scale=0.24]{illustrative exp/k=5，chebyshev polynomials.png}}
     \subfigure[$k=11$, Chebyshev polynomials]{
     \includegraphics[scale=0.24]{illustrative exp/k=11，chebyshev polynomials.png}}
    
    \caption{The behavior of three kinds of Chebyshev polynomials $T_k(x)$, $U_k(x)$, $V_k(x)$.}\label{fig:chebyshev-polynomials}
\end{figure}

\begin{corollary}{(The stable residual propagation)}\label{coro:residual-propagation}
    When setting $\lambda_k=0$ and $\phi_k=2$, the relationship between $\mathbf{r}_k$ and $\hat{\mathbf{r}}_k$ is:
    \begin{equation}
        \mathbf{r}_k=\hat{\mathbf{r}}_k-2\sum_{l=1}^{k-1}{T_{k-l}(\mathbf{P})\mathbf{\delta}_l}
    \end{equation}
    Where $\delta_l:=\hat{\mathbf{r}}_l|_{\mathcal{V}-S_l}$ is the residual produced at step $l$ with $l<k$.
\end{corollary}

\section{Missing Discussions}\label{sec:generalization}

\subsection{\chebpush for Generalized Graph Propagation}

In the field of graph machine learning, sometimes we prefer the normalized adjacency matrix $\tilde{\mathbf{A}}=\mathbf{D}^{-1/2}\mathbf{A}\mathbf{D}^{-1/2}$ to the walk matrix $\mathbf{P}=\mathbf{AD}^{-1}$. Therefore, the generalized graph propagation is defined as follows ~\cite{wang2021approximate}.

\begin{definition}
    Given a feature $X\in \mathbb{R}^{n\times k}$ and graph propagation function $f$, the generalized graph propagation is defined as: $\mathbf{Z}=f(\mathbf{D}^{-a}\mathbf{A}\mathbf{D}^{-b})\mathbf{X}$ with $a+b=1$.
\end{definition}

This model is more commonly used in machine learning for the node-wise propagation ~\cite{bojchevski2019pagerank,gasteiger2019diffusion,wu2019simplifying}. Now we still focus on the single source version: $\mathbf{z}=f(\mathbf{D}^{-a}\mathbf{A}\mathbf{D}^{-b})\mathbf{e}_s$. Similar to single node queries, we can express this function by a single-source version.

\begin{theorem}
  Given a graph propagation function $f$, the generalized graph propagation $\mathbf{z}=f(\mathbf{D}^{-a}\mathbf{A}\mathbf{D}^{-b})\mathbf{e}_s$ can be expressed as: $\mathbf{z}=d_s^{a}\mathbf{D}^{-a}\mathbf{y}$ with $\mathbf{y}=f(\mathbf{P})\mathbf{e}_s$.
\end{theorem}

\begin{proof}
    The proof is similar to the case of single-node graph centrality. By the definition of walk matrix $\mathbf{P}=\mathbf{AD}^{-1}$ and $a+b=1$, we can know that
$(\mathbf{D}^{-a}\mathbf{A}\mathbf{D}^{-b})^k\mathbf{e}_s=\mathbf{D}^{-a}(\mathbf{AD}^{-1})^k\mathbf{D}^{1-b}\mathbf{e}_s=d_s^{1-b}\mathbf{D}^{-a}(\mathbf{P}^k\mathbf{e}_s)$. 
By the Taylor expansion $\mathbf{z}=f(\mathbf{D}^{-a}\mathbf{A}\mathbf{D}^{-b})\mathbf{e}_s=\sum_{k=0}^{\infty}{\zeta_k(\mathbf{D}^{-a}\mathbf{A}\mathbf{D}^{-b})^k\mathbf{e}_s}$, we have
$\mathbf{z}=\sum_{k=0}^{\infty}{\zeta_k(\mathbf{D}^{-a}\mathbf{A}\mathbf{D}^{-b})^k\mathbf{e}_s}=\sum_{k=0}^{\infty}{\zeta_kd_s^{1-b}\mathbf{D}^{-a}\mathbf{P}^k\mathbf{e}_s }=d_s^{a}\mathbf{D}^{-a}f(\mathbf{P})\mathbf{e}_s$.
\end{proof}

Again, we can use \chebpush to compute the approximate single-source graph propagation $\hat{\mathbf{y}}$ first, and then let $\hat{\mathbf{z}}=d_s^{a}\mathbf{D}^{-a}\hat{\mathbf{y}}$. For the general graph propagation $\mathbf{Z}=f(\mathbf{D}^{-a}\mathbf{A}\mathbf{D}^{-b})\mathbf{X}$, we just run this procedure $O(nnz(\mathbf{X}))$ times and get an approximation $\hat{\mathbf{Z}}$, where $nnz(\mathbf{X})$ is the non-zero entries of $\mathbf{X}$.

\subsection{Future Directions}

Here, we list several possible future directions as follows.

\stitle{An $O\left(\frac{1}{\sqrt{\alpha}\epsilon}\right)$ algorithm for \ssppr. } Whether there exists a $O\left(\frac{1}{\sqrt{\alpha}\epsilon}\right)$ algorithm for \ssppr to satisfy the degree-normalized relative error maintains an open problem in the research community ~\cite{chen2023accelerating,fountoulakis2022open}. Recall that the truncation step $N$ for \powermethod is $\frac{1}{\alpha} \log{\frac{1}{\epsilon}}$, and the truncation step $K$ for \ltwocheb is $\frac{1}{\sqrt{\alpha}} \log{\frac{1}{\epsilon}}$. So the problem is equivalent to asking: whether there exists an algorithm that is implemented locally on the graph while maintaining the $\frac{1}{\sqrt{\alpha}}$ convergence rate. We believe the answer to this problem should connect the structure of three term recurrence with the graph structure, or otherwise, such an algorithm does not exist.

\stitle{More accelerating methods. } There are many other accelerating methods for computing matrix functions in the field of numerical linear algebra, other than Chebyshev polynomials. Among these accelerating methods, the most famous class of algorithms is perhaps the Krylov Subspace-based methods. For a given matrix function $\mathbf{y}=f(\mathbf{A})\mathbf{x}$, the main idea of Krylov Subspace methods is to find the best polynomial solution $\hat{\mathbf{y}}=p(\mathbf{A})\mathbf{x}$ to approximate $\mathbf{y}$ in the subspace $\{\mathbf{x},\mathbf{A}\mathbf{x},\mathbf{A}^2\mathbf{x},...,\mathbf{A}^k\mathbf{x}\}$. We believe that combining the Krylov Subspace methods with graph structure or locally implementing the Krylov Subspace methods is an interesting future direction.
}

\end{document}